%% file: main.tex
\documentclass[]{uiuc_template}
\input{math_commands.tex}

 \usepackage{graphicx} 
\usepackage{url}

\usepackage{amsmath}
\usepackage{amssymb}
\usepackage{mathtools}
\usepackage{amsthm}
\usepackage{booktabs}
\usepackage{multirow}
\usepackage{makecell}
\usepackage[table]{xcolor}
\newcommand{\blue}{\color{blue}}
\usepackage{xspace}
\usepackage{amssymb} % For check mark symbols
\usepackage{multirow}
\usepackage{array}
\usepackage[most]{tcolorbox}
\usepackage{listings}
\newcolumntype{P}[1]{>{\centering\arraybackslash}p{#1}}

\usepackage{enumitem}

\usepackage{tcolorbox}
\usepackage{algorithm}
\usepackage{algpseudocode}

\usepackage{enumitem}
\usepackage{natbib} 

\usepackage{tikz}
\usepackage{wrapfig}
\usepackage{diagbox}
\usepackage{caption}
\newcommand{\updelta}[1]{\textcolor{green!60!black}{\ensuremath{\uparrow}\,\scriptsize(#1)}}
\newcommand{\downdelta}[1]{\textcolor{red!70!black}{\ensuremath{\downarrow}\,\scriptsize(#1)}}
\newtcbtheorem[number within=section]{theorem}{Theorem}{
  enhanced,
  breakable,
  colback=white,
  colframe=blue!70!black,
  coltitle=white,
  colbacktitle=blue!80!black,
  fonttitle=\bfseries,
  boxrule=0.9pt,
  arc=3pt,
  left=8pt,
  right=8pt,
  top=4pt,
  bottom=6pt,
  title filled=true,
  after title={\par\smallskip},  
}{thm}

\newtcbtheorem[number within=section]{lemma}{Lemma}{
  enhanced,
  breakable,
  colback=white,
  colframe=blue!70!black,
  coltitle=white,
  colbacktitle=blue!80!black,
  fonttitle=\bfseries,
  boxrule=0.9pt,
  arc=3pt,
  left=8pt,
  right=8pt,
  top=4pt,
  bottom=6pt,
  title filled=true,
  after title={\par\smallskip},  
}{lemma}

\newtcbtheorem[number within=section]{corollary}{Corollary}{
  enhanced,
  breakable,
  colback=white,
  colframe=blue!70!black,
  coltitle=white,
  colbacktitle=blue!80!black,
  fonttitle=\bfseries,
  boxrule=0.9pt,
  arc=3pt,
  left=8pt,
  right=8pt,
  top=4pt,
  bottom=6pt,
  title filled=true,
  after title={\par\smallskip},  
}{coro}

\newtcbtheorem[number within=section]{definition}{Definition}{
  enhanced,
  breakable,
  colback=white,
  colframe=blue!70!black,
  coltitle=white,
  colbacktitle=blue!80!black,
  fonttitle=\bfseries,
  boxrule=0.9pt,
  arc=3pt,
  left=8pt,
  right=8pt,
  top=4pt,
  bottom=6pt,
  title filled=true,
  after title={\par\smallskip},  
}{def}

\newtcbtheorem[number within=section]{example}{Example}{
  enhanced,
  breakable,
  colback=white,
  colframe=blue!70!black,
  coltitle=white,
  colbacktitle=blue!80!black,
  fonttitle=\bfseries,
  boxrule=0.9pt,
  arc=3pt,
  left=8pt,
  right=8pt,
  top=4pt,
  bottom=6pt,
  title filled=true,
  after title={\par\smallskip},  
}{exa}

\theoremstyle{definition}
\newtheorem{assumption}{Assumption}[section]

\newtcolorbox{promptbox}[1]{
    colback=blue!5!white,    % 背景颜色
    colframe=blue!75!black,  % 边框颜色
    fonttitle=\bfseries,     % 标题字体
    title=#1,                % 标题内容
    enhanced,                % 启用增强特性
    breakable,               % 允许跨页
    boxrule=0.8pt,           % 边框粗细
}

\input{inc/inc_packages}

\newtcolorbox{datanotebox}[1]{
    colback=white,
    colframe=gray!75!black,
    fonttitle=\bfseries,
    title=#1,
    arc=2pt,
    outer arc=2pt,
    boxrule=0.8pt,
    left=5pt,
    right=5pt,
    top=5pt,
    bottom=5pt,
    breakable
}

\definecolor{morandiGreen}{RGB}{80, 130, 100}
\definecolor{morandiYellow}{RGB}{200, 140, 80}
\definecolor{morandiBlue}{RGB}{70, 100, 140}
\definecolor{morandiPink}{RGB}{180, 80, 100}

\definecolor{airforceblue}{rgb}{0.36, 0.54, 0.66}
\definecolor{bluegray}{rgb}{0.4, 0.6, 0.8}
\definecolor{bleudefrance}{rgb}{0.19, 0.55, 0.91}
\hypersetup{colorlinks,linkcolor={airforceblue},citecolor={bleudefrance},urlcolor={bluegray}}
\usepackage{pifont}
\usepackage{fontawesome5}
\definecolor{beijingblue}{RGB}{220,235,247}
\newcommand{\xmark}{\ding{55}}%

\newcommand{\navoffline}{\textcolor{gray}{\faLocationArrow}}

\title{
GUI-Libra: Training Native GUI Agents to Reason and Act with Action-aware Supervision and Partially Verifiable RL 
}

\author[]{%
Rui Yang$^{1\dagger}$,\hspace{0.4em}
Qianhui Wu$^{2*}$,\hspace{0.4em}
Zhaoyang Wang$^{3\dagger}$,\hspace{0.4em}
Hanyang Chen$^{1}$,\hspace{0.4em}
Ke Yang$^{1\dagger}$,\hspace{0.4em}
Hao Cheng$^{2}$\\
Huaxiu Yao$^{3}$,\hspace{0.4em}
Baolin Peng$^{2}$,\hspace{0.4em} 
Huan Zhang$^{1}$,\hspace{0.4em}
Jianfeng Gao$^{2}$,\hspace{0.4em}
Tong Zhang$^{1}$%
}

\affiliation{
\includegraphics[height=4mm, trim=100 55 100 55, clip]{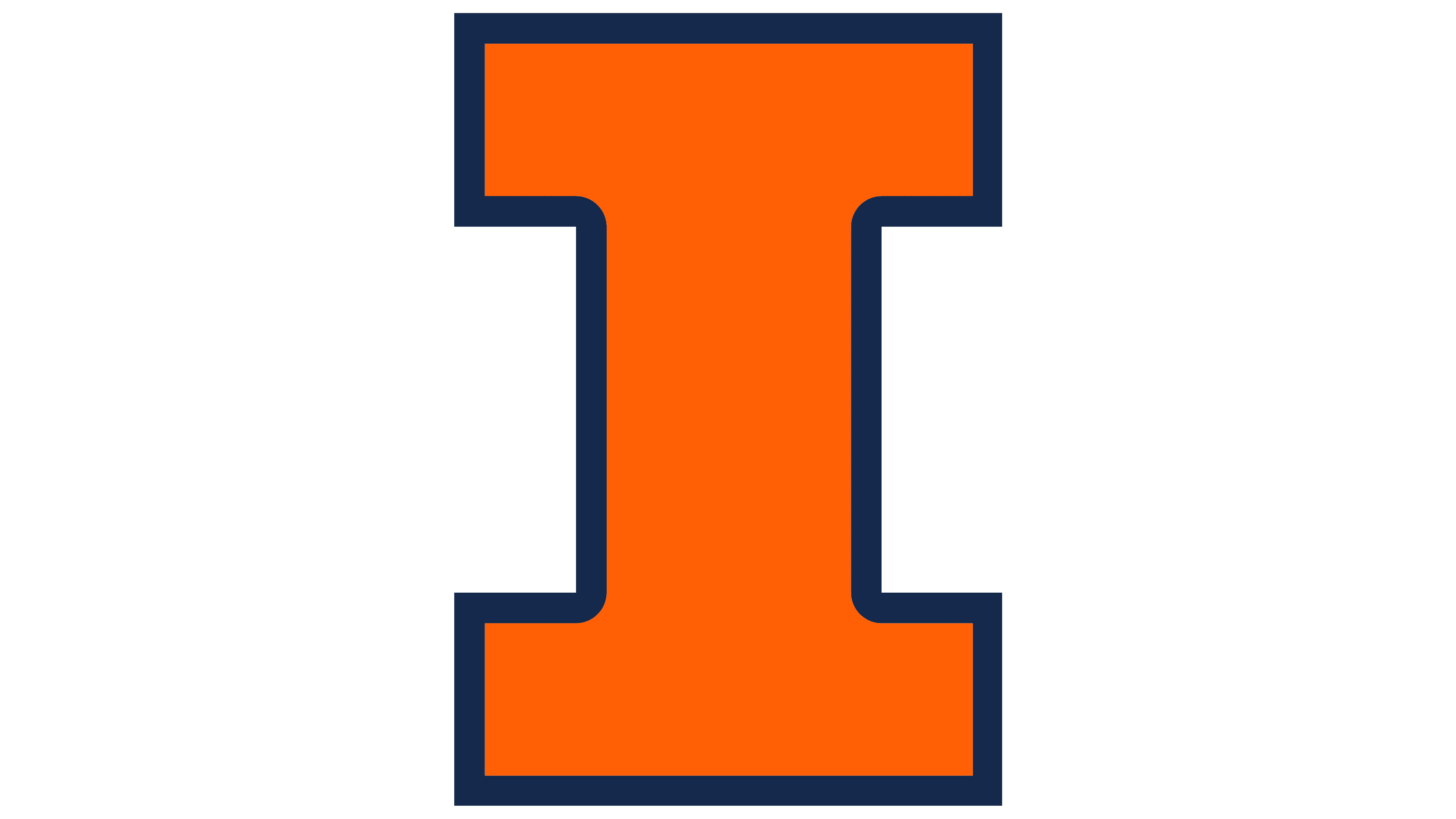} \hspace{-7pt} $^1$UIUC, \hspace{5pt}
\includegraphics[height=4mm, trim=55 55 55 55, clip]{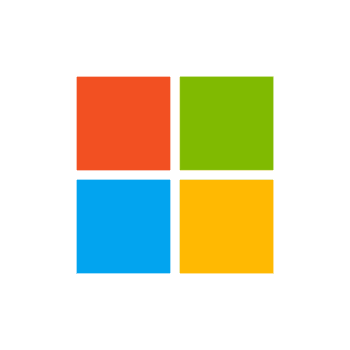} $^2$Microsoft, \hspace{2pt}
\raisebox{-0.5pt}{\includegraphics[height=4mm]{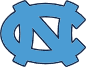}} \hspace{-2pt} $^3$UNC-Chapel Hill
\\
\vspace{6pt}
\href{https://gui-libra.github.io}{\faGlobe\ https://gui-libra.github.io}
}

\abstract{
Open-source native GUI agents have made rapid progress in visual grounding and low-level action execution, yet they still lag behind closed-source systems on long-horizon navigation tasks that demand both high-level reasoning and precise actions. This gap stems from two limitations in the open-source ecosystem: a shortage of high-quality, action-aligned reasoning data, and the direct adoption of generic post-training pipelines that overlook the unique challenges of GUI agents.
We identify two fundamental issues in these pipelines: (i) standard supervised fine-tuning (SFT) with long chain-of-thought (CoT) reasoning often hurts grounding accuracy, and (ii) step-wise RLVR-tyle training faces \emph{partial verifiability}, where multiple actions can be correct at a given state but only a single demonstrated action is used for verification. This causes reward ambiguity and makes offline step-wise metrics weak predictors of online task success during RL training.
In this work, we present \textbf{GUI-Libra}, a systematic study and tailored training recipe that addresses these challenges. First, to mitigate the scarcity of action-aligned reasoning data, we introduce a data construction and filtering pipeline and release a curated 81K GUI reasoning dataset. Second, to reconcile reasoning with grounding, we propose \emph{action-aware supervised fine-tuning} that mixes reasoning-then-action and direct-action supervision, and reweights tokens to emphasize action and grounding. Third, to stabilize RL under partial verifiability, we identify the overlooked importance of KL regularization and show, both theoretically and empirically, that a KL trust region is critical for improving offline-to-online predictability; we further introduce success-adaptive scaling to downweight unreliable negative gradients. 
Across diverse web and mobile benchmarks, GUI-Libra consistently improves both step-wise accuracy and end-to-end task completion, while strengthening the alignment between offline metrics and online performance. In particular, GUI-Libra-4B and GUI-Libra-8B improve their base models by +15.6\% and +12.2\% on AndroidWorld, +4.0\% and +8.7\% on Online-Mind2Web, and +12.5\% and +11.3\% on WebArena-Lite-v2, respectively. Our results suggest that carefully designed post-training and data curation can unlock significantly stronger task-solving capabilities without costly online data collection. We release our dataset, code, and models to facilitate further research on data-efficient post-training for reasoning-capable GUI agents.
}

\begin{document}

\maketitle

\vspace{-13pt}
\begin{figure}[ht]
    \centering
    \includegraphics[width=0.98\linewidth]{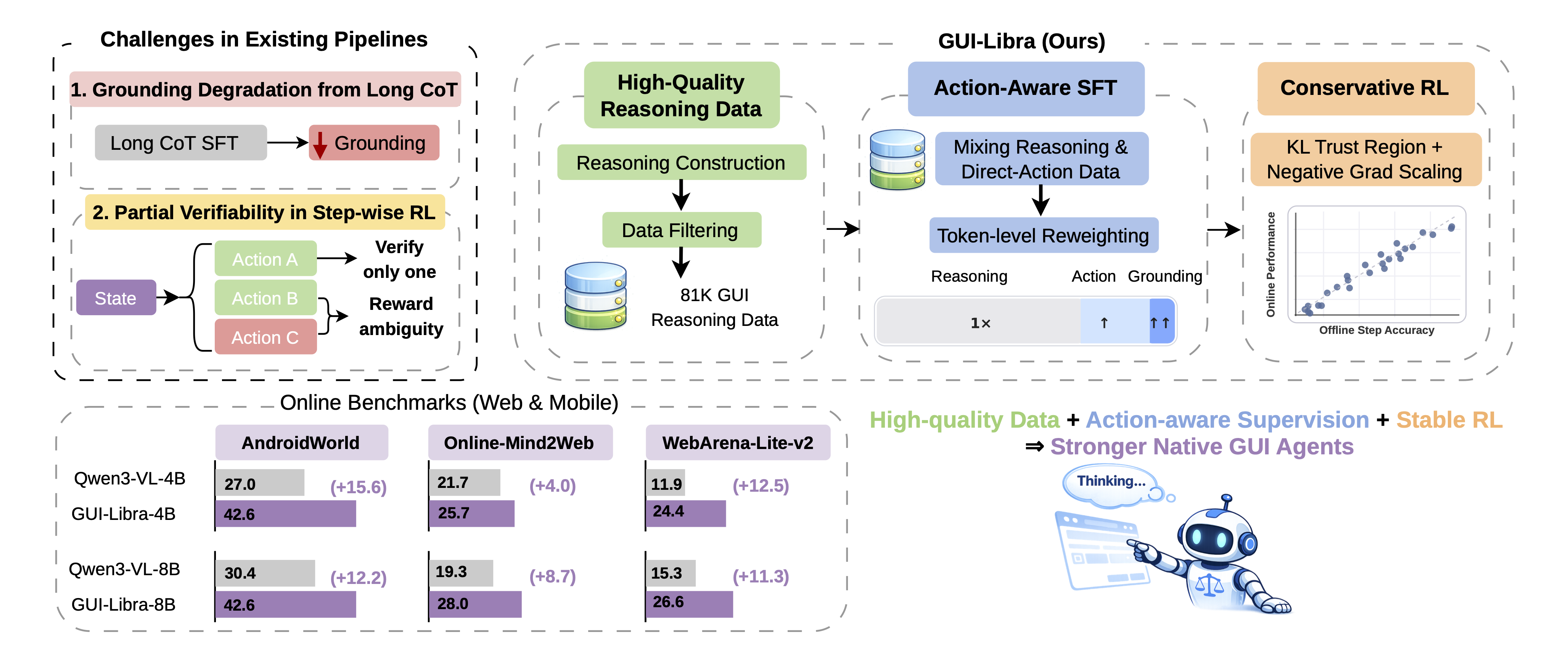}
    \vspace{-6pt}
    \caption{Overview of GUI-Libra. Using only a subset of existing open-source GUI trajectories, we tackle key limitations of prior training pipelines through action-aligned reasoning data curation, action-aware SFT, and conservative RL, yielding consistent gains on online benchmarks.}
    \vspace{-5pt}
    \label{fig:main_idea}
\end{figure}

\input{sections/intro}
\input{sections/related_work}

\input{sections/background}
\input{sections/method}

\input{sections/experiments}

\input{sections/conclusion}

\section*{Limitations}
Our work focuses on learning from existing open-source datasets. While this setting is meaningful and our results suggest that current trajectory corpora still have substantial untapped potential, we train on a relatively limited amount of data and do not explore how to extend the framework to fully online, interactive training. As more large-scale open-source GUI interaction data become available~\citep{he2025scalable,wang2025opencua,zhang2025tongui}, scaling our pipeline to incorporate broader and more diverse trajectories is a promising direction. In addition, fully online RL can be expensive, slow, and typically requires robust infrastructure and careful system design. We leave a systematic study of extending our framework to fully online scheme as future work.

\section*{Acknowledgments}
The authors thank Hao Bai, Chenlu Ye, and Xiao Yu for valuable discussions, and Boyu Gou and Yiheng Xu for guidance on benchmark and model evaluation.

\bibliography{arxiv}
\bibliographystyle{arxiv}

% \addcontentsline{toc}{section}{Appendix} % Add the appendix text to the document TOC
% \appendix
% \part{Appendix} % Start the appendix part
% \parttoc % Insert the appendix TOC
\appendix
\input{sections/appendix}

\end{document}

%% file: math_commands.tex
\usepackage{amsmath,amsfonts,bm}

\def\eqref#1{equation~\ref{#1}}
\def\1{\bm{1}}

\DeclareMathAlphabet{\mathsfit}{\encodingdefault}{\sfdefault}{m}{sl}
\SetMathAlphabet{\mathsfit}{bold}{\encodingdefault}{\sfdefault}{bx}{n}

%% file: inc/inc_packages.tex
\usepackage{makecell}
\usepackage{url}
\usepackage{breakurl}
\usepackage[breaklinks]{hyperref}

%% file: sections/intro.tex
\section{Introduction}
\label{sec:intro}

Large vision--language models (VLMs) have become a central building block for graphical user interface (GUI) agents \citep{qin2025uitars,xu2025aguvis,gou2025uground,wang2025uitars2,bai2025qwen3vltechnicalreport}, enabling autonomous systems to interpret visual interfaces and output executable actions to complete complex tasks across digital platforms. Among these approaches, native GUI agents \citep{qin2025uitars} refer to a single end-to-end model that directly maps user instructions and observations to executable actions, without relying on external planners or separate grounding modules. 
Recent open-source native agents have achieved substantial progress in visual grounding and low-level action execution \citep{xu2025aguvis,wang2025uitars2,liu2025scalecua,wang2025opencua}, significantly narrowing the gap with proprietary systems. Despite these advances, native GUI agents remain less effective at long-horizon decision making, where agents must reason over extended observation--action sequences and adapt their behavior reliably to achieve user-specified goals.

Advancing native GUI agents increasingly depends on effective post-training of GUI-centric VLMs, yet current approaches face two intertwined bottlenecks. The first is the scarcity of high-quality, action-aligned reasoning data. Existing GUI navigation datasets \citep{li2024androiedcontrol,zheng2024seeact,xu2025aguvis} often lack explicit rationales, contain only short or weakly grounded reasoning traces, or include noisy action labels, providing limited supervision for learning robust and interpretable policies. The second bottleneck is the widespread use of generic post-training recipes that do not fully account for the unique properties of GUI agents. Most existing open-source pipelines rely on either supervised fine-tuning (SFT) on brief rationales \citep{wu2025osatlas,xu2025aguvis} or reinforcement learning (RL) primarily targeting grounding accuracy \citep{luo2025guir1,lu2025uir1,zhou2025guig1,yang2025gta1}.
In practice, these approaches expose a persistent tension between reasoning and grounding: incorporating chain-of-thought (CoT) often degrades grounding performance, leading many methods to suppress explicit reasoning rather than addressing the underlying trade-off. Meanwhile, motivated by the success of RL from verifiable rewards (RLVR) in domains such as mathematical reasoning \citep{shao2024deepseekmath,yu2025dapo}, recent work \citep{hong2025glm,yang2025ferretuilite} has explored step-wise RL for GUI agents. However, these methods overlook a fundamental characteristic of GUI interaction, \textbf{partial verifiability: at each step, multiple actions may correctly advance the task, yet offline supervision verifies only a single demonstrated action}. As a result, alternative valid actions are ambiguously treated as failures, introducing biased gradients, destabilizing training, and weakening the connection between offline evaluation metrics and online task success.

To address these challenges, we propose GUI-Libra, a unified post-training framework designed to strengthen decision making in native GUI agents. GUI-Libra is driven by three insights. (i) High-quality rationales and careful data filtering are essential for data-efficient learning, especially because open-source GUI trajectories are often noisy and weakly annotated. (ii) During SFT, CoT tokens can dominate the training loss and interfere with grounding; effective learning therefore requires explicitly prioritizing action and grounding tokens, which directly determine execution. (iii) Under partially verifiable rewards, RL can become unstable without conservative constraints: unlike standard RLVR settings where dropping KL regularization often helps \citep{yu2025dapo,liu2025understanding,zhou2025guig1,yang2025gta1}, GUI agents benefit from moderate, KL-regularized RL that mitigates reward ambiguity and distribution shift, improving robustness and offline–online alignment.

Guided by these insights, GUI-Libra integrates action-aware supervised fine-tuning (ASFT) with conservative reinforcement learning. To alleviate the scarcity of high-quality reasoning data, we develop a scalable construction and filtering pipeline and release a curated 81K GUI reasoning dataset with improved alignment between reasoning traces and executable actions. In ASFT, GUI-Libra trains on a mixture of reasoning-then-action and direct-action supervision, and applies action-aware token reweighting to emphasize action and grounding tokens, reducing the grounding degradation caused by long CoT traces. In RL, GUI-Libra optimizes policies with GRPO \citep{shao2024deepseekmath} under moderate KL regularization, and further introduces a success-adaptive negative gradient scaling strategy to reduce bias from ambiguously ``negative'' outcomes. \textbf{Our pipeline has two practical advantages}: (1) it derives all training data from existing open-source resources, showing that \textbf{careful augmentation, filtering, and training method design can make modest open data competitive with closed-data systems}; and (2) \textbf{it avoids costly online environment interaction, making training scalable and accessible while strengthening the connection between offline metrics and online task success.}

Extensive experiments across web and mobile benchmarks show that the GUI-Libra series (3B–8B) consistently improves offline step-wise accuracy on standard offline benchmarks and boosts online task completion on AndroidWorld \citep{rawles2025androidworld}, WebArena-Lite-v2 \citep{liu2025scalecua}, and Online-Mind2Web \citep{online-mind2web}. Notably, GUI-Libra-4B and GUI-Libra-8B improve their base models by +15.6\% and +12.2\% on AndroidWorld, and +4.0\% and +8.7\% on Online-Mind2Web. Detailed ablations further confirm the roles of action-aware supervision and conservative regularization in mitigating grounding degradation, strengthening action prediction, and stabilizing learning under partially verifiable feedback. We also analyze the impact of data filtering and explicit reasoning, and study the trade-off between reasoning and grounding during RL training. We hope these findings and open-source resources will encourage future work on data-efficient and reliable post-training for native GUI agents.

Our main contributions are summarized as follows:
\begin{itemize}
\item We present \textbf{GUI-Libra}, a unified post-training framework for native GUI agents that tackles two key challenges: reasoning--grounding interference in SFT, addressed by action-aware SFT that emphasizes action and grounding tokens; and weak offline-to-online predictability in RL, addressed by conservative optimization that constrains policy drift and improves offline--online alignment.

\item We develop a scalable data construction and filtering pipeline and release a high-quality open-source 81K GUI reasoning dataset with improved action alignment.

\item We achieve consistent gains across representative offline and online web and mobile benchmarks, showing that smaller native VLMs trained on modest open-source data can match or even outperform much larger systems.
\end{itemize}

%% file: sections/related_work.tex
\section{Related Work}\label{sec:relatedwork}

\begin{table}[t]
\centering
\small
\caption{Comparison of existing training recipes from different perspectives. The task type column indicates the training target of the released models: \faSearch~denotes GUI grounding, \faLocationArrow~denotes GUI navigation (both online and offline), and \navoffline~denotes offline step-wise action prediction.}

\resizebox{1\textwidth}{!}{%
\begin{tabular}{lccccccc}
\toprule
Name & Task Type & Reasoning &  SFT & RL & Open Weights & Open Data & Open Code \\
\midrule
OS-Atlas~\citep{wu2025osatlas} & \faSearch \navoffline & \xmark  & \checkmark & \xmark & \checkmark &  \checkmark & \xmark \\
AGUVIS~\citep{xu2025aguvis} & \faSearch \faLocationArrow  & short   &  \checkmark & \xmark & \checkmark &  \checkmark & \checkmark \\
UGround~\citep{gou2025uground} & \faSearch  & \xmark  &  \checkmark & \xmark & \checkmark &  \checkmark & \checkmark \\
ScaleCUA~\citep{liu2025scalecua} &  \faSearch\faLocationArrow  & long  &  \checkmark & \xmark & \checkmark &  \checkmark & \checkmark \\  
OpenCUA~\citep{wang2025opencua}   & \faSearch \faLocationArrow  & long  &  \checkmark & \xmark & \checkmark &  \checkmark & \checkmark \\        
UI-TARS~\citep{qin2025uitars} & \faSearch  \faLocationArrow & short   &  \checkmark & \checkmark & \checkmark &  \xmark & \xmark \\
GLM-4.1-V~\citep{hong2025glm} & \faSearch \faLocationArrow & long   &  \checkmark & \checkmark & \checkmark &  \xmark & \xmark \\
Ferret-UI Lite~\citep{yang2025ferretuilite} & \faSearch \faLocationArrow & long   &  \checkmark & \checkmark & \xmark &  \xmark & \xmark \\
UI-R1~\citep{lu2025uir1} & \faSearch \navoffline  & short   &  \xmark & \checkmark & \checkmark &  \checkmark & \checkmark \\
GUI-R1~\citep{luo2025guir1}  & \faSearch \navoffline  & short  &  \xmark & \checkmark & \checkmark &  \checkmark & \checkmark \\
GTA1~\citep{yang2025gta1} & \faSearch   & \xmark  &  \xmark & \checkmark & \checkmark &  \checkmark & \checkmark \\
\midrule
GUI-Libra (Ours) & \faSearch \faLocationArrow  & long  &  \checkmark & \checkmark & \checkmark &  \checkmark & \checkmark \\
\bottomrule
\end{tabular}}
\label{tab:overall_training_receipe_comparison}
\end{table}

\subsection{Datasets for Training GUI Agents}
Recent progress in GUI agents has been propelled by a diverse ecosystem of datasets that target both visual perception and task execution. For robust visual grounding and screen parsing, datasets such as SeeClick~\citep{cheng2024seeclick}, UGround~\citep{gou2025uground}, GUIAct~\citep{chen2025guiact}, ScaleCUA~\citep{liu2025scalecua}, and \textsc{GUI-360}~\citep{mu2025gui360} provide large corpora of annotated screenshots and UI element supervision~\citep{deka2017RICO,li2020widgetCaptioning,li2020RICOSCA,bai2021uibert,wu2023webui,yang2025pixelweb,zheng2025agentstudio,wu2025osatlas,nayak2025uivision,luo2025guir1}.

Moving beyond single-step grounding, several large-scale context-aware and trajectory-based datasets capture multi-step interactions in realistic environments, enabling models to learn how UI state evolves over time. Examples include AITW~\citep{rawles2023aitw}, MM-Mind2Web~\citep{zheng2024seeact,deng2023mindweb}, AMEX~\citep{chai2025amex}, GUI Odyssey~\citep{lu2024guiodyssey}, and Aria-UI~\citep{yang2024ariaUI}.
In addition, datasets such as AndroidControl~\citep{li2024androiedcontrol} and \textsc{JEDI}~\citep{xie2025jedi} enrich interaction trajectories with low-level action descriptions, helping bridge high-level intent with executable operations and improving the learnability of fine-grained GUI manipulation policies.

To train agents to understand not only \emph{what} actions to take but also \emph{why}, recent efforts introduce natural-language rationales that explicitly inject observation interpretation and planning into step-by-step decision making AITZ~\citep{zhang2024aitz}, AgentTreck~\citep{xu2025agenttrek}, OS-Genesis~\citep{sun2024osgenesis}, Aguvis~\citep{xu2025aguvis}, GUI-Net-1M~\citep{zhang2025tongui}, and WebSTAR \citep{he2025scalable}. Despite their promise, such reasoning annotations are often short and noisy, limiting their effectiveness in reliably teaching long-horizon reasoning, error recovery, and strategy adaptation. AgentNet~\citep{wang2025opencua} takes a step further by synthesizing more detailed reasoning traces that include reflective thoughts, enabling agents to detect mistakes and recover mid-trajectory. However, AgentNet primarily focuses on desktop environments, and high-quality reasoning-rich data for mobile and web scenarios remains scarce, leaving open challenges for training robust, general-purpose GUI agents across platforms.

% might missed:
% webshop, miniWob, omniact, crab
% GUI-World: GUI Understanding, GUI-Oriented instruction-tuning (QA, action desc)

\subsection{VLM Post-training for GUI Agents}
Recent advances in GUI agents have been largely driven by post-training VLMs to align natural-language instructions with actionable UI interactions. Many GUI grounding-oriented methods primarily rely on SFT with curated interaction or annotation data, including representative efforts such as SeeClick~\citep{cheng2024seeclick}, OS-Atlas~\citep{wu2025osatlas}, Aria-UI~\citep{yang2024ariaUI}, and JEDI~\citep{xie2025jedi}. Beyond text-based coordinate prediction, GUI-Actor \citep{wu2025gui} applies an explicit attention mechanism to improve the generalization to out-of-distribution screenshots.
Instead of imitation-style learning, a growing body of work explores to improve grounding accuracy and robustness via reinforcement learning, including UI-R1~\citep{lu2025uir1}, GUI-R1~\citep{luo2025guir1}, GUI-G1~\citep{zhou2025guig1}, GUI-G2~\citep{tang2025guig2}, GTA1~\citep{yang2025gta1}, and InfiGUI-G1~\citep{liu2025infiguig1}. Hybrid pipelines that combine SFT+RL further push performance by leveraging high-quality demonstrations for initialization and RL for policy refinement, such as Phi-Ground~\citep{zhang2025phi} and UI-Ins~\citep{chen2025uiins}.

More recently, research increasingly targets unified native GUI models that jointly learn grounding, planning, and multi-step navigation in an end-to-end manner. Several works adopt SFT-only training on mixed trajectory data to obtain strong generalist computer-use models, including CogAgent~\citep{hong2023cogagent}, Aguvis~\citep{xu2025aguvis}, ScaleCUA~\citep{liu2025scalecua}, FARA~\citep{awadallah2025fara}, and OpenCUA~\citep{wang2025opencua}. To further equip agents with the ability to explore diverse strategies and improve long-horizon success via trial-and-error, other efforts incorporate RL-based post-training for better policy optimization and stability, such as DigiRL~\citep{bai2024digirl}, AutoGLM~\citep{liu2024autoglm}, UI-TARS~\citep{qin2025uitars,wang2025uitars2}, MAI-UI~\citep{zhou2025maiui}, UI-Venus~\citep{gu2025uivenus}, Ferret-UI-Lite~\citep{yang2025ferretuilite}, and WebGym \citep{bai2026webgym}.
Despite these advances, three limitations remain. First, encouraging free-form reasoning during RL can hurt grounding accuracy \citep{zhou2025guig1,yang2025gta1,tang2025guig2,lu2025uir1,chen2025uiins}, as models may prioritize high-level semantic logic over precise spatial execution. Second, online RL \citep{wang2025uitars2,zhou2025maiui,bai2026webgym} is expensive to scale, requiring costly environment interaction and robust infrastructure. Third, step-wise RLVR-style training \citep{yang2025ferretuilite,hong2025glm} in GUI settings faces partial verifiability, which makes rewards ambiguous and introduces noisy or biased learning signals.

In this paper, we focus on systematically understanding VLM post-training for GUI agents, and contribute open-source training recipes together with a high-quality, openly released dataset, enabling reproducible development of GUI agents with enhanced reasoning capability.

%% file: sections/background.tex
\section{Preliminaries}\label{sec:preliminaries}

\paragraph{\textbf{VLM-based GUI agents.}}
We formulate GUI interaction as a goal-conditioned partially observable Markov decision process (POMDP) with a natural-language instruction space $\mathcal{L}$,
specified by $(\mathcal{S}, \mathcal{A}, \mathcal{O}, \mathcal{T}, \mathcal{R}, \gamma)$.
Here, $\mathcal{S}$ denotes latent environment states (e.g., application/page context and UI layout), and $\mathcal{A}$ is an action space where each action consists of an operation and its arguments
(e.g., $\texttt{click}(x,y)$, $\texttt{type}(\mathrm{text})$, $\texttt{scroll}(\mathrm{direction})$).
State transitions follow $\mathcal{T}(s_{t+1}\mid s_t,a_t)$.
At the beginning of each episode, an instruction $\ell\in\mathcal{L}$ specifies the task goal.
At each step, the agent receives a partial observation $o_t\in\mathcal{O}$ derived from the underlying state $s_t$, typically a screenshot.
Due to partial observability, the agent conditions on the interaction history
$h_t=(o_0,a_0,\ldots,o_{t-1},a_{t-1})$ together with the current observation $o_t$,
and follows a VLM-parameterized policy $\pi_\theta(a_t\mid \ell, h_t, o_t)$.
The goal-conditioned reward is $r_t=\mathcal{R}(s_t,a_t;\ell)$, which is often sparse and success-only:
$r_t=1$ if the agent achieves the goal specified by $\ell$ (typically at termination), and $r_t=0$ otherwise.
The episode terminates upon success or after a maximum horizon $T$.
The objective is to maximize the expected return,
$\max_{\pi_\theta}\mathbb{E}\!\left[\sum_{t=0}^{T}\gamma^t r_t\right]$.
When $\gamma=1$, this objective is equivalent to maximizing expected task success.

\paragraph{\textbf{High-level vs.\ low-level GUI tasks.}}
We categorize GUI tasks by their temporal abstraction.
A \textit{low-level} task can be completed with a single atomic interaction, such as
``type \texttt{amazon.com} in the address bar'' or ``click the confirm button.''
In contrast, a \textit{high-level} task requires a multi-step interaction trajectory, such as
``buy a machine learning textbook on Amazon,''
which induces a sequence of low-level actions across multiple screens.
We view \textit{grounding} as a special case of low-level decision making, where the agent localizes the target UI element by predicting its interaction coordinates from the instruction and the current observation $(\ell, o_t)$. In our paper, we mainly focus on high-level navigation tasks.

\paragraph{\textbf{Post-training for GUI Models.}}
\emph{Supervised fine-tuning (SFT)} is a standard approach for post-training GUI models.
We assume a dataset of expert trajectories $D=\{\tau_i\}_{i=1}^{N}$, where each trajectory is
$\tau_i=\{(\ell^i, h_t^i, o_t^i, c_t^i, a_t^i)\}_{t=0}^{T_i-1}$.
At step $t$, the model conditions on the context $x_t^i=(\ell^i, h_t^i, o_t^i)$ and outputs a reasoning trace $c_t^i$ followed by an executable action $a_t^i$.
We concatenate reasoning and action into a single target sequence $y_t^i=[c_t^i; a_t^i]$ and minimize the negative log-likelihood:
\begin{equation}
\mathcal{L}_{\mathrm{SFT}}(\theta)
=
- \mathbb{E}_{(x_t, y_t)\sim D}
\log \pi_{\theta}(y_t \mid x_t),
\end{equation}
where the expectation is taken over all trajectories and time steps in $D$.

Beyond SFT, prior work \citep{luo2025guir1,lu2025uir1,zhou2025guig1,yang2025gta1}
adopts reinforcement learning from verifiable rewards (RLVR) to directly optimize step-wise action/coordinate correctness.
Given step contexts $x=(\ell,h_t,o_t)\sim D$, we sample a group of $G$ candidate actions
$\{a_k\}_{k=1}^{G}\sim \pi_{\theta_{\mathrm{old}}}(\cdot\mid x)$ and compute rewards
$r_k=\mathcal{R}(x,a_k)$,
where $\mathcal{R}$ is a weighted combination of rule-based matching scores on action type, value (text), and coordinates.
GRPO then updates the policy using a group-relative variant of policy gradient:
\begin{equation}\label{eq:grpo_kl}
\mathcal{L}_{\mathrm{GRPO}}(\theta)
=
- \mathbb{E}_{x \sim D, \{a_k\}_{k=1}^{G}\sim \pi_{\theta_{\rm old}}(\cdot\mid x)}\!\left[
\frac{1}{G}\sum_{k=1}^{G}
\min\!\Big(
\rho_k \hat{A}_k,\;
\mathrm{clip}(\rho_k, 1-\epsilon, 1+\epsilon)\hat{A}_k
\Big)
- \beta \cdot \mathrm{KL}\!\left(\pi_\theta(\cdot\mid x)\,\|\,\pi_{\mathrm{ref}}(\cdot\mid x)\right)
\right],
\end{equation}
where $\rho_k=\pi_\theta(a_k\mid x)/\pi_{\theta_{\mathrm{old}}}(a_k\mid x)$ and
$\hat{A}_k=(r_k-\mu_r)/(\sigma_r+\delta)$ is the group-normalized advantage.
Here, $\mu_r$ and $\sigma_r$ are the mean and standard deviation of $\{r_k\}_{k=1}^G$, and $\delta$ is a small constant for numerical stability.
The coefficient $\beta$ controls KL regularization toward a reference policy $\pi_{\mathrm{ref}}$.
In practice, \textbf{recent RLVR-style work often removes the explicit KL term (i.e., sets $\beta=0$) }\citep{yu2025dapo,liu2025understanding,zhou2025guig1,yang2025gta1}.
While RLVR provides convenient automatic supervision, step-wise rewards can be ambiguous for high-level GUI tasks: in the same state, multiple distinct actions may be valid and still make progress, making rule-based matching an imperfect proxy for correctness.
We analyze this discrepancy in Section~\ref{sec:RLPR}.

%% file: sections/method.tex
\section{Reasoning Data Curation for GUI Agents}\label{sec:data_curation}

To address the scarcity of high-quality reasoning data for GUI agents, we develop an automated pipeline to construct and filter a high-quality reasoning dataset, \textsc{GUI-Libra-81K}.

\begin{table}[t]
\centering
\caption{Comparison of our GUI reasoning dataset with previous open-source datasets in web and mobile domains. Token statistics are computed using the Qwen2.5-VL-3B-Instruct tokenizer.}
\small
\resizebox{0.9\textwidth}{!}{%
\begin{tabular}{lrrr}
\toprule
Dataset & Avg Thought Token Per Step & Total Steps & \#Traj \\
\midrule
MM-Mind2Web \citep{zheng2024seeact} & 0 & 8K   & 1K  \\ 
AndroidControl \citep{li2024androiedcontrol} & 11 & 75K & 14K  \\
GUIAct \citep{chen2025guiact} & 0 & 17K &  2.5k \\
AMEX \citep{chai2025amex} & 0 & 35K & 3K \\
GUI-Net-1M \citep{zhang2025tongui} & 37 & 4M  & 1M \\
ScaleCUA \citep{liu2025scalecua} & 0 & 170K &  19K \\
AGUVIS \citep{xu2025aguvis} Stage2 L2 (All sources)  & 56 & 300K & 35K \\
AGUVIS \citep{xu2025aguvis} Stage 2 L3 (All sources) & 85 & 300K & 35K   \\
\midrule
 \textsc{GUI-Libra-81K} (Ours) & 210 & 81K & 9K \\
\bottomrule
\end{tabular}}
\label{tab:overall_block_tokens_compare}
\end{table}

\begin{figure}[ht]
    \vspace{-0.0in}
    \centering
    \includegraphics[width=0.95\linewidth]{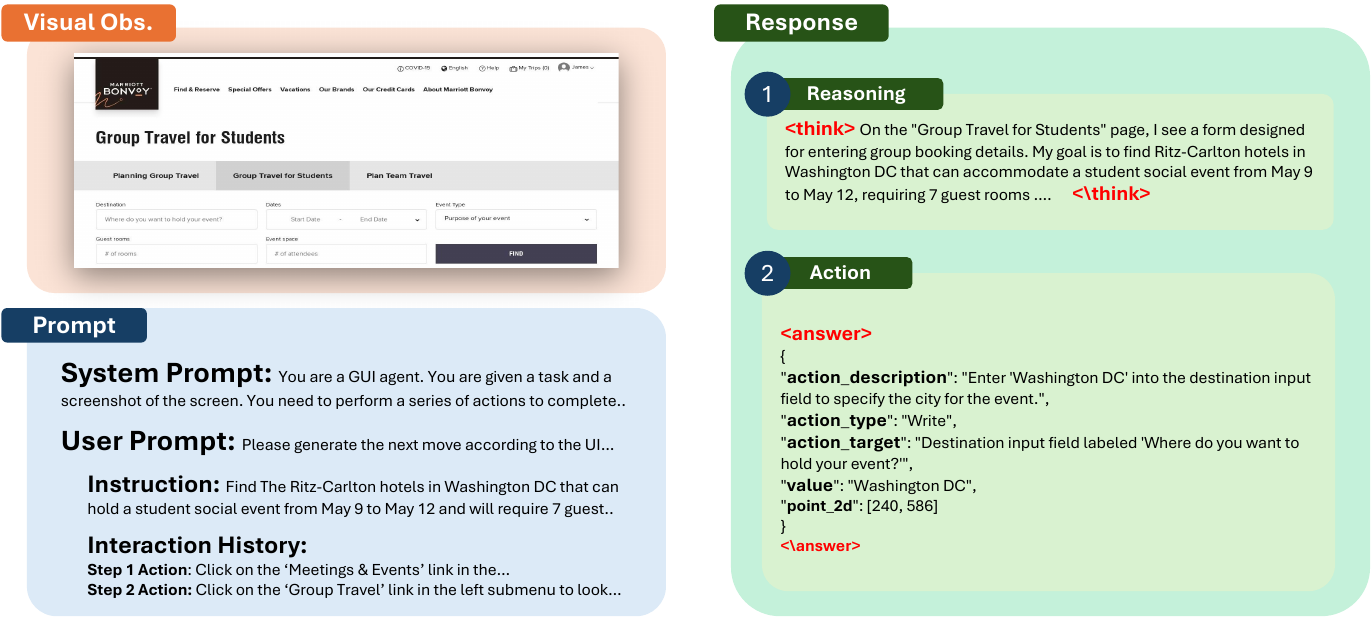}
    \caption{Example data format in \textsc{GUI-Libra-81K}. Each sample includes the current visual observation (screenshot) and textual context (system prompt, user instruction, and interaction history/previous actions). The model output is split into (1) a CoT reasoning trace and (2) a structured executable action (JSON), specifying the action type, a brief action description, the target element (if available), and action arguments such as text values or coordinates.}
    \label{fig:data_structure}
\end{figure}

\subsection{Data Curation and Filtering Pipeline}
As shown in Table~\ref{tab:overall_block_tokens_compare}, existing open-source datasets in web and mobile domains (e.g., the AGUVIS collection~\citep{xu2025aguvis}) typically provide only short rationales, often fewer than 100 thought tokens per step. Rather than collecting costly new data from online environments, \textbf{we aim to fully leverage the large volume of existing web and mobile data by augmenting it with richer CoT reasoning and filtering out low-quality samples that would otherwise induce reasoning--action mismatch.}

\subsubsection{Data Sources}
Although large-scale open-source datasets exist for GUI grounding \citep{gou2025uground,wu2025osatlas}, trajectory-based GUI navigation data remain relatively scarce due to the high cost of collecting multi-step interaction traces.
Following AGUVIS~\citep{xu2025aguvis}, we therefore aggregate trajectory data from multiple public sources that cover both web and mobile domains, including
GUI-Odyssey~\citep{lu2024guiodyssey}, AMEX~\citep{chai2025amex}, AndroidControl~\citep{li2024androiedcontrol}, AitZ~\citep{zhang2024aitz}, AitW~\citep{rawles2023aitw}, GUIAct~\citep{chen2025guiact}, and MM-Mind2Web~\citep{zheng2024seeact}.
Compared with the original AGUVIS collection, we additionally include the Chinese subset from GUIAct to broaden multilingual coverage and increase website diversity.
Overall, these datasets span diverse applications and websites, and include tasks with varying difficulty levels. 
We apply an initial cleaning stage to remove incomplete trajectories, extremely short or long traces (fewer than 3 steps or more than 50 steps), and steps containing compound actions that cannot be represented in our action space. After cleaning, we obtain 19K trajectories comprising 170K steps.

\subsubsection{Unified Structured Format}
Figure~\ref{fig:data_structure} summarizes our unified data format.
Each sample contains an input and an output.
The input includes a system prompt that enumerates available actions, the user instruction, the interaction history (previous actions), and the current screenshot.
The output contains (1) a reasoning trace enclosed by \texttt{<think>} \dots \texttt{</think>} and (2) a structured action enclosed by \texttt{<answer>} \dots \texttt{</answer>}.
The structured action is represented as a JSON object with an \texttt{action\_type} and corresponding arguments, such as \texttt{value} for text entry and \texttt{point\_2d} for click coordinates. We consider 13 common action types for web and mobile control: \texttt{Click}, \texttt{Write}, \texttt{Terminate}, \texttt{Swipe}, \texttt{Scroll}, \texttt{NavigateHome}, \texttt{Answer}, \texttt{Wait}, \texttt{OpenAPP}, \texttt{NavigateBack}, \texttt{KeyboardPress}, \texttt{LongPress}, and \texttt{Select}.
The action optionally includes an \texttt{action\_target}, a natural-language description of the UI element to interact with. In addition, we include an \texttt{action\_description} that succinctly states the intended operation, which is appended to the interaction history for subsequent steps and captures a brief step-level rationale. Retaining \texttt{action\_target} further supports downstream filtering by enabling consistency checks between the described target element and the coordinates from the original datasets. Details of the action space are provided in Appendix~\ref{ap:implementation_details}.

\subsubsection{Action-aligned Reasoning Augmentation}
Most existing GUI trajectory datasets lack detailed reasoning traces or only include short rationales. AGUVIS \citep{xu2025aguvis} augments trajectories by prompting GPT-4o with the instruction, previous actions, and the current action, then requesting a brief ``thought'' and a one-sentence action description.
We find two factors that limit the quality of such generated reasoning.
First, the prompt is not sufficiently informative.
We extend it with GUI-specific guidelines that encourage structured reasoning (observation description, reflection, and planning), enforce format constraints, and add action-related requirements.
Second, reasoning quality is sensitive to the choice of generator model.
We compare reasoning traces produced by GPT-4o, o4-mini, and GPT-4.1 and observe substantial differences across models (Figure~\ref{fig:dataset_comparison}).
Our structured output format also facilitates reliable parsing and downstream processing. Moreover, we do not force the generator to exactly follow the dataset action; instead, we treat the annotated action as a reference and allow the model to select a different action from the available set when it has sufficient justification. The full prompt template is provided in Appendix~\ref{ap:prompt_template}.

At each step, the generator produces the reasoning trace, \texttt{action\_description}, \texttt{action\_type}, \texttt{action\_target}, and \texttt{value}, while we reuse the coordinates from original dataset as \texttt{point\_2d}.
Because generation is conditioned on noisy original annotations, mismatches can still arise, for example, the dataset may contain incorrect coordinates, the generated \texttt{action\_target} may not match the provided coordinates, or the model may choose a different action than the annotation. We therefore add a dedicated filtering stage to improve action--reasoning alignment and overall data quality.

\begin{figure*}[t]
    \centering
    \includegraphics[width=1\linewidth,  trim=0pt 8pt 1pt 5pt,
  clip]{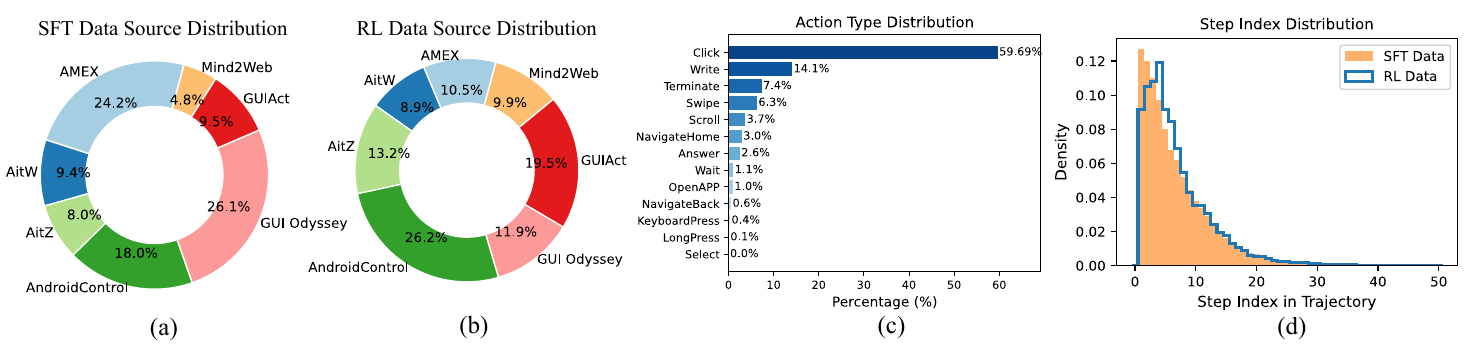}
  \vspace{-15pt}
    \caption{(a)(b) Data source distribution for SFT and RL. (c) Action type distribution of  \textsc{GUI-Libra-81K}.
(d) Comparison of step index distributions between our SFT and RL datasets.}
    \label{fig:data_statics}
\end{figure*}

\subsubsection{Data Filtering for SFT}\label{sec:sft_data_filter}
Human-collected and automatically labeled GUI trajectories are inevitably noisy \citep{yang2025gta1,xu2025aguvis}, including incorrect action types and inaccurate coordinates.
To improve data quality, we curate our SFT data using a two-step automatic filtering pipeline.

\textbf{(1) Agreement filtering via action re-prediction.}
We run Qwen3-VL-8B-Instruct for 10 stochastic runs on each input and measure how often the predicted action matches the annotation
(e.g., exact match on \texttt{action\_type} and coordinate proximity for \texttt{Click}-like actions).
We discard steps with re-prediction accuracy below $0.3$, which effectively removes uncertain or low-quality samples.

\textbf{(2) Coordinate alignment via bounding-box verification.}
We leverage \texttt{action\_target} to verify whether the original coordinate actually corresponds to the intended UI element, and to obtain bounding-box supervision for RL.
Concretely, we prompt Qwen3-VL-32B-Instruct to predict a bounding box given the current screenshot and \texttt{action\_target}, and keep a step only if the original \texttt{point\_2d} falls inside the predicted box. This filtering removes coordinate errors and reduces reasoning--action mismatch, while also providing reliable bounding-box annotations, often missing from prior datasets~\citep{xu2025aguvis}, for subsequent RL training.

\paragraph{\textbf{SFT Dataset Statistics.}}
After filtering, we obtain 81K SFT steps originating from 9K trajectories.
Figure~\ref{fig:data_statics}(a) shows the source distribution: most data comes from mobile datasets (e.g., AndroidControl, GUI-Odyssey, AMEX), while only 14.3\% comes from the web domain.
This reflects the current ecosystem where large-scale open mobile interaction data are more prevalent; scaling high-quality web trajectories remains an important direction.
Figure~\ref{fig:data_statics}(b) shows the action distribution: \texttt{Click} accounts for around 60\% of steps, followed by \texttt{Write}, \texttt{Terminate}, and \texttt{Swipe}, while \texttt{LongPress} and \texttt{Select} are rare.
This imbalance makes rare actions difficult to learn from SFT alone, motivating a subsequent RL stage.

\subsubsection{Data Filtering for RL}\label{sec:rl_data_filter}
For RL, we prioritize a more balanced training set by reducing biases in both step index and domain.
Specifically, we address two issues: (i) \textbf{early-step bias}, where many trajectories share similar initial screens and actions (e.g., mobile home screens or common web landing pages), and (ii) \textbf{domain imbalance}, where mobile trajectories dominate the pool. To mitigate these effects, we downsample early steps (small step indices) and further downsample mobile-domain trajectories, resulting in a \textbf{40K}-step dataset for RL training. Figure~\ref{fig:data_statics}(c) compares the step-index distributions, while Figure~\ref{fig:data_statics}(a) compares the domain distributions. Overall, the RL subset is substantially more balanced than the SFT dataset.

\section{GUI-Libra}\label{sec:gui_pivot_method}
In this section, we introduce GUI-Libra, a native GUI agent with enhanced reasoning capabilities. Using our curated dataset, we first conduct a systematic study of SFT with long CoT and its impact on grounding. We then analyze how step-wise RLVR-style training correlates with online performance in GUI navigation.

\begin{figure}[t]
    \centering
    \includegraphics[width=0.85\linewidth]{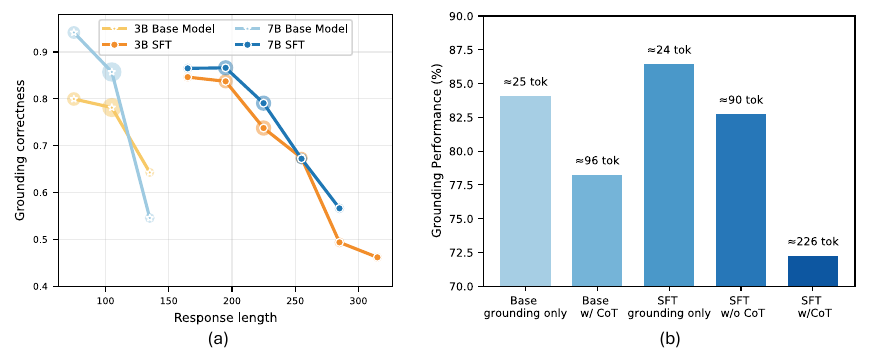}
    \vspace{-10pt}
    \caption{(a) Grounding accuracy on ScreenSpot-v2 versus response length for base models and CoT-SFT models, showing that overly long responses correlate with degraded grounding. (b) Average grounding accuracy under different SFT strategies, where excessively long reasoning traces lead to a substantial drop.}
    \label{fig:cot_sft_hurts_grounding}
\end{figure}

\subsection{SFT with Long CoT Hurts GUI Grounding}
Prior work \citep{lu2025uir1,luo2025guir1} has observed that removing CoT reasoning can improve grounding performance in GUI agents. However, systematic evidence and analysis of this effect, especially under \emph{long} CoT traces, remain limited, since most existing training data contain only short rationales.

Using \textsc{GUI-Libra-81K}, we investigate how response length correlates with grounding performance on ScreenSpot-v2 \citep{wu2025osatlas}. Specifically, we fine-tune Qwen2.5-VL-3/7B-Instruct base models and prompt them to generate reasoning before grounding, following our structured response format. To induce diverse response lengths, we sample outputs with a temperature of 1.0. We group responses into 30-token bins and discard bins with fewer than 20 samples for statistical reliability.
As shown in Figure~\ref{fig:cot_sft_hurts_grounding}(a), grounding accuracy exhibits a clear negative correlation with response length for both base and CoT-SFT models. Longer responses consistently lead to worse grounding performance. Moreover, CoT-based SFT substantially widens the length distribution, producing many responses longer than 250 tokens, which are associated with particularly severe performance drops.

To pinpoint the source of this degradation, we compare three SFT variants on Qwen2.5-VL-3B-Instruct using \textsc{GUI-Libra-81K}: (i) \textit{SFT with CoT}, which uses the full reasoning-then-action outputs; (ii) \textit{SFT without CoT}, which removes reasoning and keeps only the \texttt{<answer>...\texttt{</answer>}} action; and (iii) \textit{Grounding-only}, which predicts coordinates solely from the action-target description. Figure~\ref{fig:cot_sft_hurts_grounding}(b) shows that grounding-only SFT yields a modest gain, while SFT without CoT slightly degrades performance. In contrast, SFT with long CoT traces causes a substantial drop, indicating that the primary driver of grounding degradation is excessively long reasoning sequences.

\subsection{Action-Aware Supervised Fine-Tuning}\label{sec:action_aware_SFT}
Our goal is to build native GUI agents that can both reason and act within a single model. Rather than discarding reasoning traces, we seek to preserve reasoning ability while mitigating the grounding degradation caused by long CoT sequences. To this end, we propose \emph{action-aware supervised fine-tuning} (ASFT), a unified training framework that combines mixed data supervision with token-level reweighting to balance reasoning, action prediction, and grounding.

\paragraph{\textbf{Mixed reasoning and direct-action supervision.}}
ASFT trains on a mixture of data with and without explicit reasoning traces. To construct the direct-action data, we remove the reasoning traces and keep only the structured action output between \texttt{<answer>} and \texttt{</answer>}.
Training on both data variants provides two complementary supervision modes: (i) reasoning-then-action and (ii) direct action prediction.
Similar data mixtures were used for GUI models such as OpenCUA~\citep{wang2025opencua}, as well as LLM agent training~\citep{wang2025think,zhang2025agent}, but \textbf{their role in mitigating grounding degradation has not been clearly demonstrated.}
This dual-mode supervision serves two purposes.
First, it increases the amount of action-centric learning signal, strengthening action prediction that is essential for interactive agents.
Second, it reduces reliance on verbose intermediate reasoning, alleviating grounding degradation induced by long CoT traces.
As a result, the model can flexibly produce either concise direct actions or reasoning-then-action outputs at inference time, improving both grounding accuracy and response efficiency.

\paragraph{\textbf{Action-aware reweighting.}}
In addition to mixed supervision, ASFT further assigns higher weights to action and grounding tokens.
Although grounding is part of the action output in our formulation, the action sequence contains both semantic components (e.g., action description, action type, and value) and spatial components (coordinates), which can be weighted differently at the token level.

Concretely, we treat tokens inside \texttt{<answer>}...\texttt{</answer>} as the action output, and further split them into \emph{action tokens} (all tokens excluding the \texttt{point\_2d} field) and \emph{grounding tokens} (tokens associated with \texttt{point\_2d} field).
Let $c_t$, $a_t$, and $g_t$ denote the reasoning, action, and grounding tokens at step $t$, respectively, and let $x_t=(\ell,h_t,o_t)$ denote the conditioning context.
We denote the mixed training set by $D_{\rm mix}$, which contains both reasoning-then-action and direct-action samples; for direct-action samples we set $c_t$ to an empty sequence.
Under this unified representation, the ASFT objective is
\begin{equation}
\mathcal{L}_{\text{ASFT}}(\theta)
=
- \mathbb{E}_{(x_t, c_t, a_t, g_t) \sim D_{\rm mix}}
\frac{
\log \pi_\theta(c_t \mid x_t)
+ \alpha_a \log \pi_\theta(a_t \mid x_t, c_t)
+ \alpha_g \log \pi_\theta(g_t \mid x_t, c_t, a_t)
}{
|c_t| + \alpha_a |a_t| + \alpha_g |g_t|
},
\end{equation}
where $\alpha_a$ and $\alpha_g$ control the relative importance of action and grounding tokens. By adjusting these coefficients, ASFT recovers several common training strategies as special cases: 
(i) $\alpha_a=\alpha_g=1$ reduces to standard SFT;
(ii) $\alpha_a=\alpha_g\!\gg\!1$ emphasizes action/grounding tokens and approximates CoT-free SFT;
and (iii) $\alpha_g\!\gg\!\alpha_a$ with $\alpha_g\!\gg\!1$ further reduces to grounding-only SFT.
Overall, ASFT provides a flexible mechanism for balancing reasoning, action, and grounding during supervised fine-tuning of native GUI agents.

\subsection{Reinforcement Learning from Partial Verifiable Rewards}
\label{sec:RLPR}

Applying RLVR-style training to optimize step-wise action or coordinate correctness for GUI agents has been explored in prior work \citep{luo2025guir1,lu2025uir1,zhou2025guig1,yang2025gta1}. However, multi-step GUI navigation differs from standard RLVR settings in two key ways. (i) \emph{Errors accumulate and induce distribution shift}: small step mistakes compound over time and change the distribution of states the agent visits. (ii) \emph{Rewards are partially verifiable}: at each step, multiple actions can correctly advance the task, yet offline supervision typically provides and verifies only a single demonstrated action. We show that both factors are crucial for understanding when offline step-wise metrics can reliably predict online task success.

\paragraph{Setup.}
For ease of analysis, we adopt a finite-horizon MDP that is consistent with our earlier goal-conditioned POMDP formulation.
Specifically, we consider a goal-conditioned MDP $\mathcal{M}=(\mathcal{S},\mathcal{A},P,H)$, where the instruction $\ell$ is included in the state, and partial observability is handled by treating the agent’s history (or belief state) as the effective state.
A policy $\pi$ induces a sequence of state visitation distributions $\{d_{\pi,t}\}_{t=1}^{H}$, where $d_{\pi,t}(s)$ is the probability of visiting state $s$ at step $t$.
For each state $s$, let $\mathcal{A}^*(s)\subseteq\mathcal{A}$ denote the set of \emph{valid} actions that can correctly advance the task.
Offline supervision provides only a single demonstrated action $\tilde a(s)\in \mathcal{A}^*(s)$.

\begin{definition}{Partially Verifiable Reward}{}
\label{def:pvr}
For each state $s$, the offline dataset provides a single demonstrated action $\tilde a(s)\in \mathcal{A}^*(s)$.
The step-wise reward induced by offline verification is
\[
\tilde r(s,a) \triangleq \mathbf{1}\{a=\tilde a(s)\}.
\]
We call $\tilde r$ \emph{partially verifiable} if
\[
\tilde r(s,a)=1 \Rightarrow a\in \mathcal{A}^*(s)
\quad\text{but}\quad
\tilde r(s,a)=0 \not\Rightarrow a\notin \mathcal{A}^*(s),
\]
i.e., positive feedback is reliable, while negative feedback is ambiguous because $\mathcal{A}^*(s)$ often contains valid actions beyond $\tilde a(s)$.
\end{definition}

\paragraph{Offline vs.\ online metrics.}
Offline evaluation typically measures one-step action matching on a fixed state distribution $d_\mu$, the marginal induced by an expert dataset $D_\mu$.
We define the offline score as
\begin{equation}
\label{eq:offline_metric}
M_{\text{off}}(\pi)
\triangleq
\mathbb{E}_{s\sim d_\mu}\big[\pi(\tilde a(s)\mid s)\big]
=
\mathbb{E}_{(s,\tilde a)\sim D_\mu}\;\mathbb{E}_{a\sim\pi(\cdot\mid s)}\big[\mathbf{1}\{a=\tilde a\}\big].
\end{equation}
Online evaluation measures trajectory-level task success under closed-loop interaction. We define the probability that $\pi$ completes the task within horizon $H$ as: $J(\pi)\triangleq \Pr_{\tau\sim \pi,P}\!\big(\text{success}(\tau)=1\big)$.

\paragraph{Two quantities controlling predictability.}
We characterize when $S_{\text{off}}(\pi)$ is predictive of $J(\pi)$ using two factors:
(i) \emph{occupancy mismatch} between the online state distribution induced by $\pi$ and the offline distribution $d_\mu$, and
(ii) \emph{step-wise ambiguity} due to partial verifiability.
Formally, define the occupancy mismatch coefficient
\begin{equation}
\label{eq:occupancy_alpha_defs_main}
C(\pi)\triangleq \max_{t\in[H]}\sup_{s:\,d_\mu(s)>0}\frac{d_{\pi,t}(s)}{d_\mu(s)},
\end{equation}
and define the \emph{off-demo validity mass} of $\pi$ at state $s$ as
\begin{equation}
\eta_\pi(s)\triangleq \pi\!\left(\mathcal{A}^*(s)\setminus\{\tilde a(s)\}\mid s\right),
\qquad
\bar\eta_\pi\triangleq \mathbb{E}_{s\sim d_\mu}[\eta_\pi(s)].
\end{equation}
Note that the true step-wise validity probability satisfies
\[
\pi\!\left(\mathcal{A}^*(s)\mid s\right)=\pi\!\left(\tilde a(s)\mid s\right)+\eta_\pi(s),
\]
since $\mathcal{A}^*(s)$ may contain valid actions beyond the single demonstrated action $\tilde a(s)$, which are not credited by offline matching.

\begin{assumption}{}{}
\label{ass:failure}
If an episode fails under policy $\pi$, then there exists at least one step $t\le H$ such that $a_t\notin \mathcal{A}^*(s_t)$.
\end{assumption}

\begin{theorem}{Offline-to-online bound under partial verifiability}{off2on_main}
Assume Assumption~\ref{ass:failure} and that, for all $t\in[H]$, $d_{\pi,t}(s)>0$ implies $d_\mu(s)>0$ (i.e., $\mathrm{supp}(d_{\pi,t})\subseteq \mathrm{supp}(d_\mu)$).
This condition ensures the occupancy ratio $C(\pi)$ is well-defined.
Then the online success probability satisfies
\begin{equation}
\label{eq:main_bound_1}
J(\pi)\ \ge\
1\;-\;H\cdot C(\pi)\cdot\Big(1-M_{\mathrm{off}}(\pi)-\bar\eta_\pi\Big).
\end{equation}
In particular, if $C(\pi)$ is uniformly bounded over a policy class and $\bar\eta_\pi$ is small \emph{or stable across policies},
then $M_{\mathrm{off}}(\pi)$ becomes predictive of $J(\pi)$ through the affine lower bound in Eq~\ref{eq:main_bound_1}.
\end{theorem}
\paragraph{Takeaway.}
Theorem~\ref{thm:off2on_main} shows that offline-to-online predictability is governed by two factors:
(1) \emph{distribution shift} captured by $C(\pi)$ and
(2) \emph{non-identifiability under partial verifiability} captured by the unobserved off-demo validity mass $\bar\eta_\pi$.
Thus, offline one-step matching can be a poor proxy for online success when either the policy drifts to states outside the offline support or the probability mass over valid actions shifts from the demonstrated action to other valid alternatives, changing $M_{\mathrm{off}}(\pi)$ without reflecting true step validity.
A detailed proof and discussion are deferred to Appendix~\ref{app:RLPR}.

\subsubsection{Why Standard RLVR is Easier to Predict?}

\begin{corollary}{Fully verifiable, single-step RLVR}{}
\label{cor:rlvr_easy}
Suppose $H=1$ and the reward is fully verifiable, i.e., $\mathcal{A}^*(s)=\{\tilde a(s)\}$ for all $s$ (hence $\eta_\pi(s)\equiv 0$).
More generally, for $H=1$, Eq~\ref{eq:main_bound_1} reduces to
\[
J(\pi)\ge 1-C(\pi)\big(1-M_{\mathrm{off}}(\pi)\big),
\]
indicating substantially tighter offline-to-online alignment.
\end{corollary}

Standard RLVR is often easier to analyze and predict because it is typically single-step ($H=1$) and fully verifiable ($\eta_\pi\equiv 0$): one-step matching directly reflects true correctness and there is no error accumulation over time.
In contrast, for multi-step GUI agents, distribution shift across steps and partial verifiability jointly weaken the link between offline matching and online success.
This motivates methods that explicitly address both state-distribution shift and reward ambiguity in RL for long-horizon GUI navigation.

\subsubsection{KL Regularization Improves Predictability}
\label{sec:kl_predictability_main}
While many RLVR pipelines omit KL regularization for efficiency \citep{yu2025dapo,liu2025understanding,zhou2025guig1,yang2025gta1},
we find it is crucial in partially verifiable, multi-step GUI settings.
Intuitively, a KL trust region constrains policy drift, which in turn helps control the two quantities governing offline-to-online predictability in Theorem~\ref{thm:off2on_main}:
the occupancy mismatch $C(\pi)$ and the off-demo validity mass $\bar\eta_\pi$.

\paragraph{KL-induced bounds for occupancy mismatch and off-demo validity mass (informal).}
Let $\pi_{\mathrm{ref}}$ be a reference policy (e.g., the SFT initialization trained on demonstrations) and assume a per-state KL constraint
$\mathrm{KL}\!\left(\pi(\cdot\mid s)\,\|\,\pi_{\mathrm{ref}}(\cdot\mid s)\right)\le \varepsilon$ for all $s$.
Under this constraint, the state visitation distribution induced by $\pi$ cannot drift too far from that of $\pi_{\mathrm{ref}}$.
In particular, if the offline distribution has a positive lower bound on its support, $\rho \triangleq \inf_{s:\,d_\mu(s)>0} d_\mu(s) \;>\; 0$, and $d_{\pi_{\mathrm{ref}},t}\ll d_\mu$ for all $t\in[H]$, then the occupancy mismatch is controlled as
\begin{equation}
\label{eq:C_growth_main}
C(\pi) \;\le\; C(\pi_{\mathrm{ref}})+\frac{H\sqrt{2\varepsilon}}{\rho}.
\end{equation}
KL regularization also limits how much probability mass can move away from the demonstrated action.
If the reference policy is demo-concentrated, i.e., $\pi_{\mathrm{ref}}(\tilde a(s)\mid s)\ge 1-\delta(s)$, then the off-demo validity mass satisfies
\begin{equation}
\label{eq:eta_bar_bound_main}
\bar\eta_\pi \;\le\; \bar\delta+\sqrt{\varepsilon/2},
\qquad
\bar\delta \triangleq \mathbb{E}_{s\sim d_\mu}[\delta(s)].
\end{equation}
Full statements and proofs are provided in Appendix~\ref{ap:KL_improves_predictability}.
The above bounds provide a principled explanation for why KL-regularized policy optimization improves predictability:
a KL trust region simultaneously limits state-distribution shift and constrains how much probability mass can move away from the demonstrated action.
As a result, KL-regularized RL keeps training in a regime where the offline matching score $M_{\mathrm{off}}(\pi)$ remains a more stable proxy for the online success rate $J(\pi)$.

\subsection{Success-adaptive Negative Gradient Scaling}
\label{sec:sngs}

Under partial verifiability, the step-wise reward
provides reliable positive feedback, while $\tilde r(s,a)=0$ is ambiguous:
it conflates truly invalid actions with valid-but-uncredited alternatives.
Consequently, treating every non-match as equally negative can produce biased and overly aggressive updates, pushing the policy to overfit the demonstrator’s particular choice.
To address this issue, we propose \emph{success-adaptive negative gradient scaling (SNGS)}, which conservatively downweights gradients induced by ambiguous ``negative'' outcomes.
Importantly, negative updates in policy-gradient methods such as GRPO remain useful for stabilizing training and avoiding premature collapse \citep{zhu2025surprising}.
Therefore, rather than suppressing all negative gradients equally, SNGS rescales them using a state-conditioned reliability signal estimated from the GRPO sampling group.
Concretely, GRPO samples a group of $G$ candidate actions for the same state and computes group-relative advantages.
Let $\{(a_k, \tilde r_k)\}_{k=1}^{G}$ denote a GRPO group at state $s$, where
$\tilde r_k=\mathbf{1}\{a_k=\tilde a(s)\}\in\{0,1\}$.
We define the empirical group success rate as $\hat p_g(s) \;\triangleq\; \frac{1}{G}\sum_{k=1}^{G} \tilde  r_k $,
which measures how concentrated the current policy is on the demonstrated action.

We introduce a scaling factor $\lambda_g(s)$ that rescales only the \emph{negative} advantages. 
\begin{equation}
\label{eq:lambda_def}
\lambda_g(s) \;\triangleq\; \min \left( \lambda_0 + \kappa\,\hat p_g(s),\;  1 \right).
\end{equation}
Here $\lambda_0$ is an offset, and $\kappa$ controls how $\lambda_g$ varies with $\hat p_g$. With $\kappa>0$, $\lambda_g$ increases with $\hat p_g$: as the policy becomes more concentrated on $\tilde a(s)$, non-matching samples are more likely to be genuinely incorrect, so we downweight negative gradients less and gradually recover the standard GRPO update as $\lambda_g\to 1$. With $\kappa<0$, $\lambda_g$ decreases with $\hat p_g$, making updates more conservative for high-success groups. In our experiments, we find $\kappa>0$ works well in most settings and use it as the default.

Let $A_k$ denote the GRPO advantage for sample $k$.
SNGS modifies only the negative advantages:
\begin{equation}
\label{eq:sngs_adv}
\tilde A_k \;\triangleq\;
\begin{cases}
A_k, & A_k \ge 0,\\[2pt]
\lambda_g (s)\,A_k, & A_k < 0.
\end{cases}
\end{equation}
We then replace the advantage term in the GRPO objective (Eq.~\ref{eq:grpo_kl}) with $\tilde A_k$.
As a result, SNGS preserves positive learning signals corresponding to reliably verified matches, while attenuating updates driven by potentially ambiguous negatives.
This reduces over-penalization of valid alternatives and leads to more robust policy optimization under partial verification.

\subsection{Reward Function Implementation}
Each rollout produces a structured prediction string $y$ following our output structure:
\[
y=\texttt{<think>}\cdots\texttt{</think><answer>}a\texttt{</answer>},
\]
where the \texttt{<answer>} block contains a structured action $a$ that can be parsed into a JSON object
\[
a=\{\texttt{action\_type},\ \texttt{action\_description},\ \texttt{value},\ \texttt{point\_2d}\},
\]
with $\texttt{point\_2d}\in\mathbb{R}^2$ (or \texttt{"none"} when not applicable). We implement two automated verifiers: a \emph{format verifier} and an \emph{accuracy verifier}.
The resulting step-wise reward is a weighted sum of the two:
\begin{equation}
\label{eq:step_reward}
\tilde r(s,a)
\;=\;
w_{\text{fmt}}\, r_{\text{fmt}}
\;+\;
(1 - w_{\text{fmt}})\, r_{\text{acc}},
\qquad
w_{\text{fmt}}\in[0,1],
\end{equation}
where $r_{\text{fmt}}$ checks output validity and $r_{\text{acc}}$ scores action correctness. We set $w_{\text{fmt}}=0.1$ so that our reward mainly focus on action correctness.

\noindent
\textbf{Format reward.}
The format reward $r_{\text{fmt}}$ is $1$ if the output contains valid \texttt{<think>} and \texttt{<answer>} tags and the \texttt{<answer>} block can be parsed into the required JSON schema; otherwise $r_{\text{fmt}}=0$.

\noindent
\textbf{Accuracy reward.}
The accuracy reward $r_{\text{acc}}$ evaluates semantic correctness of the predicted action: $r_{\text{acc}}
\;=\;
r_{\text{act}}
\cdot
r_{\text{val}}
\cdot
r_{\text{g}}$,
where each component is computed as follows:
(1) \textbf{Action-type reward} $r_{\text{act}}$ checks whether $\texttt{action\_type}$ matches the demonstrated action type.
(2) \textbf{Value reward} $r_{\text{val}}$ compares the predicted $\texttt{value}$ $v$ with the demonstrated value $v^\star$ using word-level F1,
and sets $r_{\text{val}}=1$ if $\mathrm{F1}(v, v^\star) > 0.5$.
(3) \textbf{Grounding reward} $r_{\text{g}}$ evaluates point grounding by checking whether the predicted point
$\mathbf{u}$ falls inside the demonstrated bounding box $b^\star$, i.e.,
$r_{\text{g}}=\mathbf{1}\{\mathbf{u}\in b^\star\}$.

Together, these verifiers yield a step-wise signal: positive rewards indicate reliably correct predictions, whereas low rewards may arise from either incorrect actions or valid but uncredited alternatives. This design matches the partial-verifiability setting analyzed in Sec.~\ref{sec:RLPR}.

\begin{figure*}[t]
    \vspace{-0.0in}
    \centering
    \includegraphics[width=1\linewidth]{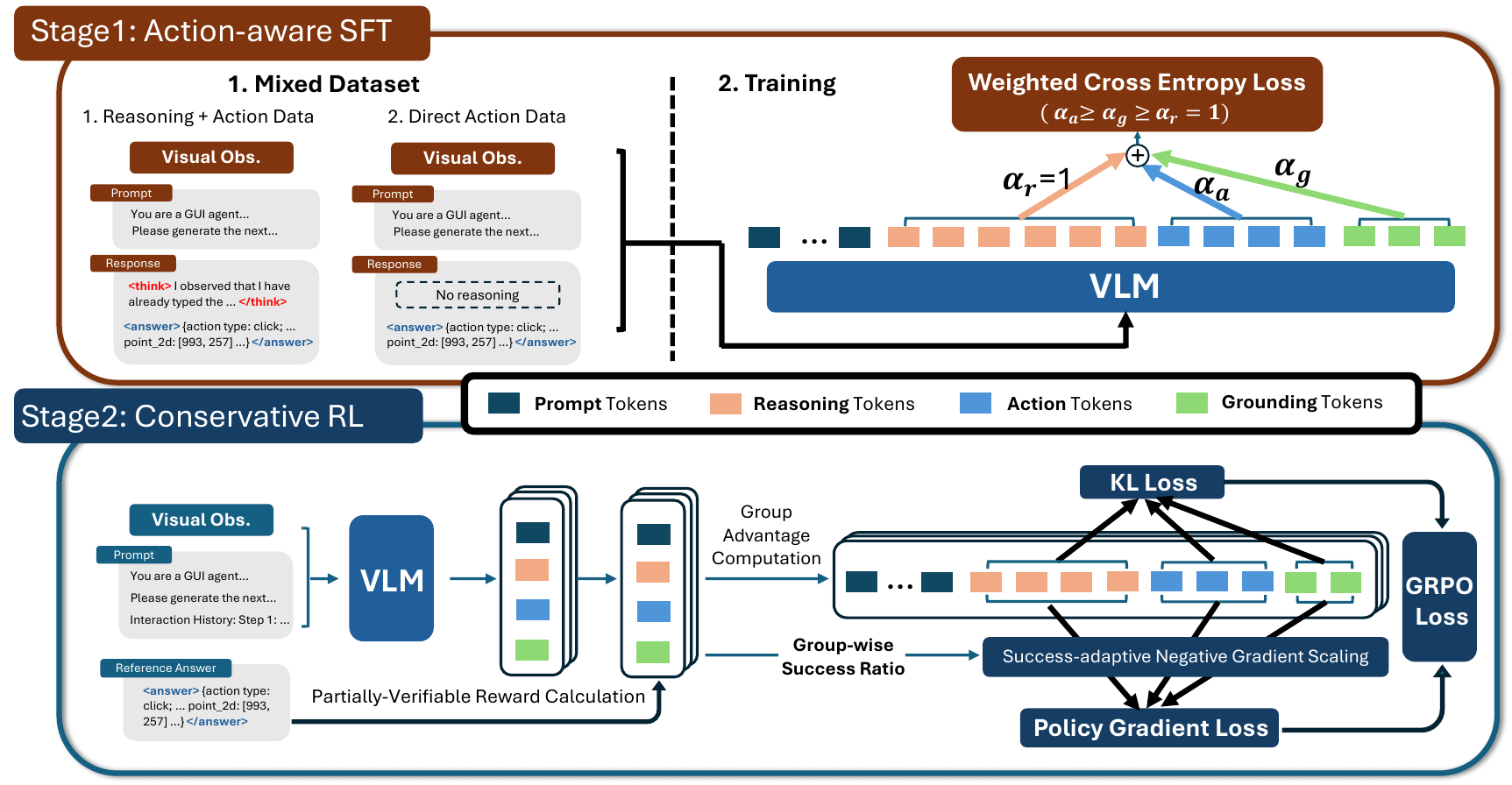}
    \caption{Overall training framework of GUI-Libra: Stage 1 applies action-aware SFT with mixed supervision and token reweighting; Stage 2 performs KL-regularized GRPO with success-adaptive negative gradient scaling.}
    \label{fig:training_framework}
    \vspace{-0.in}
\end{figure*}

\subsection{Overall Training Framework for GUI-Libra}
\label{sec:overall_training_framework}

Figure~\ref{fig:training_framework} summarizes the overall training framework of GUI-Libra.
Based on our augmented and filtered datasets, GUI-Libra consists of two stages: 
\begin{itemize}

\item In the SFT stage, we apply ASFT to equip the base model with \emph{action-aligned reasoning} and mitigate grounding degradation caused by long CoT. ASFT mixes reasoning-then-action and direct-action supervision, and uses an action-aware reweighted objective that emphasizes action and grounding tokens.

\item In the RL stage, we further optimize the policy with conservative GRPO under partially verifiable step-wise rewards.
To stabilize learning and improve offline-to-online predictability, we adopt a conservative RL design with two components:
    (i) \emph{KL regularization} to constrain distribution shift and the effect of ambiguous rewards, and
    (ii) \emph{success-adaptive negative gradient scaling} to downweight unreliable negative updates caused by valid-but-uncredited alternatives.
\end{itemize}

Overall, GUI-Libra promotes step-wise improvements that are \emph{behaviorally meaningful}: gains in offline action matching are more likely to translate into better decisions along the policy’s own trajectories. In addition, by leveraging partially verifiable offline feedback, our framework enables scalable optimization on large static datasets without requiring costly online interaction during training.

%% file: sections/experiments.tex
\section{Experiments}
In this section, we evaluate GUI-Libra on a diverse set of offline and online GUI navigation benchmarks. Beyond overall results, we also study the impact of our key design choices in both the SFT and RL stages, and examine when offline step-wise metrics can reliably predict online task success for GUI agents.

\begin{figure*}[t]
    \centering
    \includegraphics[width=1\linewidth]{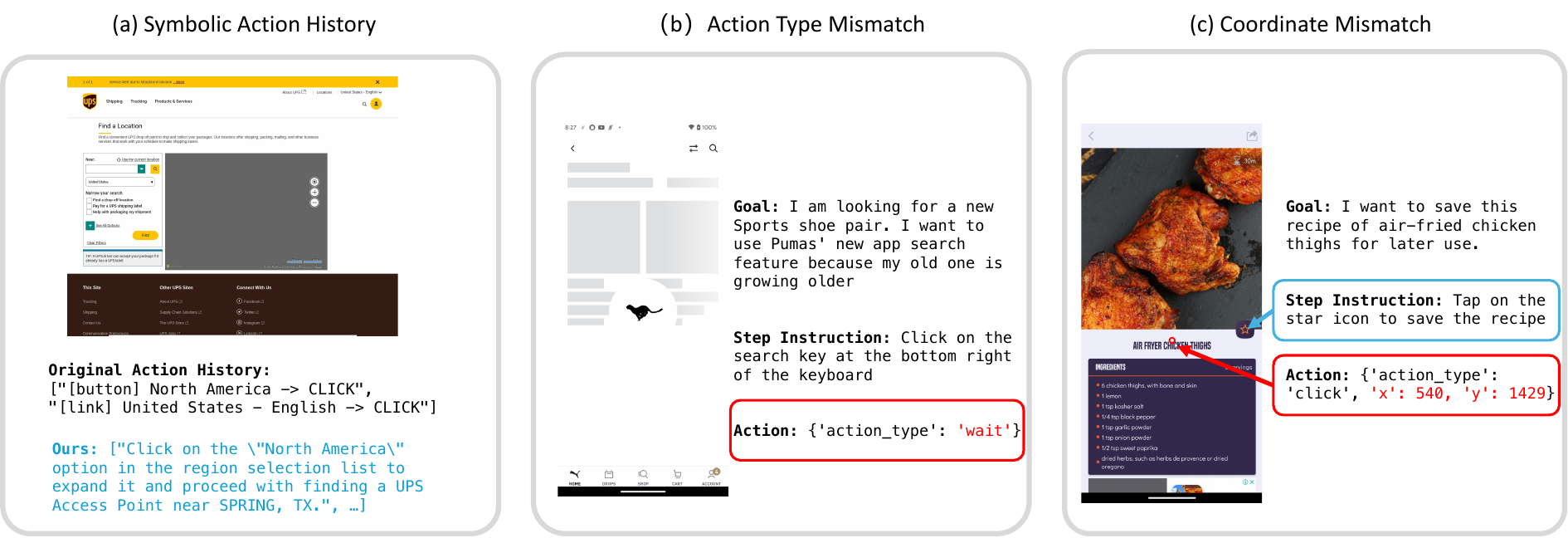}
    \caption{Limitations of current offline benchmarks. (a) Symbolic action history not in natural language in MM-Mind2Web, (b) Action type mismatch and (c) coordinate mismatch in AndroidControl.}
    \label{fig:limitations_of_MMM2Web_AC}
\end{figure*}

\input{tables/offline_benchmarks}

\subsection{Experimental Setups}

\paragraph{\textbf{GUI-Libra Details.}} We train GUI-Libra models from Qwen2.5-VL-3B/7B-Instruct~\citep{bai2025qwen2} and Qwen3-VL-4B/8B-Instruct~\citep{bai2025qwen3vltechnicalreport}. We use GUI-Libra-81K for SFT and a downsampled 40K subset for RL. For \textbf{SFT}, we use a learning rate of $1\times10^{-5}$ with an effective batch size of 256, and set ASFT weights to $\alpha_a=2$ and $\alpha_g=4$ by default. To ensure fair comparison, we train baselines on GUI-Libra-81K for two epochs, while models trained with mixed reasoning and direct-action data (double size) for one epoch.
Notably, our SFT corpus is substantially smaller than those used in recent GUI models~\citep{yang2025ferretuilite,liu2025scalecua} and we do not include any direct grounding-only data (e.g., low-level instructions paired with coordinate supervision), focusing on GUI reasoning and multi-step navigation.
For \textbf{RL}, we use a learning rate of $1\times10^{-6}$, rollout batch size 256, group size 8, and KL coefficient 0.005 (7B) or 0.001 (others). While SNGS can improve performance, it is sensitive to hyperparameters; therefore, for ablations unrelated to SNGS, we use KL-regularized GRPO to isolate the effects of the other components. Additional implementation details are provided in Appendix~\ref{ap:implementation_details}.

\paragraph{\textbf{Evaluation Benchmarks.}} We evaluate models on both offline and online benchmarks. For \textbf{offline evaluation}, we follow UGround~\citep{gou2025uground} but substantially refine the underlying datasets to improve annotation quality and realism. As illustrated in Figure~\ref{fig:limitations_of_MMM2Web_AC}, the original MM-Mind2Web~\citep{zheng2024seeact} uses symbolic action histories that do not reflect real-world usage, while AndroidControl~\citep{li2024androiedcontrol} contains roughly 20\% errors in action types and coordinates. To address these issues, \textbf{we enhance AndroidControl and MM-Mind2Web by correcting label errors and translating non-natural symbolic action histories, yielding AndroidControl-v2 and Multimodal-Mind2Web-v2 (MM-Mind2Web-v2)}, respectively. We report step success rate, which requires the predicted action type, textual value, and coordinates to be correct, and include both Pass@1 and Pass@4 step accuracy. For AndroidControl-v2, following UGround~\citep{gou2025uground}, we evaluate on 398 filtered samples with both high-level and low-level instructions.
For \textbf{online evaluation}, we use AndroidWorld~\citep{rawles2025androidworld}, WebArena-Lite-v2~\citep{liu2025scalecua}, and Online-Mind2Web~\citep{online-mind2web}, which assess agents in realistic interactive environments. Notably, Online-Mind2Web is evaluated on live websites, introducing additional real-world variability and complexity. We follow the official protocols and report task success rate as the primary metric, with a maximum of 20 steps for AndroidWorld, 15 for WebArena-Lite-v2, and 30 for Online-Mind2Web. Additional details are provided in Appendix~\ref{ap:benchmark_details}.

\paragraph{\textbf{Baselines.}} We compare GUI-Libra series against a diverse set of native GUI agents, including Qwen2.5-VL-3/7/32/72B~\citep{bai2025qwen2}, Qwen3-VL-4/8/32B~\citep{bai2025qwen3vltechnicalreport}, Aguvis-7B~\citep{xu2025aguvis}, UI-TARS-1.5-7B~\citep{qin2025uitars}, GLM-4.1-V-9B-Thinking~\citep{hong2025glm}, and GUI-R1-3/7B~\citep{luo2025guir1}. We also evaluate proprietary models paired with a grounding module, following UGround~\citep{gou2025uground}, including GPT-4o, GPT-4.1, GPT-5-mini, and GPT-5, and include reported/reproduced results from ScaleCUA~\citep{liu2025scalecua} on the two online benchmarks. Because evaluation pipelines can substantially affect reported performance and are often unreleased by previous work, we make our best effort to evaluate all models under a unified and consistent protocol for fair comparison.

\begin{figure*}[t]
    \vspace{-0.0in}
    \centering
    \includegraphics[width=0.95\linewidth]{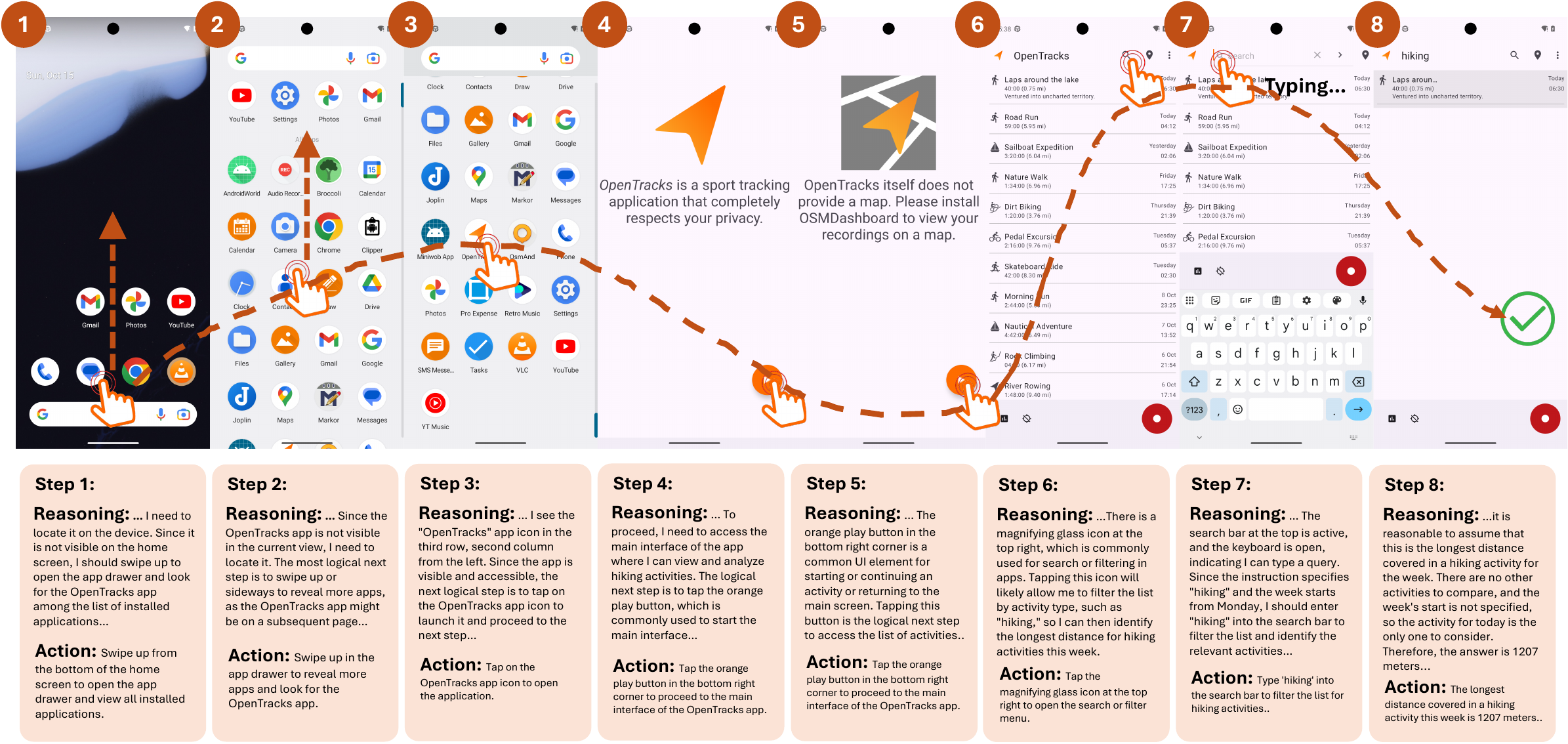}
    \caption{Trajectory Example of GUI-Libra-7B on AndroidWorld.}
    \label{fig:trajectory_example}
    \vspace{-0.1in}
\end{figure*}

\subsection{Performance on Offline and Online GUI Navigation Benchmarks}

\subsubsection{Offline Benchmarks}
Tables~\ref{tab:androidcontrol_v2_accuracy} and~\ref{tab:mind2web_v2_accuracy} report step-wise accuracy on AndroidControl-v2 and MM-Mind2Web-v2, comparing GUI-Libra with open-source native GUI models and proprietary systems using the SeeAct-V~\citep{zheng2024seeact} framework. Overall, the GUI-Libra series achieves the \textbf{best Pass@1 performance} on both benchmarks, outperforming not only similarly sized models but also several substantially larger open-source and proprietary models. Importantly, GUI-Libra consistently improves over its corresponding base models. For example, GUI-Libra-3B improves Pass@1 over Qwen2.5-VL-3B by $+20.9$ and $+14.8$ on AndroidControl-v2 high-level and low-level tasks, respectively, and by $+19.3$ on the average Pass@1 of MM-Mind2Web-v2. We also observe a clear scaling trend for both Qwen baselines and GUI-Libra models, indicating that larger models have greater potential to achieve strong offline decision-making performance.
In terms of Pass@4, large models (e.g., Qwen2.5-VL-72B and Qwen3-VL-32B) can be competitive, but they rely on substantially more parameters. In contrast, GUI-Libra is more parameter-efficient and consistently outperforms models at similar scale, and even GPT5. For instance, GUI-Libra-3B improves Pass@4 over Qwen2.5-VL-3B by $+16.3$, and $+19.3$ on AndroidControl-v2 (high-level) and MM-Mind2Web-v2, respectively. 

We further find that gains on Qwen3-based backbones are relatively smaller than those on Qwen2.5-based ones, which we attribute to Qwen3’s heavier post-training (especially RL for reasoning) that already strengthens planning and decision-making. Nevertheless, GUI-Libra still provides meaningful improvements: GUI-Libra-4/8B outperforms Qwen3-VL-4/8B by 13.0/9.5 points on Pass@1 of AndroidControl-v2 (high-level) and by 8.8/6.7 points on MM-Mind2Web-v2, demonstrating consistent benefits even with strong pretrained backbones.

\begin{table}[t]
\centering
\caption{Performance on the online benchmark \textbf{AndroidWorld} in 20 steps.
$^*$ denotes numbers reported by original papers.
Left: \emph{Native Models} (single VLM).
Right: \emph{Agent Frameworks} ($\geq$2 VLM modules).}
\label{tab:androidworld_side_by_side}

\begin{subtable}[t]{0.48\linewidth}
\centering
\caption{Native Models}
\resizebox{0.65\linewidth}{!}{
\begin{tabular}{l r}
\toprule
\textbf{Model} & \textbf{Acc.} \\
\midrule
UI-TARS-1.5-7B & 16.5 \\
GLM-4.1V-9B-Thinking & 18.3 \\
Qwen2.5-VL-32B & 29.6 \\
Qwen2.5-VL-72B & 32.2 \\
Qwen3-VL-32B & 34.8 \\
\midrule
Qwen2.5-VL-3B (Baseline) & 3.5 \\
\rowcolor{beijingblue}
\textbf{GUI-Libra-3B (Ours)} & \textbf{25.2} \\
Qwen2.5-VL-7B (Baseline) & 7.8 \\
\rowcolor{beijingblue}
\textbf{GUI-Libra-7B (Ours)} & \textbf{29.6} \\
Qwen3-VL-4B (Baseline) & 27.0 \\
\rowcolor{beijingblue}
\textbf{GUI-Libra-4B (Ours)} & \textbf{42.6} \\
Qwen3-VL-8B (Baseline) & 30.4 \\
\rowcolor{beijingblue}
\textbf{GUI-Libra-8B (Ours)} & \textbf{42.6} \\
\bottomrule
\end{tabular}}
\end{subtable}
\hfill
\begin{subtable}[t]{0.5\linewidth}
\centering
\caption{Agent Frameworks ($\geq$2 VLM Modules)}
\resizebox{1\linewidth}{!}{
\begin{tabular}{l c r}
\toprule
\textbf{Model} & \textbf{Additional Module} & \textbf{Acc.} \\
\midrule
Qwen2.5-VL-3B & Step-wise Summary & 7.0 \\
Qwen2.5-VL-7B & Step-wise Summary  & 15.7 \\
Qwen3-VL-4B & Step-wise Summary  & 36.5 \\
Qwen3-VL-8B & Step-wise Summary  & 39.1 \\
ScaleCUA-3B$^*$ & Step-wise Summary  & 23.7 \\
ScaleCUA-7B$^*$ & Step-wise Summary  & 27.2 \\
ScaleCUA-32B$^*$ & Step-wise Summary  & 30.6 \\
GLM-4.1V-9B-Thinking & UGround-v1-7B & 20.9 \\
GPT-4o & UGround-v1-7B & 42.6 \\
GPT-4.1 & UGround-v1-7B & 37.4 \\
GPT-5-mini & UGround-v1-7B & 40.9 \\
GPT-5 & UGround-v1-7B & 48.7 \\
\bottomrule
\end{tabular}}
\end{subtable}
\vspace{-5pt}
\end{table}

%%%% START of WebArena-lITE-V2
\begin{table}[ht]
\centering
\caption{Performance comparison on \textbf{WebArena-Lite-v2} in 15 steps. $^*$ denotes numbers reported in \citet{liu2025scalecua}.}
\label{tab:online_evaluation_webarena}
\resizebox{0.97\linewidth}{!}{
\begin{tabular}{l | c c c c c | r}
\toprule
 & \textbf{GitLab} & \textbf{MAP} &\textbf{Reddit} & \textbf{Shopping} & \textbf{ShoppingAdmin} & \textbf{Average} \\
\midrule
\midrule
\multicolumn{7}{l}{\textbf{Native Models}} \\
\midrule
Aguvis-72B$^*$ & - & - & - & - & - & 5.8 \\
Qwen2.5-VL-72B$^*$ & - & - & - & - & - & 15.6\\
InternVL3.5-241B-A28B$^*$ & - & - & - & - & - & 11.7 \\
UI-TARS-1.5-7B$^*$ & - & - & - & - & - & 20.8 \\
UI-TARS-72B-DPO$^*$ & - & - & - & - & - & 23.4 \\
% ScaleCUA-3B$^*$ & - & - & - & - & - & 31.8 \\  % cannot reproduce their results
% ScaleCUA-7B$^*$ & - & - & - & - & - & 37.7 \\
% ScaleCUA-32B$^*$ & - & - & - & - & - & 44.2 \\
\midrule
ScaleCUA-3B & 21.7 & 7.7 & 13.2 & 16.5 & 23.6 & 17.2 \\
ScaleCUA-7B & 28.3 & 15.4 & 27.6 & 18.8 & 30.7 & 23.9 \\
ScaleCUA-32B & 34.2 & 10.6 & 26.3 & 16.5 & 33.6 & 24.0 \\
\midrule
Qwen2.5-VL-3B & 1.7 & 0.0 & 0.0 & 1.7 & 0.0 & 0.8 \\
\rowcolor{beijingblue}
\textbf{GUI-Libra-3B (Ours)}  & 25.8 & 9.6 & 18.4 & 17.6 & 12.1 & 16.7 \\
Qwen2.5-VL-7B & 8.3 & 1.0 & 2.6 & 7.4 & 2.9 & 4.9 \\
\rowcolor{beijingblue}
\textbf{GUI-Libra-7B (Ours)}  & 25.0 & 10.6 & 26.3 & 26.1 & 22.9 & 22.6 \\
Qwen3-VL-4B & 17.5 & 5.8 & 10.5 & 13.1 & 10.7 & 11.9 \\
\rowcolor{beijingblue}
\textbf{GUI-Libra-4B (Ours)} & 29.2 & 10.6 & 34.2 & 30.1 & 17.9 & 24.4 \\
Qwen3-VL-8B & 15.0 & 5.8 & 17.1 & 17.0 & 19.3 & 15.3 \\
\rowcolor{beijingblue}
\textbf{GUI-Libra-8B (Ours)} & 31.7 & 17.3 & 35.5 & 26.1 & 25.0 & \textbf{26.6} \\
\midrule
\midrule
\multicolumn{7}{l}{\textbf{Agent Framework (GPT-4o as the Planner)}} \\
\midrule
GPT-4o + UI-TARS-1.5-7B \citep{qin2025uitars}$^*$ & - & - & - & - & - & 22.6 \\
GPT-4o + UGround-V1-7B \citep{gou2025uground}$^*$ & - & - & - & - & - & 23.2\\
GPT-4o + ScaleCUA-7B \citep{liu2025scalecua}$^*$ & - & - & - & - & - & \textbf{28.6} \\
\bottomrule
\end{tabular}}
\end{table}
%%%% END  of WebArena-Lite-v2

\subsubsection{Online Benchmarks} Tables~\ref{tab:androidworld_side_by_side},~\ref{tab:online_evaluation_webarena}, and~\ref{tab:o4mini_webjudge} report task success rates on AndroidWorld, WebArena-Lite-v2, and Online-Mind2Web, respectively. However, many prior works do not release their evaluation frameworks, which hinders reproducibility and can compromise fair comparison. Even open-source studies such as ScaleCUA~\citep{liu2025scalecua} report results from agent frameworks augmented with additional modules (e.g., step-wise summaries) on AndroidWorld as native agent performance. As shown in Table~\ref{tab:androidworld_side_by_side}, step-wise summaries improve Qwen3-VL-4B and Qwen3-VL-8B by 9.5 and 8.7 points over their native counterparts, respectively. To support a fair comparison, we therefore explicitly report results for true native models and agent frameworks.

On \textbf{AndroidWorld} (Table~\ref{tab:androidworld_side_by_side}), GUI-Libra substantially strengthens native GUI models across all scales. Relative to their corresponding baselines, GUI-Libra yields large and consistent gains: GUI-Libra-3B increases the success rate from $3.5$ to $25.2$ (+21.7) and GUI-Libra-8B from $30.4$ to $42.6$ (+12.2). Notably, GUI-Libra-4B/8B ($42.6$) surpass several much larger native models (e.g., Qwen2.5-VL-32/72B and Qwen3-VL-32B), and also match or outperform multi-module agent frameworks that add external step-wise summary modules. For example, Qwen3-VL-8B with step-wise summary reaches $39.1$, whereas our native GUI-Libra-8B achieves $42.6$. Moreover, GUI-Libra-4B/8B reaches performance comparable to strong proprietary systems such as GPT-4o ($42.6$) and GPT-5-mini ($40.9$) equipped with UGround, despite using a simpler single-VLM architecture. We further provide qualitative examples of GUI-Libra-7B successfully completing AndroidWorld tasks in Figures~\ref{fig:trajectory_example} and Appendix~\ref{ap:long_traj_example}.

On \textbf{WebArena-Lite-v2} (Table~\ref{tab:online_evaluation_webarena}), a \textbf{locally deployed web benchmark (rather than live websites)}, GUI-Libra shows strong generalization across diverse web tasks despite being trained on only 15K web-related samples, far fewer than the web corpora used by many existing GUI models. Even in this low-data regime, GUI-Libra delivers large gains over its base models: GUI-Libra-7B improves the average success rate from $4.9$ to $22.6$, and GUI-Libra-8B increases performance from $15.3$ to $26.6$. These results are competitive with strong proprietary systems such as GPT-4o equipped with UI-TARS and UGround. Moreover, GUI-Libra-8B outperforms large-scale models including ScaleCUA-32B, UI-TARS-72B, and Aguvis-72B, all trained on substantially larger web datasets.

On \textbf{Online-Mind2Web} (Table~\ref{tab:o4mini_webjudge}), which evaluates agents on \textbf{live websites} with real-world variability, we evaluate GUI-Libra using two independent judge models: o4-mini and WebJudge-7B~\citep{online-mind2web}. GUI-Libra consistently improves over its corresponding base models across all difficulty levels. In particular, GUI-Libra-8B increases the average overall score from $19.3$ (Qwen3-VL-8B) to $28.0$, achieving the best result among all evaluated native models, including those with substantially more parameters. Similarly, GUI-Libra-7B improves from $15.8$ (Qwen2.5-VL-7B) to $25.5$, GUI-Libra-4B from $21.7$ (Qwen3-VL-4B) to $25.7$. Even at the 3B scale, GUI-Libra-3B achieves an average overall of $21.3$, a notable leap from $4.8$ for Qwen2.5-VL-3B. 
Notably, \emph{while the ScaleCUA family performs competitively on locally deployed benchmarks, its performance is less competitive on live websites}: ScaleCUA-7B and ScaleCUA-32B reach only $23.7$ and $23.5$ average overall, both of which are surpassed by GUI-Libra-4/7/8B. Together, these results suggest that \textbf{GUI-Libra not only closes the gap between smaller open-source models and larger agents, but also provides stronger robustness and generalization in realistic, dynamically changing web environments.}

Overall, GUI-Libra delivers consistent gains on online mobile and web benchmarks, generalizing from locally deployed environments to live websites. It matches or surpasses larger models while remaining highly data-efficient, using relatively small training data, especially for the web domain.

\begin{table}[t]
\centering
\caption{Performance comparison on \textbf{Online-Mind2Web} in 30 steps using o4-mini and WebJudge-7B as judges.}
\label{tab:o4mini_webjudge}
\resizebox{0.9\linewidth}{!}{
\begin{tabular}{l | c c c c | c c c c | r}
\toprule
\multirow{2}{*}{\diagbox[width=10em, height=2.1\line]{\textbf{Model}}{\textbf{Judge}}}
& \multicolumn{4}{c|}{\textbf{o4-mini}}
& \multicolumn{4}{c|}{\textbf{WebJudge-7B}}
& \textbf{Avg.} \\
 & \textbf{Easy} & \textbf{Medium} & \textbf{Hard} & \textbf{Overall}
 & \textbf{Easy} & \textbf{Medium} & \textbf{Hard} & \textbf{Overall} & \textbf{Overall} \\
\midrule
\midrule
\multicolumn{10}{c}{\textbf{Agent Framework}} \\
\midrule
GPT-4o + UGround-v1-7B & 38.8 & 14.0 & 6.5 & 18.7 & 45.0 & 27.3 & 19.5 & 30.0 & 24.3 \\
GPT-4.1 + UGround-v1-7B & 41.3 & 21.7 & 5.2 & 22.7 & 47.5 & 35.0 & 28.6 & \textbf{36.7} & \textbf{29.7} \\
GPT-5 + UGround-v1-7B & 40.0 & 24.5 & 14.3 & \textbf{26.0} & 45.0 & 31.5 & 24.7 & 33.3 & \textbf{29.7} \\
\midrule
\multicolumn{10}{c}{\textbf{Native Models}} \\
\midrule
Qwen2.5-VL-32B & 12.5 & 7.7 & 1.3 & 7.3 & 28.8 & 14.7 & 19.5 & 19.7 & 13.5 \\
Qwen3-VL-32B   & 33.8 & 17.5 & 7.8 & 19.3 & 45.0 & 31.5 & 28.6 & 34.3 & 26.8 \\
ScaleCUA-3B  & 30.0 & 4.9 & 2.6 & 11.0 & 37.5 & 18.2 & 11.7 & 21.7 & 16.3 \\
ScaleCUA-7B  & 33.8 & 14.7 & 3.9 & 17.0 & 47.5 & 27.3 & 18.2 & 30.3 & 23.7 \\
ScaleCUA-32B & 31.3 & 14.7 & 6.5 & 17.0 & 43.8 & 29.4 & 16.9 & 30.0 & 23.5 \\
\midrule
Qwen2.5-VL-3B & 3.8 & 0.0 & 1.3 & 1.3 & 16.3 & 7.7 & 1.3 & 8.3 & 4.8 \\
\rowcolor{beijingblue}
\textbf{GUI-Libra-3B (Ours)} & 28.8 & 9.8 & 5.2 & 13.7 & 47.5 & 21.0 & 24.7 & 29.0 & 21.3 \\
Qwen2.5-VL-7B & 22.5 & 7.7 & 0.0 & 9.7 & 36.3 & 18.9 & 13.0 & 22.0 & 15.8 \\
\rowcolor{beijingblue}
\textbf{GUI-Libra-7B (Ours)} & 36.3 & 15.4 & 2.6 & 17.7 & 47.5 & 30.8 & 23.4 & 33.3 & 25.5 \\
Qwen3-VL-4B & 33.8 & 10.5 & 6.5 & 15.7 & 43.8 & 23.8 & 18.2 & 27.7 & 21.7 \\
\rowcolor{beijingblue}
\textbf{GUI-Libra-4B (Ours)} & 36.3 & 18.2 & 6.5 & \textbf{20.0} & 45.0 & 30.1 & 19.5 & 31.3 & 25.7 \\
Qwen3-VL-8B & 23.8 & 9.8 & 0.0 & 11.0 & 43.8 & 23.1 & 19.5 & 27.7 & 19.3 \\
\rowcolor{beijingblue}
\textbf{GUI-Libra-8B (Ours)} & 31.3 & 17.5 & 10.4 & 19.3 & 42.5 & 37.8 & 28.6 & \textbf{36.7} & \textbf{28.0} \\
\bottomrule
\end{tabular}}
% \vspace{-10pt}
\end{table}

\subsection{Action-aware SFT and RL Mitigate Grounding Performance Degradation}\label{sec:asft_rl_grounding}
Figure~\ref{fig:grounding_length_ASFT} analyzes how grounding performance varies with response length for the base model (Qwen2.5-VL-3B) and SFT variants trained with different strategies. We evaluate on ScreenSpot-v2, group outputs into 30-token length bins, discard bins with fewer than 20 samples, and report the average grounding accuracy within each bin. This setup enables a fine-grained analysis of how increasingly long CoT outputs affect grounding performance. As response length increases, both the base model and standard SFT exhibit a pronounced degradation in grounding accuracy, indicating that long-form reasoning interferes with precise action execution. In contrast, \textbf{action-aware SFT substantially mitigates this degradation across all response lengths}. By incorporating direct-action supervision, action-aware SFT supports both reasoning and no-reasoning modes. Weighted training objectives further stabilize performance under long responses. In particular, stronger weighting strategies preserve high grounding correctness even beyond 250 tokens, significantly outperforming both the base model and standard SFT.

Table~\ref{fig:grounding_asft_rl_comparison} reports overall grounding accuracy and average response length across different models and inference modes. Models trained with mixed data can be flexibly prompted to inference in either reasoning or no-reasoning modes, denoted as \textit{Reason} and \textit{No-Reason}. Across both 3B and 7B scales, \textbf{mixed-data training and action-aware weighting consistently improve average grounding accuracy in both modes}. For example, at the 7B scale, mixed-data SFT improves grounding accuracy in reasoning mode from 79.0\% to 81.4\%, while action-aware weighting further increases it to 83.4\%. Despite these improvements, \textbf{ASFT alone does not fully eliminate the gap between reasoning and no-reasoning modes}. For instance, ASFT-3B in reasoning mode still underperforms its no-reasoning counterpart by 4.8 points, and ASFT-7B exhibits a similar 3.4-point gap, suggesting that residual interference between reasoning and grounding remains.

This gap is largely eliminated after our RL training. GUI-Libra models achieve comparable, and even superior, grounding accuracy in reasoning mode despite producing longer responses compared to the no-reasoning mode. In particular, GUI-Libra-7B attains higher grounding accuracy in reasoning mode than in no-reasoning mode (89.3\% vs.\ 88.5\%), and GUI-Libra-3B achieves similar accuracy (83.4\% vs.\ 83.2\%) while generating substantially more tokens (176 vs.\ 124 and 206 vs.\ 59, respectively). Notably, our RL stage does not use direct grounding supervision, unlike prior work~\citep{luo2025guir1,lu2025uir1}; instead, it leverages step-wise data derived from high-level and multi-step tasks. These results demonstrate that \textbf{RL further reshapes the policy to better align reasoning with grounding, fully mitigating grounding degradation under long CoT outputs and complementing the benefits of action-aware SFT.}

\begin{figure}[t]
\centering
% ================= Left: image =================
\begin{minipage}[]{0.45\linewidth}
    \centering
    \includegraphics[width=0.95\linewidth]{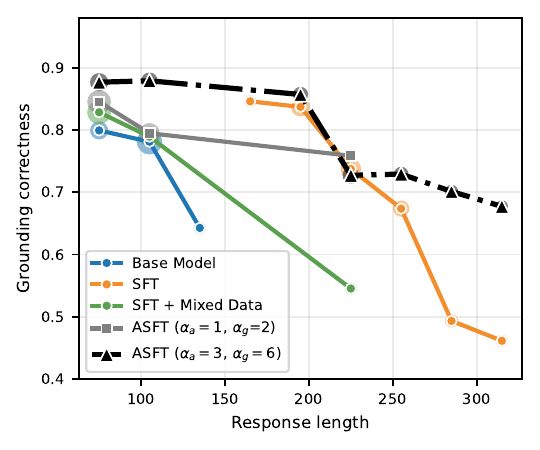}
    \vspace{-0pt}
    \caption{Grounding accuracy under different response lengths.
    Action-aware SFT strategies, including mixing direct-action data and weighted objectives, help preserve grounding accuracy under long CoT outputs.}
    \label{fig:grounding_length_ASFT}
\end{minipage}
\hfill
% ================= Right: table =================
\begin{minipage}[]{0.53\linewidth}
    \centering
    \vspace{-5pt}
    \resizebox{\linewidth}{!}{
    \begin{tabular}{l l c c}
        \toprule
        \textbf{Model Name} &
\makecell{\textbf{Inference}\\\textbf{Mode}} &
\makecell{\textbf{Average}\\\textbf{Tokens}} &
\makecell{\textbf{Grounding}\\\textbf{Accuracy (\%)}} \\
\midrule
SFT-3B & Reason & 223.4 & 73.4 \\
SFT 3B + Mixed Data & No-Reason & 77.3 & 79.4 \\
SFT 3B + Mixed Data & Reason & 200.6 & 73.8 \\
ASFT 3B & No-Reason & 67.7 & 81.0 \\
ASFT 3B  & Reason & 200.2 & 76.2 \\
\midrule
GUI-Libra-3B (ASFT+RL) & No-Reason & 59.0 & 83.2 \\
GUI-Libra-3B (ASFT+RL) & Reason & 206.5 & 83.4 \\
\midrule
SFT-7B & Reason & 218.2 & 79.0 \\
SFT 7B + Mixed Data & No-Reason & 69.4 & 85.6 \\
SFT 7B + Mixed Data & Reason & 168.1 & 81.4 \\
ASFT 7B   & No-Reason & 76.3 & 86.8 \\
ASFT 7B   & Reason & 169.6 & 83.4 \\
\midrule
GUI-Libra-7B (ASFT+RL) & No-Reason & 124.4 & 88.5 \\
GUI-Libra-7B (ASFT+RL) & Reason & 176.1 & 89.3 \\
        \bottomrule
    \end{tabular}
    }
    \captionof{table}{Grounding accuracy and average response tokens across different models and inference modes.
    ``Reason'' and ``No-Reason'' indicate whether explicit reasoning mode is encouraged through prompting.}
    \label{fig:grounding_asft_rl_comparison}
\end{minipage}

\end{figure}
\begin{tabular}{l l l c c}
\toprule

\bottomrule
\end{tabular}

\begin{figure*}[t]
    \centering
    \includegraphics[width=1\linewidth]{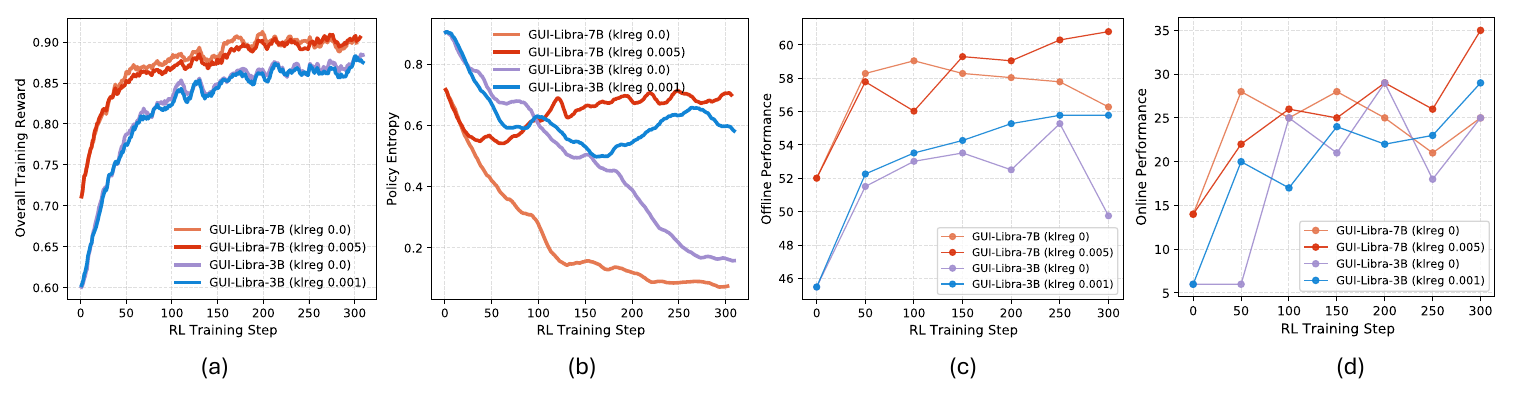}
    \vspace{-15pt}
   \caption{Comparison of training and evaluation metrics with and without KL regularization: (a) training reward, (b) policy entropy during training, (c) offline evaluation performance on AndroidControl-High, and (d) online evaluation performance on AndroidWorld.}
    \label{fig:train_evaluation_offline_online_metrics}
\end{figure*}

\subsection{On the Effectiveness of KL Regularization for RL}
As analyzed in Section~\ref{sec:kl_predictability_main}, KL regularization theoretically controls both distribution shift and reward ambiguity, making offline step-wise metrics more reliable predictors of online task completion. To validate this empirically, we visualize training and evaluation metrics on AndroidControl-High and AndroidWorld in Figure~\ref{fig:train_evaluation_offline_online_metrics}, which represent typical offline and online settings that share the same action space. We compare RL runs with and without KL regularization, using KL coefficients of 0.0 vs.\ 0.005 for 7B models and 0.0 vs.\ 0.001 for 3B models, with identical initialization. As shown in Figure~\ref{fig:train_evaluation_offline_online_metrics}(a), training reward curves largely overlap across different KL settings, indicating similar reward optimization behavior. However, evaluation behavior differs markedly: in Figure~\ref{fig:train_evaluation_offline_online_metrics}(c) and (d), models trained without KL regularization exhibit noticeable performance degradation despite increasing training rewards, reflecting a form of reward hacking commonly observed in RLHF. Moreover, Figure~\ref{fig:train_evaluation_offline_online_metrics}(b) shows that removing KL regularization leads to a pronounced decrease in policy entropy, indicating premature policy collapse and overfitting. In contrast, a small KL penalty stabilizes policy entropy and yields more consistent offline and online performance.

To quantitatively examine how well offline metrics predict online performance, we plot the offline and online scores of all intermediate checkpoints from both 3B and 7B models in Figure~\ref{fig:offline_online_correlation}. The results reveal an approximately linear relationship, supporting our theoretical analysis in Theorem~\ref{thm:off2on_main}. In Figure~\ref{fig:offline_online_correlation}(b), we report both Pearson correlation, which measures linear dependence, and Spearman correlation, which captures rank consistency between offline and online performance. Although the overall Pearson correlation is moderate ($r=0.76$), further analysis shows substantial differences across different KL regularization settings. When separating checkpoints by KL strength, models trained with KL regularization ($\mathrm{KL}>0$) exhibit significantly stronger alignment between offline and online performance. As shown in Figure~\ref{fig:offline_online_correlation}(b), KL-regularized models achieve Pearson and Spearman correlations of $0.89$ ($p<10^{-4}$) and $0.83$ ($p=2\times10^{-4}$), respectively, indicating strong and statistically robust dependence. In contrast, models trained without KL regularization show weaker correlations (Pearson $r=0.63$, $p=0.015$; Spearman $r=0.53$, $p=0.053$), with the rank correlation failing to reach conventional significance levels. Overall, these results provide strong empirical evidence that \textbf{KL regularization improves training stability and enhances the predictability of online task success from offline evaluations}, complementing our theoretical analysis in Section~\ref{sec:kl_predictability_main}.

\begin{figure}[t]
    \centering
    \includegraphics[width=0.65\linewidth]{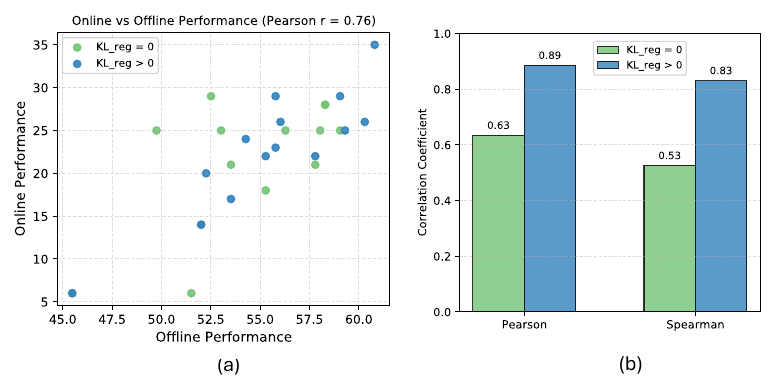}
    \vspace{-5pt}
    \caption{(a) Correlation between offline and online performance. (b) Comparison of Pearson and Spearman correlations with and without KL regularization.}
    \label{fig:offline_online_correlation}
\end{figure}

\subsection{Ablations}
In this section, we conduct a series of ablation studies to better understand the contribution of individual designs and parameters in our pipeline on GUI navigation tasks. Specifically, we examine the impact of our data filtering strategies, analyze the role of each component in ASFT, and study the impact of different KL coefficient during RL training.

\begin{figure}[t]
    \centering
    \includegraphics[width=1\linewidth]{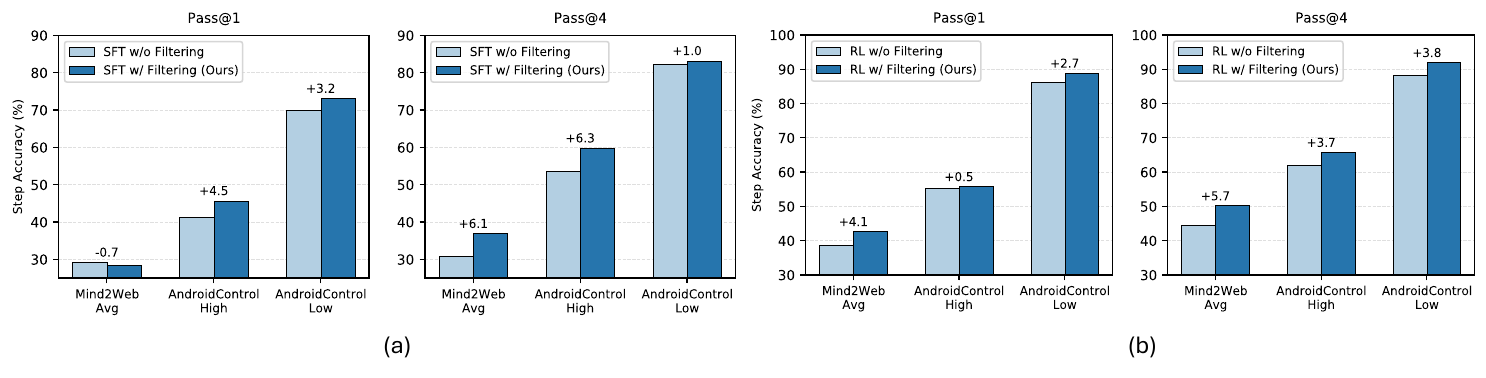}
    \vspace{-20pt}
    \caption{Ablation study of data filtering at the (a) SFT and (b) RL stages. Data filtering consistently improves both Pass@1 and Pass@4 performance across three benchmarks.}
    \label{fig:comparison_of_data_filtering}
\end{figure}

\paragraph{\textbf{Ablation of Data Filtering for SFT and RL.}}
In Sections~\ref{sec:sft_data_filter} and~\ref{sec:rl_data_filter}, we introduce data filtering pipelines for both the SFT and RL stages, where approximately half of the original dataset is retained. Specifically, for SFT, we remove low-quality and ambiguous samples, while for RL, we reduce domain imbalance and early-step bias. Figure~\ref{fig:comparison_of_data_filtering} reports results using a 3B base model. Across most settings, filtering consistently improves performance, with especially large gains in Pass@4. For example, in SFT, filtering improves AndroidControl-High by $+4.5$ Pass@1 and $+6.3$ Pass@4, while in RL it yields an additional $+0.5$ Pass@1 and $+3.7$ Pass@4. These results highlight \textbf{the importance of data quality in both SFT and RL: focusing on a smaller but cleaner and less biased dataset can generalize better than using a larger, noisier corpus.}

\begin{table}[h]
\centering
\caption{Ablations study on ASFT, KL regularization, and reasoning across benchmarks with Qwen2.5-VL-3B as the base model.}
\label{tab:ablation_navigation_3b}
\resizebox{1\linewidth}{!}{
\begin{tabular}{lccccccc}
\toprule
\textbf{Model} 
& \multicolumn{2}{c}{\textbf{MM-Mind2Web-v2}} 
& \multicolumn{2}{c}{\textbf{AC-v2 (High)}} 
& \multicolumn{2}{c}{\textbf{AC-v2 (Low)}}  & \textbf{AndroidWorld}\\
\cmidrule(lr){2-3} \cmidrule(lr){4-5} \cmidrule(lr){6-7}
& \textbf{Pass@1} & \textbf{Pass@4} 
& \textbf{Pass@1} & \textbf{Pass@4} 
& \textbf{Pass@1} & \textbf{Pass@4} \\
\midrule
Base Model & 23.4 & 28.3 & 36.4 & 50.8 & 71.1 & 79.2 & 3.5  \\

SFT
& 28.5 & 36.9
& 45.7 & 59.8
& 73.1 & 83.2 & 5.2 \\

SFT+ Mixed Data
& 30.2 & 42.0
& 45.5 & 64.8
& 72.6 & 85.4 & 11.3 \\

ASFT 
& 32.0 & 41.3
& 44.5 & 64.6
& 75.4 & 86.9 & 13.0 \\

\midrule
GUI-Libra w/o ASFT (KL\_reg=0.0)  &  40.9   & 45.1 &  50.5 & 54.8 & 78.6 & 81.9 &     17.4     \\
GUI-Libra w/o ASFT (KL\_reg=0.001)  & 41.9 & 48.0 & 57.0 & 63.6 &  86.2 & 90.0 & 20.9   \\

\midrule

GUI-Libra (KL\_reg=0.0)
& \textbf{43.8}  & 49.2 & 49.8 & 58.0 &  87.2 & 91.0 & 21.7 \\

GUI-Libra (KL\_reg=0.001)
& 42.7 & 50.2
& \textbf{55.8} & 65.8
& \textbf{89.5} & 91.5 & \textbf{25.2} \\

% GUI-Libra (KL\_reg=0.001) & 42.7 & 50.3 & \textbf{57.3} &\textbf{67.1} & 85.9 & 90.5 & \textbf{25.2} \\

GUI-Libra (KL\_reg=0.01) & 43.4 & 50.6  &  51.5& 64.8 & 85.9& 92.0 & 21.7 \\
GUI-Libra (KL\_reg=0.05) & 41.4 & \textbf{51.1} & 49.8& \textbf{66.1}&  87.9 & \textbf{92.2} & 20.0 \\

\midrule

ASFT w/o CoT & 35.2 & 41.4 & 40.2 & 56.0 &71.4 & 87.7 & 5.2 \\
GUI-Libra (KL\_reg=0.001)  w/o CoT & 43.4 & 47.1 & 48.7 & 55.8 &  85.2 &  87.7  & 12.2 \\
ASFT infer w/o CoT & 37.1 & 43.8  &  42.7 & 58.3 & 75.6& 89.2 & 8.7 \\
GUI-Libra (KL\_reg=0.001) infer w/o CoT & 42.5 & 48.7 & 52.0 & 59.3 & 88.2 & \textbf{92.2}  & 18.3\\

\bottomrule
\end{tabular}
}
\end{table}

\paragraph{\textbf{Ablations on ASFT and RL for Navigation Tasks.}}
In Section~\ref{sec:asft_rl_grounding}, we analyzed the effects of action-aware SFT and RL on mitigating grounding degradation under long CoT outputs. Here, we further examine their impact on navigation performance.
Table~\ref{tab:ablation_navigation_3b} reports results on both offline and online benchmarks using the Qwen2.5-VL-3B backbone.
We observe an improving trend on most benchmarks and metrics as we progress from the base model to SFT, mixed-data SFT, and ASFT. In particular, incorporating mixed supervision and action-aware weighting improves Pass@1 on MM-Mind2Web-v2 from 23.4 to 32.0 and on AndroidControl-v2 (High) from 36.4 to 44.5, while AndroidWorld sucess rate increases from 3.5 to 13.0. These results indicate that all components of ASFT benefit not only grounding accuracy, but also offline navigation and long-horizon online decision making. Beyond ASFT, RL brings further substantial gains. With RL training and moderate KL regularization (e.g., $\mathrm{KL}=0.001$), Pass@1 and Pass@4 on MM-Mind2Web-v2 improve by 10.7 and 8.9 points over ASFT, respectively, and AndroidWorld performance increases markedly from 13.0 to 25.2, highlighting the limitations of supervised fine-tuning alone and the importance of RL for generalization in dynamic environments.

\paragraph{\textbf{Ablations of KL Regularization Coefficient.}}
We further observe that the KL coefficient plays an important role in balancing Pass@1 and Pass@4 performance. As shown in Table~\ref{tab:ablation_navigation_3b}, increasing the KL coefficient generally improves Pass@4 performance, while Pass@1 may drop slightly. This trend is consistent with our observation that stronger KL regularization retains higher policy entropy. 
Within this trade-off, moderate regularization (e.g., $\mathrm{KL}=0.001$) yields strong and stable results across benchmarks, achieving the highest Pass@1 on AndroidControl-v2 and competitive Pass@1 and Pass@4 on MM-Mind2Web-v2, while substantially improving online performance on AndroidWorld. In contrast, overly large penalties (e.g., $\mathrm{KL}=0.05$) or removing KL regularization tend to degrade overall performance, reducing the AndroidWorld success rate to near 20.0. These results further validate that \textbf{moderate KL regularization effectively balances distribution shift and reward ambiguity, whereas excessively large penalties lead to overly conservative policies}.

\paragraph{\textbf{Ablations of Reasoning in Model Training and Inference.}} We analyze the role of reasoning in both training and inference by systematically ablating CoT usage in ASFT and GUI-Libra, as summarized in Table~\ref{tab:ablation_navigation_3b}. We first consider models \textbf{trained without CoT and evaluated without CoT (denoted as \emph{ASFT w/o CoT} and \emph{GUI-Libra w/o CoT})}. Compared to models trained and evaluated with CoT, performance degrades across most benchmarks, with the most pronounced drops on the online benchmark AndroidWorld: GUI-Libra’s success rate decreases from 25.2 to 12.2, and ASFT drops from 13.0 to 5.2. The declines are substantially larger than those observed on offline benchmarks, highlighting the importance of CoT for generalization in dynamic online environments.

Next, we remove CoT \textbf{only at inference time (denoted as \emph{ASFT infer w/o CoT} and \emph{GUI-Libra infer w/o CoT})}, while using checkpoints trained with CoT. Since ASFT incorporates direct-action supervision, these models can be prompted to produce direct actions at inference. Under this setting, ASFT shows improved performance on several offline benchmarks, such as MM-Mind2Web-v2 and AndroidControl-v2 (Low), compared to inference with CoT, but still exhibits a notable drop on AndroidWorld (13.0 $\rightarrow$ 8.7). In contrast, GUI-Libra consistently degrades on both offline and online benchmarks (except AndroidControl-v2 (Low)) when CoT is removed at inference. This suggests that ASFT can benefit from direct-action inference on in-distribution tasks, whereas RL relies more strongly on the joint presence of reasoning traces and actions to leverage their coupling for decision making. Importantly, across all ablations, training with CoT consistently yields better performance than removing CoT during training, even when inference is ultimately performed without CoT. Overall, these results highlight that explicit reasoning during both training and inference is important for effective GUI agent, especially for strong online generalization.

\begin{table}[t]
\centering
\caption{Comparison between GUI-Libra w/ SNGS and w/o SNGS across benchmarks.}
\label{tab:sngs_ablation}
\resizebox{\linewidth}{!}{
\begin{tabular}{l  cc  cc  cc  c  c}
\toprule
\textbf{Model}
& \multicolumn{2}{c}{\textbf{MM-Mind2Web-v2}}
& \multicolumn{2}{c}{\textbf{AC-v2 (High)}}
& \multicolumn{2}{c}{\textbf{AC-v2 (Low)}}
& \textbf{AndroidWorld}
& \textbf{WebArena-Lite-v2} \\
\cmidrule(lr){2-3} \cmidrule(lr){4-5} \cmidrule(lr){6-7}
& \textbf{Pass@1} & \textbf{Pass@4}
& \textbf{Pass@1} & \textbf{Pass@4}
& \textbf{Pass@1} & \textbf{Pass@4}
&
& \\
\midrule
GUI-Libra-4B (w/o SNGS) & 49.1 & 55.1 & 59.8 & 69.9 & 87.7 & 92.0 & 39.1 & 22.2 \\
GUI-Libra-4B (w/ SNGS)  & 50.0 & 55.6 & 62.3 & 68.6 & 86.4 & 93.0 & 42.6 & 24.4 \\
\bottomrule
\end{tabular}
}
\end{table}

\paragraph{\textbf{Ablation of SNGS.}}
Table~\ref{tab:sngs_ablation} examines the effect of SNGS on GUI-Libra-4B. Overall, enabling SNGS consistently improves online generalization, boosting performance on both AndroidWorld and WebArena-Lite-v2. For example, GUI-Libra-4B improves from 39.1$\rightarrow$42.6 (+3.5) on AndroidWorld and from 22.2$\rightarrow$24.4 (+2.2) on WebArena-Lite-v2, respectively. On offline benchmarks, SNGS yields smaller but generally positive gains on reasoning-demanding settings such as AndroidControl-v2 (High) and MM-Mind2Web-v2. These gains come with minor trade-offs on low-level metrics, i.e., AndroidControl-v2 (Low) Pass@1, suggesting that SNGS reduces overfitting to short-horizon action prediction and instead favors generalizable reasoning and more robust online behavior.

\begin{table}[h]
\centering
\caption{Effect of mixing grounding data into RL training on ScreenSpot-v2 (SS-v2), ScreenSpot-Pro (SS-Pro), MM-Mind2Web-v2, and AndroidControl-v2 (AC-v2). 
We report Pass@1 on all benchmarks. 
Green arrows indicate performance gains, and red arrows indicate degradations relative to the corresponding GUI-Libra models.}
\resizebox{1\linewidth}{!}{
\begin{tabular}{l|lllll}
\toprule
 & SS-v2 & SS-Pro & MM-Mind2Web-v2 & AC-v2 (high) & AC-v2 (low) \\
\midrule
Qwen3-VL-4B & 91.7 & 52.8 & 41.2 & 49.3 & 78.9 \\
GUI-Libra-4B & 92.3 & 54.3 & 49.1 & 59.8 & 87.7 \\
GUI-Libra-4B + Mix Grounding 20k 
& 94.6\,\updelta{+2.3}
& 61.4\,\updelta{+7.1}
& 43.9\,\downdelta{-5.2}
& 61.3\,\updelta{+1.5}
& 84.9\,\downdelta{-2.8} \\
\midrule
Qwen3-VL-8B & 92.1 & 52.7 & 43.8 & 54.8 & 77.6 \\
GUI-Libra-8B & 90.7 & 54.1 & 50.3 & 65.6 & 88.7 \\
GUI-Libra-8B + Mix Grounding 20k
& 94.8\,\updelta{+4.1}
& 59.9\,\updelta{+5.8}
& 49.5\,\downdelta{-0.8}
& 61.7\,\downdelta{-3.9}
& 86.4\,\downdelta{-2.3} \\
\bottomrule
\end{tabular}}
\label{tab:mixing_grounding_data_rl}
\end{table}

\subsection{RL with Mixed Navigation and Grounding Data}
Prior work~\citep{yang2025gta1,luo2025guir1,lu2025uir1} shows that adding direct grounding supervision (element descriptions paired with coordinates) during RL can substantially improve visual localization. To study how this supervision affects both grounding and reasoning, we take the grounding dataset from~\citep{yang2025gta1}, downsample 20K examples, and mix it with our 40K navigation-focused RL dataset for joint training. We convert grounding samples into our unified action format by keeping only click actions, and apply the same reward computation as in our RL pipeline. Table~\ref{tab:mixing_grounding_data_rl} reports results on grounding benchmarks (ScreenSpot-v2 and ScreenSpot-Pro) and navigation benchmarks (MM-Mind2Web-v2 and AndroidControl-v2). We observe a clear trade-off. Mixing grounding data consistently improves grounding accuracy on ScreenSpot-v2 and ScreenSpot-Pro by 2--7 points, indicating stronger visual localization. In contrast, performance on navigation benchmarks generally declines, suggesting weaker reasoning and decision making. Overall, these results reveal competing optimization pressures: \textbf{adding direct grounding supervision strengthens spatial alignment, but can reduce performance on reasoning-intensive navigation tasks}.

%% file: tables/offline_benchmarks.tex
\begin{table}[]
\centering
\caption{Step accuracy Performance on AndroidControl-v2.}
\label{tab:androidcontrol_v2_accuracy}
\resizebox{0.76\linewidth}{!}{
\begin{tabular}{lllll}
\toprule
 & \multicolumn{2}{c}{\textbf{High Level}} 
 & \multicolumn{2}{c}{\textbf{Low Level}} \\
\textbf{Model}
 & Pass@1 & Pass@4 
 & Pass@1 & Pass@4 \\
\midrule
\multicolumn{5}{c}{\textbf{Proprietary Models with SeeAct-V Framework}}\\
\midrule
GPT-4o + UGround-v1-7B & 57.0 & 66.3
& 78.4 & 85.4 \\

GPT-4.1 + UGround-v1-7B & 57.5 & 63.3
& 78.4 & 83.2 \\

GPT-5-mini + UGround-v1-7B  & 52.8 & 58.8
& 77.1 & 83.2 \\

GPT-5 + UGround-v1-7B & 61.3 & 69.4
& 86.2 & 90.0 \\

\midrule
\multicolumn{5}{c}{\textbf{Open-source Native Models}}\\
\midrule

 GUI-R1-3B   & 40.0 & 54.0
& 55.8 & 71.9 \\

 GUI-R1-7B & 39.7 & 56.3
& 62.3 & 72.6 \\

Aguvis-7B  & 37.7 & 43.7 
 & 48.0 & 48.7  \\

UI-TARS-1.5-7B
 & 45.2 & 63.1 
 & 48.5 & 70.9 \\

GLM-4.1V-9B-Thinking
 & 37.2 & 49.0 
 & 67.1 & 73.4  \\

Qwen2.5-VL-32B
 & 49.0 & 66.6 
 & 78.4 & 85.4  \\

Qwen2.5-VL-72B 
 & 56.5 & \textbf{72.9} 
 & 82.9 & 90.2 \\

Qwen3-VL-32B & 58.8 & 69.6
& 83.9 & 86.9 \\

\midrule

Qwen2.5-VL-3B (Baseline)
 & 36.4 & 50.8 
 & 71.1 & 79.2  \\

% \rowcolor{beijingblue}
% \textbf{GUI-Libra-3B old (Ours)}
%  & 55.8 {\scriptsize (+19.4)}
%  & 65.8 {\scriptsize (+15.0)}
%  & \textbf{88.9} {\scriptsize (+17.8)}
%  & \textbf{92.0} {\scriptsize (+12.8)} \\

\rowcolor{beijingblue}
\textbf{GUI-Libra-3B (Ours)}
 & 57.3 {\scriptsize (+20.9)}
 & 67.1 {\scriptsize (+16.3)}
 & 85.9 {\scriptsize (+14.8)}
 & 90.5 {\scriptsize (+11.3)} \\
 
Qwen2.5-VL-7B (Baseline)
 & 46.5 & 58.5 
 & 67.8 & 81.7 \\

% \rowcolor{beijingblue}
% \textbf{GUI-Libra-7B old (Ours)}
%  & \underline{60.8} {\scriptsize (+14.3)}
%  & 67.8 {\scriptsize (+9.3)}
%  & 85.9 {\scriptsize (+18.1)}
%  & 90.5 {\scriptsize (+8.8)} \\

\rowcolor{beijingblue}
\textbf{GUI-Libra-7B  (Ours)}
 & 59.3 {\scriptsize (+12.8)}
 & 67.3 {\scriptsize (+8.8)}
 & 85.2 {\scriptsize (+17.4)}
 & 90.7 {\scriptsize (+9.0)} \\

Qwen3-VL-4B (Baseline)  & 49.3 & 63.3
 & 78.9 & 82.4 \\

% \rowcolor{beijingblue}
% \textbf{GUI-Libra-4B old (Ours)}  & 59.8 {\scriptsize (+10.6)}
%  & 69.9 {\scriptsize (+6.5)}
%  & 87.7 {\scriptsize (+8.8)}
%  & \textbf{92.0} {\scriptsize (+9.6)} \\

\rowcolor{beijingblue}
\textbf{GUI-Libra-4B (Ours)} 
 & \underline{62.3} {\scriptsize (+13.0)}
 & 68.6 {\scriptsize (+5.3)}
 & \underline{86.4} {\scriptsize (+7.5)}
 & \textbf{93.0} {\scriptsize (+10.6)} \\
 
Qwen3-VL-8B (Baseline)  & 54.8 & 66.1
 & 77.6 & 83.2 \\

% \rowcolor{beijingblue}
% \textbf{GUI-Libra-8B old (Ours)} 
% & \textbf{65.6} {\scriptsize (+10.8)}
%  & \underline{71.9} {\scriptsize (+5.8)}
%  & \underline{88.7} {\scriptsize (+11.0)}
%  & \underline{91.7} {\scriptsize (+8.5)} \\

 \rowcolor{beijingblue}
\textbf{GUI-Libra-8B (Ours)} 
 & \textbf{64.3} {\scriptsize (+9.5)}
 & \underline{70.6} {\scriptsize (+4.5)}
 & \textbf{88.9} {\scriptsize (+11.3)}
 & \underline{91.7} {\scriptsize (+8.5)} \\

\bottomrule
\end{tabular}}
\end{table}

\begin{table}[]
\centering
\caption{Step accuracy Performance on Multimodal-Mind2Web-v2.}
\label{tab:mind2web_v2_accuracy}
\resizebox{0.95\linewidth}{!}{
\begin{tabular}{lllllllll}
\toprule
& \multicolumn{2}{c}{\textbf{Cross-Task} }
& \multicolumn{2}{c}{\textbf{Cross-Website} }
& \multicolumn{2}{c}{\textbf{Cross-Domain}} & \multicolumn{2}{c}{\textbf{Average}} \\
\textbf{Model}
& Pass@1 & Pass@4  & Pass@1 & Pass@4   & Pass@1 & Pass@4 & Pass@1 & Pass@4   \\
\midrule
\multicolumn{9}{c}{\textbf{Proprietary Models with SeeAct-V Framework}}\\
\midrule
GPT-4o + UGround-v1-7B  & 35.7 & 38.9 & 33.9 & 37.6 & 39.1 & 42.2 & 36.2 & 39.6 \\

GPT-4.1 + UGround-v1-7B  & 41.1 & 44.8
 & 36.2 & 39.7
 & 43.0 & 46.4
 & 40.1 & 43.6 \\

GPT-5-mini + UGround-v1-7B  & 44.2 & 48.0 & 40.4 & 44.2 & 45.8 & 48.1 & 43.5 & 46.7 \\

GPT-5 + UGround-v1-7B  & 47.7 & 51.7 & 45.0 & 47.6 & 48.2 & 51.6 & 47.0 & 50.3 \\

\midrule
\multicolumn{9}{c}{\textbf{Open-source Native Models}}\\
\midrule

GUI-R1-3B 
& 24.0 & 37.7 & 22.3 & 37.7 & 24.6 & 39.1 & 23.6 & 38.2 \\

GUI-R1-7B & 37.0 & 50.1 & 34.1 & 46.3 & 39.6 & 50.5 & 36.9 & 49.0 \\

Aguvis-7B 
& 37.7 & 48.0
& 31.7 & 41.5
& 36.9 & 45.1
& 35.4 & 44.9 \\

UI-TARS-1.5-7B
& 37.2 & 48.0
& 31.3 & 42.9
& 35.6 & 47.2
& 34.7 & 46.0 \\

GLM-4.1V-9B-Thinking
& 26.9 & 32.9
& 23.0 & 29.3
& 28.7 & 35.3
& 26.2 & 32.5 \\

Qwen2.5-VL-32B & 46.2 & 55.8 & 42.6 & \textbf{55.7} & 46.0 & 57.9 & 44.9 & 56.5 \\

Qwen2.5-VL-72B
& 49.1 & \textbf{60.2} & 45.1 & 54.0 & 49.8 & \underline{58.6} & 48.0 & \textbf{57.6} \\

Qwen3-VL-32B
& 48.8 & \underline{57.5}
& 44.3 & \underline{55.2}
& 49.6 & \textbf{58.9}
& 47.6 & \underline{57.2} \\

\midrule

Qwen2.5-VL-3B (Baseline) & 24.4 & 29.0 & 18.6 & 24.6 & 27.1 & 31.2 & 23.4 & 28.3 \\

\rowcolor{beijingblue}
\textbf{GUI-Libra-3B  (Ours)} 
 & 42.7 & 50.8
 & 40.6 & 48.4
 & 44.8 & 51.9
 & 42.7 {\scriptsize (+19.3)} & 50.3 {\scriptsize (+22.0}\\

% \rowcolor{beijingblue}
% \textbf{GUI-Libra-3B old (Ours)}
%  & 43.8 
%  & 50.8 
%  & 39.4
%  & 47.2 
%  & 45.0
%  & 52.8
%  & 42.7 {\scriptsize (+19.3)}
%  & 50.2 {\scriptsize (+21.9)} \\

Qwen2.5-VL-7B (Baseline)
& 31.6 & 45.6 & 30.4 & 42.1 & 35.6 & 48.0 & 32.5 & 45.2 \\

\rowcolor{beijingblue}
\textbf{GUI-Libra-7B (Ours)} 
 & 46.3 & 52.8
 & 45.5 & 52.2
 & 47.6 & 55.7
 & 46.5 {\scriptsize (+14.0)} & 53.6 {\scriptsize (+8.4)}\\

% \rowcolor{beijingblue}
% \textbf{GUI-Libra-7B old (Ours)} 
% & 47.5 
%  & 54.4 
%  & 45.3 
%  & 52.1
%  & 49.1 
%  & 56.7
%  & 47.3 {\scriptsize (+14.8)}
%  & 54.4 {\scriptsize (+9.2)} \\

Qwen3-VL-4B (Baseline)  & 42.9 & 52.2 & 38.8 & 50.0 & 42.0 & 51.7 & 41.2 & 51.3 \\

% \rowcolor{beijingblue}
% \textbf{GUI-Libra-4B old (Ours)} 
%  & 48.4 & 55.4
%  & \textbf{48.0} & 53.4
%  & \underline{50.9} & 56.5
%  & \underline{49.1} {\scriptsize (+7.9)} & 55.1 {\scriptsize (+3.8)} \\

\rowcolor{beijingblue}
\textbf{GUI-Libra-4B (Ours)} 
 & \underline{50.8} & 56.3
 & \textbf{48.4} & 53.2
 & \underline{50.8} & 57.5
 & \underline{50.0} {\scriptsize (+8.8)} & 55.6 {\scriptsize (+4.3)} \\

Qwen3-VL-8B (Baseline) 
 & 44.7 & 54.1 & 41.0 & 51.0 & 45.6 & 53.4 & 43.8 & 52.8 \\

% \rowcolor{beijingblue}
% \textbf{GUI-Libra-8B old (Ours)}
%  & \textbf{51.1} 
%  & 55.9 
%  & \underline{47.9} 
%  & 53.1
%  & \textbf{51.9} 
%  & 56.2 
%  &\textbf{50.3} {\scriptsize (+6.5)}
%  & 55.1 {\scriptsize (+2.3)} \\

 \rowcolor{beijingblue}
\textbf{GUI-Libra-8B (Ours)}  & \textbf{51.2} & 55.3
 & \underline{47.9} & 53.6
 & \textbf{52.4} & 56.7
 & \textbf{50.5} {\scriptsize (+6.7)} & 55.2 {\scriptsize (+2.4)} \\

\bottomrule
\end{tabular}}
\end{table}

%% file: sections/conclusion.tex
\section{Conclusion}
\label{sec:conclusion}

We introduce GUI-Libra, a unified framework for training reasoning-capable native GUI agents based on our curated GUI-Libra-81K dataset. The key takeaway is that competitive long-horizon navigation can be obtained by fully leveraging existing trajectory corpora with a carefully designed post-training pipeline: action-aware SFT preserve grounding performance under long reasoning traces and conservative RL that improves decision making from partially verifiable feedback by controlling policy drift. Across diverse mobile and web benchmarks, GUI-Libra achieves strong offline and online results with favorable data and parameter efficiency, without relying on expensive online interaction during training. Beyond benchmark gains, we analyze the effects of key design choices in ASFT and RL, and show that GUI-Libra makes offline evaluation more reliable and more predictive of online task success, an important property for real-world deployment. We hope our findings and released resources will encourage further work on data-efficient and reliable learning frameworks for interactive GUI agents in real-world settings.

%% file: sections/appendix.tex
\section{Additional Related Works}

\paragraph{Reinforcement Learning from Verifiable Rewards (RLVR)}
To optimize LLMs, learning a reward model as the training signal is a common practice in reinforcement learning from human feedback (RLHF) \citep{ouyang2022training,wang2023aligning,yang2024rewards}. However, reward models are known to suffer from reward hacking and misalignment issues \citep{gao2023scaling,yang2024regularizing}. To address these limitations, recent work has shifted toward \emph{verifiable rewards}, where supervision is derived from ground-truth verifiers, such as exact mathematical answer matching or code execution, rather than from a learned and potentially ambiguous reward model.
\citet{shao2024deepseekmath} introduce GRPO, which enables effective policy optimization from such verifiable signals and demonstrates the emergence of complex reasoning behaviors. Building on this framework, DAPO \citep{yu2025dapo} and Dr.~GRPO \citep{liu2025understanding} propose simple yet effective techniques, such as more aggressive clipping and dynamic sampling, to stabilize training and mitigate learning bias in GRPO. More recently, GSPO \citep{zheng2025group} leverages sequence-level importance sampling to further improve training stability, particularly for mixture-of-experts models. Despite their success, we find that directly applying RLVR recipes to step-wise agent training often leads to suboptimal policies, due to distribution shift and ambiguous intermediate rewards.

\paragraph{Post-training for VLM-based Agents.} VLMs have demonstrated strong capabilities in visual perception and multimodal reasoning\citep{openai2024gpt4ocard,bai2025qwen2,bai2025qwen3vltechnicalreport,zheng2024seeact}. However, deploying them as agents in visually grounded environments requires moving beyond static understanding toward robust, long-horizon decision-making. To bridge the gap, recent research adopts a two-stage SFT-then-RL paradigm\citep{chen2025eratransformingvlmsembodied,zhai2024finetuninglargevisionlanguagemodels,zhan2025visual}. In the first stage, SFT equips VLMs with essential agentic skills, including visual grounding, structured reasoning, and action prediction through curated dataset\citep{hong2024cogagentvisuallanguagemodel,cheng2024seeclickharnessingguigrounding,wu2024osatlasfoundationactionmodel,lin2024showuivisionlanguageactionmodelgui,qin2025uitarspioneeringautomatedgui,xu2025agenttrekagenttrajectorysynthesis}. Nevertheless, SFT is inherently limited by the coverage and diversity of demonstrations, therefore prone to compounding errors when encountering out-of-distribution states\citep{chen2025eratransformingvlmsembodied,liu2026rlbringvlageneralization,deng2025openvlthinkercomplexvisionlanguagereasoning}. RL complements SFT by enabling agents to interact directly with environments. Through exploration, agents can learn from both successes and failures, gradually developing capabilities such as error recovery, self-correction, and long-horizon planning \citep{bai2024digirltraininginthewilddevicecontrol,qi2025webrltrainingllmweb,putta2024agentqadvancedreasoning,feng2025groupingrouppolicyoptimizationllm,wang2025vagenreinforcingworldmodel}. Under this two-stage paradigm, SFT first establishes a stable foundation of core skills, after which RL enhances long-horizon decision-making via environment interaction and policy optimization.

\section{Implementation Details}\label{ap:implementation_details}

\paragraph{\textbf{Action Space.}}
We model GUI interaction with a unified action space where each step outputs a structured tuple
\texttt{(action\_type, action\_target, value, point\_2d)}.
Here, \texttt{action\_type} specifies the operation (e.g., \texttt{Click}, \texttt{Write}, \texttt{Scroll}),
\texttt{action\_target} describes the target UI element when applicable, and \texttt{value} provides
additional arguments such as input text, scroll/swipe direction, key name, waiting time, or app name.
For actions that require spatial grounding (e.g., \texttt{Click}, \texttt{LongPress}, and optionally \texttt{Swipe} or \texttt{Write}),
\texttt{point\_2d} records the screen coordinate \texttt{[x,y]}; otherwise it is set to \texttt{None}.
This unified schema supports both web and mobile environments, including device-level controls
(e.g., \texttt{NavigateBack}, \texttt{NavigateHome}, \texttt{OpenApp}) and a terminal action \texttt{Terminate}. See Table~\ref{tab:action_space_point2d} for the full specification.
Following the base model's coordinate system, we use absolute pixel coordinates for Qwen2.5-VL-based models,
and normalized coordinates in $[0,1000]$ for Qwen3-VL-based models.

\begin{table}[ht]
\centering
\caption{Unified action space. Each action is a tuple \texttt{(action\_type, action\_target, value, point\_2d)}.}
\label{tab:action_space_point2d}
\small
\setlength{\tabcolsep}{4pt}
\renewcommand{\arraystretch}{1.15}

\resizebox{\linewidth}{!}{%
\begin{tabular}{%
l
p{0.25\linewidth}
p{0.2\linewidth}
p{0.13\linewidth}
p{0.4\linewidth}} 
\toprule
\textbf{action\_type} & \textbf{action\_target} & \textbf{value} & \textbf{point\_2d} & \textbf{details} \\
\midrule
Answer & None & answer text & \texttt{[-100,-100]} & Return the final answer to the user's question. \\
\hline
Click & element description & None & \texttt{[x,y]} & Tap/click a specific UI element and provide its coordinates. \\

\hline
Select & element description & option value & \texttt{[-100,-100]} & Select an item in a list or dropdown menu. \\

\hline
LongPress & element description & None & \texttt{[x,y]} & Press-and-hold on a UI element (mobile only) and provide its coordinates. \\
\hline
Write & element description or None & input text & \texttt{[x,y]} or \texttt{[-100,-100]} &
Enter text into a specific input field; if \texttt{point\_2d} is \texttt{[-100,-100]}, type at the current focus. \\
\hline
KeyboardPress & None & key name (e.g., \texttt{enter}) & \texttt{[-100,-100]} & Press a specific key on the keyboard. \\
\hline
Scroll & None & direction (\texttt{up}/\texttt{down}/\texttt{left}/\texttt{right}) & \texttt{[-100,-100]} & Scroll a view/container in the specified direction. \\
\hline
Swipe & element description or None & direction (\texttt{up}/\texttt{down}/\texttt{left}/\texttt{right}) & \texttt{[x,y]} or \texttt{[-100,-100]} &
Perform a swipe gesture on a touchscreen in the given direction; provide coordinates if applicable. \\
\hline
Wait & None & seconds & \texttt{[-100,-100]} & Pause execution for a specified duration to allow UI updates. \\
\hline
NavigateHome & None & None & \texttt{[-100,-100]} & Navigate to the device's home screen. \\
\hline
NavigateBack & None & None & \texttt{[-100,-100]} & Press the system ``Back'' button. \\
\hline
OpenApp & None & app name & \texttt{[-100,-100]} & Launch an app by its name (mobile only). \\
\hline
Terminate & None & end-task message & \texttt{[-100,-100]} & Signal the end of the current task with a final message. \\
\bottomrule
\end{tabular}%
}
\end{table}

\paragraph{\textbf{SFT.}} We summarize our shared SFT and Action-aware SFT implementation parameters in Table \ref{tab:sft_config}. We apply full parameter tuning on Qwen2.5-VL and Qwen3-VL base models from 3B,4B, to 7B and 8B. We use a learning rate of  $1\times10^{-5}$  and an effective batch size of 256. We train SFT and ASFT models for either two epochs on GUI-Libra-81K or 1 epoch on mixing reasoning and direct-action data (it doubles data size). For action-aware SFT, we by default use $\alpha_a=2$ and $\alpha_g=4$ for ASFT, except for GUI-Libra-4B, where $\alpha_a=\alpha_g=1$. We use 8 B200 GPUs for approximately 4 hours for Qwen3-VL-4B and 5.5 hours for Qwen3-VL-8B.

\begin{table}[h]
\centering
\small
\setlength{\tabcolsep}{6pt}
\renewcommand{\arraystretch}{1.15}
\begin{tabular}{l c}
\toprule
\textbf{Category} & \textbf{Configuration} \\
\midrule
Backbone & Qwen2.5-VL / Qwen3-VL \\
Training Strategy & Full-parameter fine-tuning \\
Epochs & 1 \\
Learning Rate & $1\times10^{-5}$ \\
Scheduler & Cosine \\
Warmup Ratio & 0.01 \\
Weight Decay & 0 \\
Per-device Batch Size & 4 \\
Gradient Accumulation Steps & 8 \\
Effective Batch Size & 256 (8 GPUs) \\
Gradient Checkpointing & Enabled \\
\bottomrule
\end{tabular}
\caption{SFT configuration used in our experiments.}
\label{tab:sft_config}
\end{table}

\paragraph{\textbf{RL.}}
We adopt GRPO~\citep{shao2024deepseekmath} as our RL algorithm, implemented with the \texttt{verl} framework\footnote{\url{https://github.com/verl-project/verl}} and EasyR1\footnote{\url{https://github.com/hiyouga/EasyR1}}. GUI-Libra is initialized from the SFT/ASFT checkpoints and further optimized via online rollouts. At each iteration, the model samples trajectories, computes step-wise rewards, and updates the policy using the GRPO objective. We train for 300 RL iterations with a learning rate of $1\times10^{-6}$, global batch size 128, and rollout group size $n=8$. We set the KL regularization coefficient to 0.001 by default and increase it to 0.005 for GUI-Libra-7B to improve training stability. For SNGS, we use model-specific $(\lambda_0,\kappa)$: $(0.9,0.5)$ for 3B, $(1.4,-0.5)$ for 7B, $(0.5,1.5)$ for 4B, and $(0.5,2.0)$ for 8B. Training is conducted on 8 NVIDIA B200 GPUs and takes approximately 16 hours for Qwen3-VL-4B and 20 hours for Qwen3-VL-8B.

\begin{table}[h]
\centering
\small
\setlength{\tabcolsep}{6pt}
\renewcommand{\arraystretch}{1.15}
\begin{tabular}{l l}
\toprule
\textbf{Category} & \textbf{Configuration} \\
\midrule
RL Training Framework &  VERL  \\
Distributed Training Backend & FSDP \\
Inference Engine & vLLM \\
Backbone & Qwen2.5-VL, Qwen3-VL \\

\midrule

\multicolumn{2}{l}{\textit{Rollout}} \\
Rollout Samples per Prompt & 8 \\
Rollout Batch Size & 256 \\
Sampling Strategy & Top-$p$ 0.98 \\
Temperature & 1.0 \\
Max Prompt Length & 8092 \\
Max Response Length & 1500 \\
\midrule

\multicolumn{2}{l}{\textit{Optimization}} \\
Training Iterations & 300 \\
Learning Rate & $1\times 10^{-6}$ \\
Optimizer & AdamW (bf16) \\
Global Batch Size & 128 \\
Micro Batch (Update) & 4 \\
Micro Batch (Experience) & 8 \\
Clip Ratio ($\epsilon$) & 0.2 \\
KL Coefficient ($\beta$) & 0.001 by default, and 0.005 for GUI-Libra-7B  \\
\midrule

\multicolumn{2}{l}{\textit{Algorithm}} \\
Advantage Estimator & GRPO, GRPO w/ SNGS \\
Reward Function & $\tilde r(s,a)
\;=\;
w_{\text{fmt}}\, r_{\text{fmt}}
\;+\;
(1 - w_{\text{fmt}})\, r_{\text{acc}}, w_{\text{fmt}}=0.1$ \\
\bottomrule
\end{tabular}
\caption{RL configuration for our experiments.}
\label{tab:rl_config}
\end{table}

\paragraph{\textbf{Evaluation.}} For inference, we use vLLM as the serving backend. We set the temperature to 0.0 and top-$p$ to 1.0, and allow up to 1024 completion tokens by default. The system prompt follows Appendix~\ref{ap:prompt_template}. The available action list within the system prompt can be adjusted depending on the deployment environment. For example, if a mobile environment does not support the \texttt{OpenAPP} action, it can be removed from the action list; in this case, the model typically resorts to alternative strategies such as scrolling to access the app drawer.

For models trained with mixed direct-action data, we optionally use an explicit instruction prompt to elicit direct action generation without intermediate reasoning. Specifically, we can append the following format instruction after user instruction:

\begin{verbatim}
The response should be structured in the following format. 
Make sure the output between <answer> and </answer> is a valid JSON object. 
Regarding the key "point_2d", please provide the coordinates on the screen where the action 
is to be performed; if not applicable, use [-100, -100]:
<answer>
{ 
    "action\_description": "the description of the action to perform, summarized in one sentence",
    "action\_type": "the type of action to perform. Please follow the system prompt for 
                    available actions.",
    "value": "the input text or direction ('up', 'down', 'left', 'right') for the 'scroll' action 
            or the app name for the 'openapp' action; otherwise, use 'None'", 
    "point\_2d": [x, y] 
}
</answer> 
\end{verbatim}

\section{Benchmark Details}\label{ap:benchmark_details}
In our experiments, we adopt a diverse of benchmarks for evaluation, mainly focus on GUI navigation benchmarks that measures step-wise success or task-level completion. We also use grounding benchmarks to evaluate the correctness of grounding after long CoT generation. Details about these benchmarks are as follows.

\subsection{Grounding Benchmarks}
We adopt ScreenSpot-V2 \citep{cheng2024seeclick,wu2025osatlas} and ScreenSpot-Pro \citep{li2025screenspot} to evaluate grounding accuracy. Each task provides a short instruction specifying the target element or intent, together with a screenshot from a digital interface (mobile, desktop, or web). The ground-truth target is given as a bounding box, and we measure success by whether the model’s predicted click coordinate falls inside the box. ScreenSpot-V2 corrects labeling errors in the original ScreenSpot benchmark and contains 1,269 tasks, with most screenshots below 2560$\times$1440 resolution. In contrast, ScreenSpot-Pro includes 1,555 tasks and features substantially higher-resolution screenshots (up to 5120$\times$2880), resulting in denser visual content and more fine-grained targets. Overall, ScreenSpot-V2 reflects common UI settings, while ScreenSpot-Pro stresses precise grounding in high-resolution, information-rich interfaces.

\subsection{Offline GUI Navigation Benchmarks} \label{ap:offline_benchmark}
\paragraph{\textbf{Multimodal-Mind2Web-v2}} We build our benchmark on \textit{Multimodal-Mind2Web} (MM-Mind2Web) \citep{zheng2024seeact}, the multimodal extension of Mind2Web \citep{deng2023mindweb}, to evaluate offline web navigation on realistic user tasks. MM-Mind2Web aligns each step in a human demonstration with a webpage screenshot (and the corresponding HTML/DOM state), forming a golden multi-step trajectory conditioned on a high-level natural-language instruction \citep{deng2023mindweb,zheng2024seeact}. The test split spans 100+ websites; all webpages along the golden trajectories are cached to support fully offline evaluation, and tasks are crowdsourced to reflect real user intents. 
As shown in Figure~\ref{fig:limitations_of_MMM2Web_AC}, MM-Mind2Web represents action history as symbolic records that are neither natural-language descriptions nor aligned with real user interaction. To address this, we use Qwen3-VL-32B-Instruct to rewrite each action into a natural-language description and use these descriptions as the history context. The resulting dataset, \textit{Multimodal-Mind2Web-v2} (MM-Mind2Web-v2), contains three subsets, Cross-Task, Cross-Website, and Cross-Domain, with 1,328, 1,019, and 1,002 samples, respectively. 
We report \textbf{step success rate} as the primary metric, which requires both correct target grounding (\emph{element accuracy}) and correct operation execution. Operation correctness is measured by an exact-match F1 score (F1$=1$) over the serialized action string ``\texttt{ActionType Value}'', where \texttt{Value} can be the typed text for \texttt{Write} actions or the app/website identifier for \texttt{OpenApp} actions.

\paragraph{\textbf{AndroidControl-v2}} \textit{AndroidControl-v2} is based on AndroidControl \citep{li2024androiedcontrol}, an offline Android GUI navigation benchmark that pairs step-wise instructions with mobile screenshots and demonstrated actions. However, AndroidControl contains non-trivial annotation noise (about 20\% errors in action types and/or coordinates), as illustrated in Figure~\ref{fig:limitations_of_MMM2Web_AC}. To improve evaluation reliability, we use Qwen3-VL-32B-Instruct to filter mismatched samples by checking the consistency between each demonstrated action and its oracle low-level step instruction. We start from a sampled set of 500 examples from UGround \citep{gou2025uground} and obtain a cleaned subset of \textbf{398 examples} after filtering. We evaluate under both high-level and low-level instructions and report step accuracy, where a step is counted as successful only if the predicted action type, textual value (when applicable), and target coordinates are all correct. \textbf{Note that prior work may use inconsistent evaluation protocols for this benchmark.} For example, OS-Atlas \citep{wu2025osatlas} and GUI-R1 \citep{luo2025guir1} treat grounding as a distance-to-target threshold, which can be misleading in practice because UI elements vary greatly in size (so a fixed threshold is not comparable across screens). Instead, we follow UGround' strategy \citep{gou2025uground}: we use the accessibility tree to map the predicted coordinate to its nearest UI element, and then match that element against the ground-truth target.

\subsection{Online GUI Navigation Benchmarks}

\paragraph{\textbf{AndroidWorld}} We evaluate online mobile agent performance on \textit{AndroidWorld} \citep{rawles2025androidworld}, which runs interactive tasks in Android emulators and scores agents by whether they reach the correct final device states. AndroidWorld contains 116 tasks across 20 real-world apps and covers diverse multi-step workflows (e.g., search, form filling, and settings changes) under realistic UI dynamics. With the official Docker environment, we found that Task \#82 (\textsc{SimpleSmsReplyMostRecent}) cannot be initialized, so we report results on the remaining 115 tasks.
Evaluation uses the benchmark’s rule-based completion checker with a maximum horizon of 20 steps. To assess self-verification, we count a task as successful only when (i) the agent explicitly outputs a \texttt{Terminate} action and (ii) the environment state satisfies the completion rules. Our agent follows the See-Act-V framework following UGround \citep{gou2025uground}, but \textbf{we remove the step-wise reflection and summary module. This design choice isolates the \emph{native} capability of the underlying model, rather than relying on a hand-crafted control structure.} For completeness, we also report baseline results with the summary modules, and show that our native model (without such modules) can surpass agents that depend on these additional components, highlighting the potential of our recipe to reduce human-designed scaffolding.

\paragraph{\textbf{WebArena-Lite-v2}} We evaluate online web agent performance on \textit{WebArena-Lite-v2}~\citep{liu2025scalecua}, a locally deployed website environments upgraded from \textit{WebArena-Lite}~\citep{liu2025visualagentbench} consisting of 154 tasks. To setup the website environment, we follow the instruction by Maxime Gasse\footnote{\url{https://github.com/gasse/webarena-setup/tree/main/webarena}} for a more stable Map setup\footnote{We found the container for Map website is brittle when following the official WebArena-Lite-v2 setup: \url{https://github.com/OpenGVLab/ScaleCUA/tree/main/evaluation/WebArenaLiteV2}.}. Building upon the official WebArena-Lite-v2 implementation, we further improve the robustness of action parsing and execution by: (i) parsing the "response" action as "answer" rather than treating it as illegal; (ii) appending an automated "terminate" action after "action" and "response"; (iii) supporting multi-line answer strings (separated by \texttt{\textbackslash n}) instead of only the first line; (iv) clearing blank content before typing predicted messages. Moreover, we replace \texttt{gpt-4o-2024-11-20} with \texttt{gpt-5} for more accurate LLM-based fuzzy evaluation as we observed false positives when using GPT-4o. Given the high variance in results, we report the average across four runs for all experiments, with the temperature set to 0.0 and top\_p to 1.0.

\paragraph{\textbf{Online-Mind2Web}} We also include Online-Mind2Web benchmark~\citep{online-mind2web} spanning 136 live real-world websites and covering 300 web agent tasks. We set the maximum interaction steps to 30 for each task. We adopt the proposed WebJudge method backed by either o4-mini~\footnote{\url{https://developers.openai.com/api/docs/models/o4-mini}} or WebJudge-7B~\footnote{\url{https://huggingface.co/osunlp/WebJudge-7B}} models. The LLM-based judge first identifies key points of the task, then selects task-relevant key screenshots from each step.
Finally, the judge is provided with task description, agent textual actions, task completion key points and selected key screenshots to make a binary outcome judgment indicating whether all key points are satisfied.

\section{Additional Results}\label{app:additional_exp}
This section presents additional results to further clarify our approach. We include an auxiliary study on grounding as a single-step verifiable case in Appendix~\ref{ap:grounding_correlation} to contrast with multi-step navigation under partial verifiability, a controlled comparison of reasoning-augmentation models in Appendix~\ref{ap:reasoning_model_comparison}, as well as complementary offline metrics (grounding accuracy, action-type accuracy or operation F1) in Appendix~\ref{ap:addtional_metrics}.

\subsection{Grounding as a Single-step Verifiable Setting}
\label{ap:grounding_correlation}
Grounding provides a near-ideal \emph{single-step verifiable} setting: each example typically refers to a specific UI element, and we can directly verify correctness by checking whether the predicted coordinate falls inside the annotated bounding box. Since the agent produces only a single action, this setting matches the assumptions in Corollary~\ref{cor:rlvr_easy}. To study this regime, we train Qwen3-VL-4B and Qwen3-VL-8B with GRPO on a 40K downsampled grounding dataset from GTA1~\citep{yang2025gta1}. Results are shown in Figure~\ref{fig:grounding_correlation}.

Unlike navigation, grounding does not exhibit an offline--online evaluation gap; therefore, we assess \emph{predictability} using two grounding benchmarks instead: ScreenSpot-v2 and ScreenSpot-Pro. ScreenSpot-v2 is closer to our training distribution (similar image resolution), while ScreenSpot-Pro contains substantially higher-resolution screenshots, making it a useful test of distribution shift. Figures~\ref{fig:grounding_correlation}(a)--(b) show that performance on both benchmarks improves steadily and then plateaus after roughly 200 RL steps, without the significant reward-hacking-style drops observed in multi-step navigation. Moreover, the two benchmarks are strongly correlated: Figure~\ref{fig:grounding_correlation}(c) shows a tight relationship between ScreenSpot-v2 and ScreenSpot-Pro scores, and Figure~\ref{fig:grounding_correlation}(d) reports very high Pearson and Spearman correlations. Interestingly, removing KL regularization yields even higher correlation (Pearson $\approx 0.98$), consistent with our analysis that in single-step, verifiable tasks, RLVR-style training can remain stable and highly predictable even without KL regularization.

\begin{figure}
    \centering
    \includegraphics[width=1\linewidth,trim=0 40 0 40, clip ]{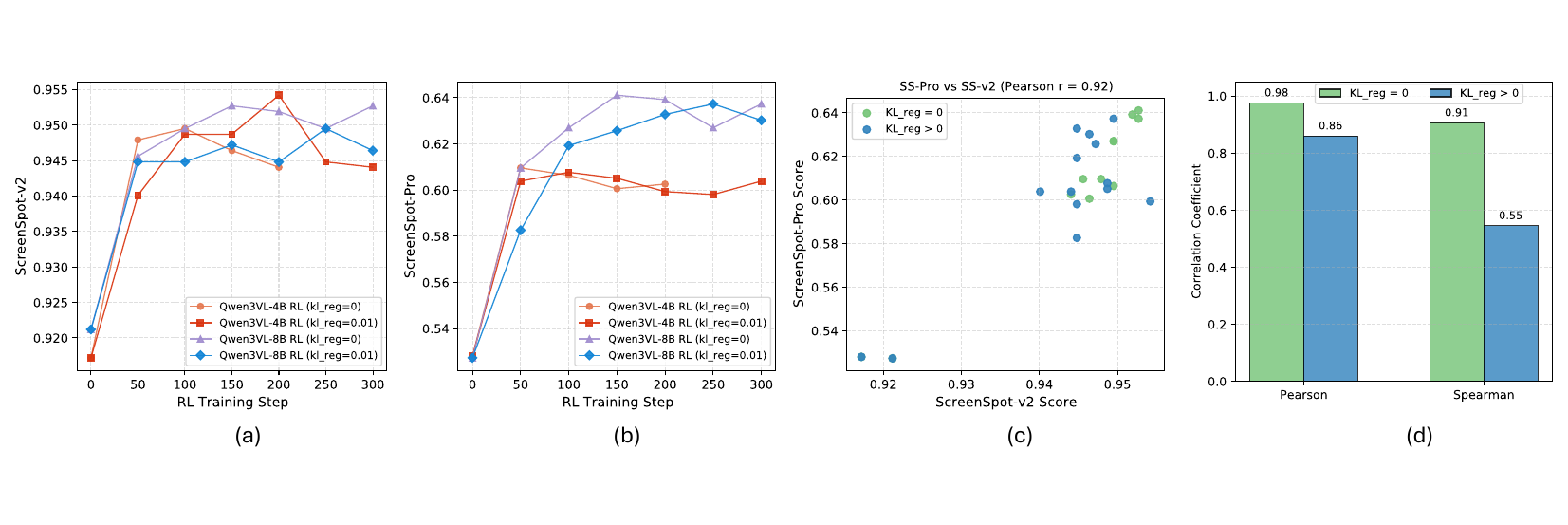}
    \caption{RL for grounding exhibits stable improvements and strong cross-benchmark predictability. (a) ScreenSpot-V2 and (b) ScreenSpot-Pro performance over RL training. (c) Correlation between the two benchmark scores across checkpoints. (d) Pearson and Spearman correlations with and without KL regularization.}
    \label{fig:grounding_correlation}
\end{figure}

\begin{figure}[ht]
    \centering
    \includegraphics[width=0.45\linewidth]{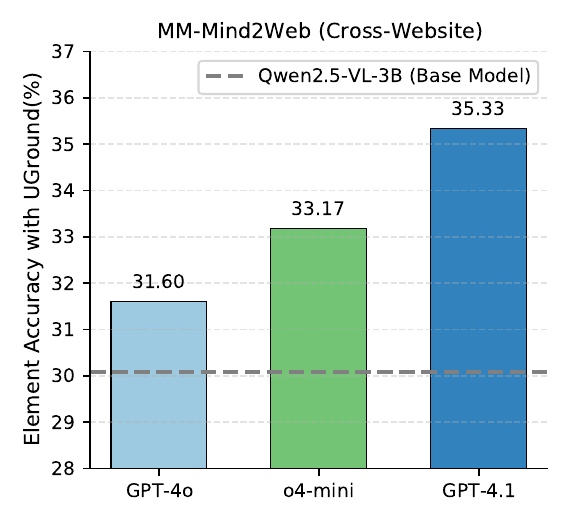}
   \caption{Performance comparison of different models for reasoning generation. All models are used to augment the same 30K web samples from AGUVIS and are fine-tuned on the same 3B base model. This controlled evaluation reveals substantial performance differences among generator models.}
    \label{fig:dataset_comparison}
\end{figure}

\subsection{Comparing Different Models for Reasoning Augmentation}\label{ap:reasoning_model_comparison}
In our reasoning augmentation pipeline, we use GPT-4.1 to generate reasoning traces. This choice is motivated by preliminary experiments showing that GPT-4.1 can be prompted to produce richer, more informative rationales for VLM training. With the same prompt, GPT-4o often generates shorter reasoning, while reasoning models such as o4-mini typically provide limited visible traces and hide their true reasoning traces, resulting in similarly short rationales. In contrast, GPT-4.1 more reliably produces detailed reasoning that provide more useful thought and better aligns with the actions.

We further quantify this effect with a controlled comparison. Using the same 30K web samples from AGUVIS, we generate reasoning traces with each model under an identical prompt and fine-tune the same Qwen2.5-VL-3B-Instruct base model on the resulting augmented data. To isolate the impact of reasoning quality, we extract action targets from the model outputs and use the same UGround-7B-v1 \citep{gou2025uground} model for coordinate prediction across all settings. Figure~\ref{fig:dataset_comparison} reports results on the original MM-Mind2Web Cross-Website subset: GPT-4.1 yields the best performance, outperforming GPT-4o by $+3.7$ and o4-mini by $+2.1$. These results indicate that the choice of reasoning generator is an important factor in effective reasoning augmentation.

\subsection{Comparing with Uniform Negative Gradient Scaling Strategy}
In our method, we use an adaptive negative gradient scaling method to enable adaptive scaling behaviors. Table~\ref{tab:ablation_uniform_scaling} compares our adaptive negative gradient scaling strategy with a uniform variant that applies a constant scaling factor $\lambda_g$ across all states. Uniform scaling yields weaker overall performance: both $\lambda_g=0.75$ and $\lambda_g=0.9$ perform worse on MM-Mind2Web-v2, AC-v2 (High/Low), and AndroidWorld.
These results suggest that a single global scaling factor cannot capture the heterogeneous difficulty and ambiguity across states and can lead to suboptimal optimization, whereas adaptive scaling provides the flexibility needed to stabilize training and improve decision making.

\begin{table}[h]
\centering
\caption{Ablations study of negative scaling strategy on 3B base model.}
\label{tab:ablation_uniform_scaling}
\resizebox{0.75\linewidth}{!}{
\begin{tabular}{lccccccc}
\toprule
\textbf{Model} 
& \multicolumn{2}{c}{\textbf{MM-Mind2Web-v2}} 
& \multicolumn{2}{c}{\textbf{AC-v2 (High)}} 
& \multicolumn{2}{c}{\textbf{AC-v2 (Low)}}  & \textbf{AndroidWorld} \\
\cmidrule(lr){2-3} \cmidrule(lr){4-5} \cmidrule(lr){6-7}
& \textbf{Pass@1} & \textbf{Pass@4} 
& \textbf{Pass@1} & \textbf{Pass@4} 
& \textbf{Pass@1} & \textbf{Pass@4} \\
\midrule

GUI-Libra ($\lambda_g=0.75$) & 42.2 & 48.5 & 54.5& 64.1& 86.7& 90.2 & 19.1\\
GUI-Libra ($\lambda_g=0.9$)
& \textbf{42.7}  & 48.1  & 52.3& 61.1 &  87.4 & 90.7 & 20.0\\
GUI-Libra 
& \textbf{42.7} & \textbf{50.2}
& \textbf{55.8} & \textbf{65.8}
& \textbf{89.5} & \textbf{91.5} & \textbf{25.2} \\

\bottomrule
\end{tabular}
}
\end{table}

\subsection{Additional Metrics on Offline Benchmarks}\label{ap:addtional_metrics}
In our main offline evaluations, we use \emph{step accuracy} as the primary metric. To better understand \emph{where} the gains come from, we also report decomposed metrics that separate grounding from action prediction. Specifically, we report (i) \textbf{grounding accuracy} (whether the predicted coordinate lies inside the target element) and \textbf{action-type accuracy} on AndroidControl-v2, and (ii) \textbf{grounding accuracy} and \textbf{operation F1} on MM-Mind2Web-v2. Operation F1 is computed following the protocol in Appendix~\ref{ap:offline_benchmark}.

\paragraph{AndroidControl-v2.}
Tables~\ref{tab:AC_grounding_accuracy} and~\ref{tab:ac_type_accuracy} show that GUI-Libra remains strong under these more fine-grained metrics. Notably, GUI-Libra-8B achieves the best Pass@1 grounding accuracy on both high-level and low-level tasks, reaching 76.3 and 95.5 Pass@1 on the high-level and low-level settings, respectively. It also outperforms GPT-5 + UGround-7B-v1 and larger open-weight models (e.g., 32B and 72B). The overall trend for action-type accuracy is similar, but the gaps to strong baselines are smaller, suggesting that our improvements are driven primarily by better grounding rather than action-type prediction. This is expected: action types are often inferred reliably from language and prior knowledge, whereas accurate grounding, especially for high-level, goal-directed tasks, is more challenging for GUI tasks.

\paragraph{MM-Mind2Web-v2.}
Tables~\ref{tab:mind2web_grounding} and~\ref{tab:mind2web_operation_f1} report grounding accuracy and operation F1, respectively. For grounding, GUI-Libra-8B achieves the highest Pass@1 across all three subsets, with an average Pass@1 of 57.8, surpassing strong baselines including Qwen3-VL-32B, Qwen2.5-VL-72B, and GPT-5 with UGround. For Pass@4, Qwen2.5-VL-72B is 1.4 points higher than GUI-Libra-8B, indicating that web-domain evaluation can still benefit from additional model capacity and/or web-specific training data. We attribute this gap mainly to domain imbalance in our training set: roughly 85\% of our SFT data comes from mobile, with only 15\% from web. We expect that scaling high-quality web data would further improve Pass@4, consistent with our strong Pass@1/Pass@4 results on AndroidControl-v2. For operation F1, we report \texttt{Best@N} (the maximum over $N\!\in\!\{1,4\}$ samples), since F1 is a continuous score in $[0,1]$. While larger models achieve strong \texttt{Best@4} results, GUI-Libra-7B attains the best \texttt{Best@1} performance (85.1), and GUI-Libra-3B remains competitive with 32B--72B baselines.

Overall, these additional metrics show that GUI-Libra improves not only step accuracy but also its underlying components, providing a more complete view of the sources of improvement.

\input{sections/additional_metrics}

\input{sections/appendix-proof}

\section{Prompt Templates}\label{ap:prompt_template}

\subsection{Prompt for Reasoning Augmentation}
We use the following prompts with GPT-4.1 to generate reasoning traces, which form our initial reasoning dataset. We draw inspiration from AGUVIS \citep{xu2025aguvis} and further refine and adapt the prompt through iterative trial-and-error during our data generation process. 

\begin{promptbox}{Prompt Template for Web Data}
\noindent Please generate detailed reasoning and explain your logic according to the UI screenshot, instruction, interaction history and the reference action.

\vspace{1em}
\noindent \textbf{\#\# Instruction: \{\}} \\
\textbf{\#\# Interaction History: \{\}} \\
\textbf{\#\# Reference Action Description for Current Step: \{\}} \\
\textbf{\#\# Reference Action Command for Current Step: \{\}}

\vspace{1em}
\noindent Note: If there exists a referenced element, it is highlighted by a small red hollow circle in the screenshot.

\vspace{1em}
\noindent \textbf{\#\# Guidelines}
\begin{itemize}
    \item Carefully observe and interpret the screenshot to extract any relevant information that can inform your reasoning.
    \item Reason step by step to accomplish the instruction, aligning your thought process with the overall goal.
    \item Use the reference action as a guide, but do not simply repeat it. Instead, explain why you would take that action (if it is reasonable) without referring to it.
    \item Do not rely on the current reference action or the highlighted hollow circle as justification—they are outcomes of hindsight.
    \item Use a first-person perspective to express your reasoning, as if you are the one navigating and making decisions.
    \item If the question is primarily in a language other than English, please reason in English but provide the final answer in the corresponding language.
    \item The reference action may involve both clicking and writing simultaneously. In such cases, describe the action as: first click the target element, then type the desired text. You should specify the action type as Write and indicate the target element where the text should be entered in the json.
\end{itemize}

\vspace{1em}
\noindent \textbf{\#\# Response Format:} \\
Return exactly one \texttt{<think></think> + <answer></answer>} pair for one step, with the following structured format:

\vspace{1em}
\noindent \texttt{<think>} \\
Explain your reasoning in detail, step by step. \\
$\bullet$ State what you observe and extract useful information from the screenshot. Keep important information in the thought because you may need to refer to them later. \\
$\bullet$ Reflect on the instruction and interaction history to understand the context and current status. \\
$\bullet$ Decide what action to take next and why, following the guidelines. Do not summarize—show your full chain of thought. \\
\texttt{</think>} \\
\vspace{1em}
\noindent \texttt{<answer>} \\
\{ \\
\hspace*{1em} "action\_description": "Describe the current step action with a clear instruction", \\
\hspace*{1em} "action\_type": "The type of action to take, e.g., Click, Select, Write, KeyBoardPress, Scroll, Terminate or Answer", \\
\hspace*{1em} "action\_target": "Provide a description of the element you want to operate on the screen. (If action\_type == Scroll and Answer, this field should be 'None'.) It should include the element's identity, type (button, input field, dropdown menu, tab, etc.), text on it (if have), and location hints if needed to disambiguate. If you want to write text into an input field, specify the input field's identity and necessary information. If you already clicked the input field in last step, you can leave this field as 'None'.", \\
\hspace*{1em} "value": "Specify the input or selected value. For Write, Select, and Answer actions, provide the exact text or option. For KeyboardPress, specify the key(s). For Scroll, indicate the direction as in the reference action command. For other action types, use 'None'." \\
\} \\
\texttt{</answer>}
\end{promptbox}

\begin{promptbox}{Prompt Template for Mobile Data}
\noindent Please generate detailed reasoning and explain your logic according to the UI screenshot, instruction, interaction history and the reference action.

\vspace{1em}
\noindent \textbf{\#\# Instruction:} \\
\{\}

\vspace{1em}
\noindent \textbf{\#\# Interaction History:} \\
\{\}

\vspace{1em}
\noindent \textbf{\#\# Reference Action Description for Current Step (labeled by human):} \{\} \\
\textbf{\#\# Reference Action Command for Current Step (labeled by human):} \{\} \\
Note: If there exists a referenced element, it is highlighted by a small red hollow circle in the screenshot.

\vspace{1em}
\noindent \textbf{\#\# Guidelines}
\begin{itemize}
    \item Carefully observe and interpret the screenshot to extract any relevant information that can inform your reasoning.
    \item Reason step by step to accomplish the instruction, aligning your thought process with the overall goal.
    \item Use the reference action as a guide, but do not simply repeat it. Instead, explain why you would take that action (if it is reasonable) without referring to it.
    \item Do not rely on the current reference action or the highlighted hollow circle as justification—they are outcomes of hindsight.
    \item Use a first-person perspective to express your reasoning, as if you are the one navigating and making decisions.
    \item If the question is primarily in a language other than English, please reason in English but provide the final answer in the corresponding language.
    \item The reference action may involve both clicking and writing simultaneously. In such cases, describe the action as: first click the target element, then type the desired text. You should specify the action type as Write and indicate the target element where the text should be entered in the json.
\end{itemize}

\vspace{1em}
\noindent \textbf{\#\# Response Format:} \\
Return exactly one \texttt{<think></think> + <answer></answer>} pair for one step, with the following structured format:

\vspace{0.5em}
\noindent \texttt{<think>} \\
Explain your reasoning in detail, step by step. \\
{\tiny\ding{110}} State what you observe and extract useful information from the screenshot. Keep important information in the thought because you may need to refer to them later. \\
{\tiny\ding{110}} Reflect on the instruction and interaction history to understand the context and current status. \\
{\tiny\ding{110}} Decide what action to take next and why, following the guidelines. \\
Do not summarize—show your full chain of thought. \\
\texttt{</think>}\\
\vspace{0.5em}
\noindent \texttt{<answer>} \\
\{ \\
\indent "action\_description": "Describe the current step action with a clear instruction", \\
\indent "action\_type": "The type of action to take, e.g., Click, Write, LongPress, Scroll, Swipe, NavigateHome, Wait, NavigateBack, Terminate, OpenAPP, or Answer. If the referece action is terminate and you need to answer the question from the user, you should specify the action type as Answer and provide the answer in the value field.", \\
\indent "action\_target": "Provide a description of the element you want to operate on the screen. (If action\_type == Scroll, Terminate, or Answer, this field should be 'None'). It should include the element's identity, type (button, input field, dropdown menu, tab, etc.), text on it (if have), and location hints if needed to disambiguate. If you want to write text into an input field, specify the input field's identity and necessary information. If you already clicked the input field in last step, you can leave this field as 'None'.", \\
\indent "value": "Specify the input or selected value. For Write, Select, and Answer actions, provide the exact text or option. For Scroll and Swipe, indicate the direction. For Wait, specify the wait time in seconds. For OpenAPP, specify the app name. For other action types, use 'None'. " \\
\} \\
\texttt{</answer>}
\end{promptbox}

\subsection{SFT Data Example}

\begin{datanotebox}{SFT Data Example}
\textbf{Image:} \texttt{20162/screenshot\_1.png}

\begin{center}
\includegraphics[width=0.3\linewidth]{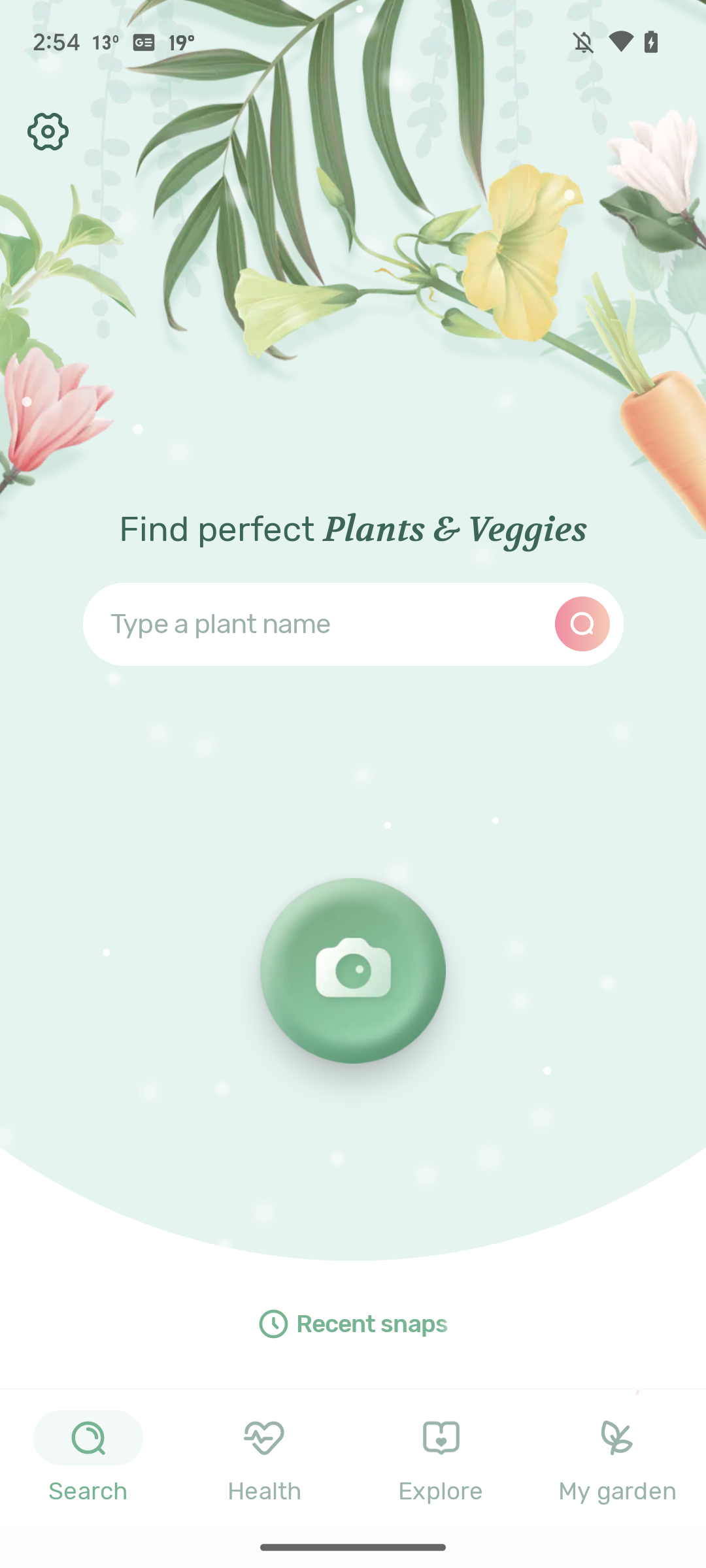}
\end{center}

\vspace{0.5em}
\textbf{System Prompt:}
\begin{tcolorbox}[colback=white]
\ttfamily
You are a GUI agent. You are given a task and a screenshot of the screen. 
You need to perform a series of actions to complete the task. 
You need to choose actions from the following list:

action\_type: Answer, action\_target: None, value: Answer text, point\_2d: None  

\#\# Explanation: Return the final answer to the user's question

action\_type: Click, action\_target: Element description, value: None, point\_2d: [x, y]  

\#\# Explanation: Tap or click a specific UI element and provide its coordinates

action\_type: Write, action\_target: Element description or None, value: Text to enter, point\_2d: [x, y] or None  

\#\# Explanation: Enter text into a specific input field or at the current focus if coordinate is None

action\_type: LongPress, action\_target: Element description, value: None, point\_2d: [x, y]  

\#\# Explanation: Press and hold on a specific UI element (mobile only) and provide its coordinates

action\_type: Scroll, action\_target: None, value: "up" | "down" | "left" | "right", point\_2d: None  

\#\# Explanation: Scroll a view or container in the specified direction

action\_type: Swipe, action\_target: Optional position or None, value: "up" | "down" | "left" | "right", point\_2d: [x, y] or None  

\#\# Explanation: Perform a swipe gesture on the screen in the given direction (mobile)

action\_type: Wait, action\_target: None, value: Number of seconds, point\_2d: None  

\#\# Explanation: Pause execution to allow the UI to load or update

action\_type: NavigateHome, action\_target: None, value: None, point\_2d: None

\#\# Explanation: Navigate to the device's home screen

action\_type: NavigateBack, action\_target: None, value: None, point\_2d: None

\#\# Explanation: Press the system "Back" button

action\_type: OpenApp, action\_target: None, value: App name, point\_2d: None 

\#\# Explanation: Launch an app by its name (mobile only)

action\_type: Terminate, action\_target: None, value: End-task message, point\_2d: None

\#\# Explanation: Signal the end of the current task with a final message
\end{tcolorbox}

\vspace{0.5em}
\textbf{Human Instruction and History:}
\begin{tcolorbox}[colback=white]
\ttfamily
Please generate the next move according to the UI screenshot 
(origin image size 1080x2400), instruction and previous actions.

\textbf{Instruction:} Using the Blossom app, I would like to know the botanical 
name of the ihi plant in order to identify the type of plant it is.

\textbf{Interaction History:}
Step 1 Action: Open the Blossom app to start searching for 
the botanical name of the ihi plant.
\end{tcolorbox}

\vspace{0.5em}
\textbf{Response:}
\ttfamily
\begin{verbatim}
<think>
The screenshot shows that the Blossom app is now open, and I am on the 
main search interface. The instruction is to find the botanical name 
of the "ihi" plant. There is a prominent search bar near the center 
of the screen labeled "Type a plant name," which is intended for 
entering the name of a plant to search for information.

Since I need to search for the ihi plant and identify its botanical 
name, the next logical step is to activate the search bar so I can 
type "ihi" into it. This will allow me to perform a search for the plant 
and get the relevant botanical information.
</think>
\end{verbatim}
\begin{verbatim}
<answer>
{
  "action_description": "Tap on the search bar labeled 
  'Type a plant name' to start entering the plant's name.",
  "action_type": "Click",
  "action_target": "Search bar input field labeled 
  'Type a plant name' located below the text 
  'Find perfect Plants & Veggies'",
  "value": "None",
  "point_2d": [500, 397]
}
</answer>
\end{verbatim}
\end{datanotebox}

\subsection{RL Data Example}

\begin{datanotebox}

\textbf{Image:}
\begin{center}
    \includegraphics[width=0.3\linewidth]{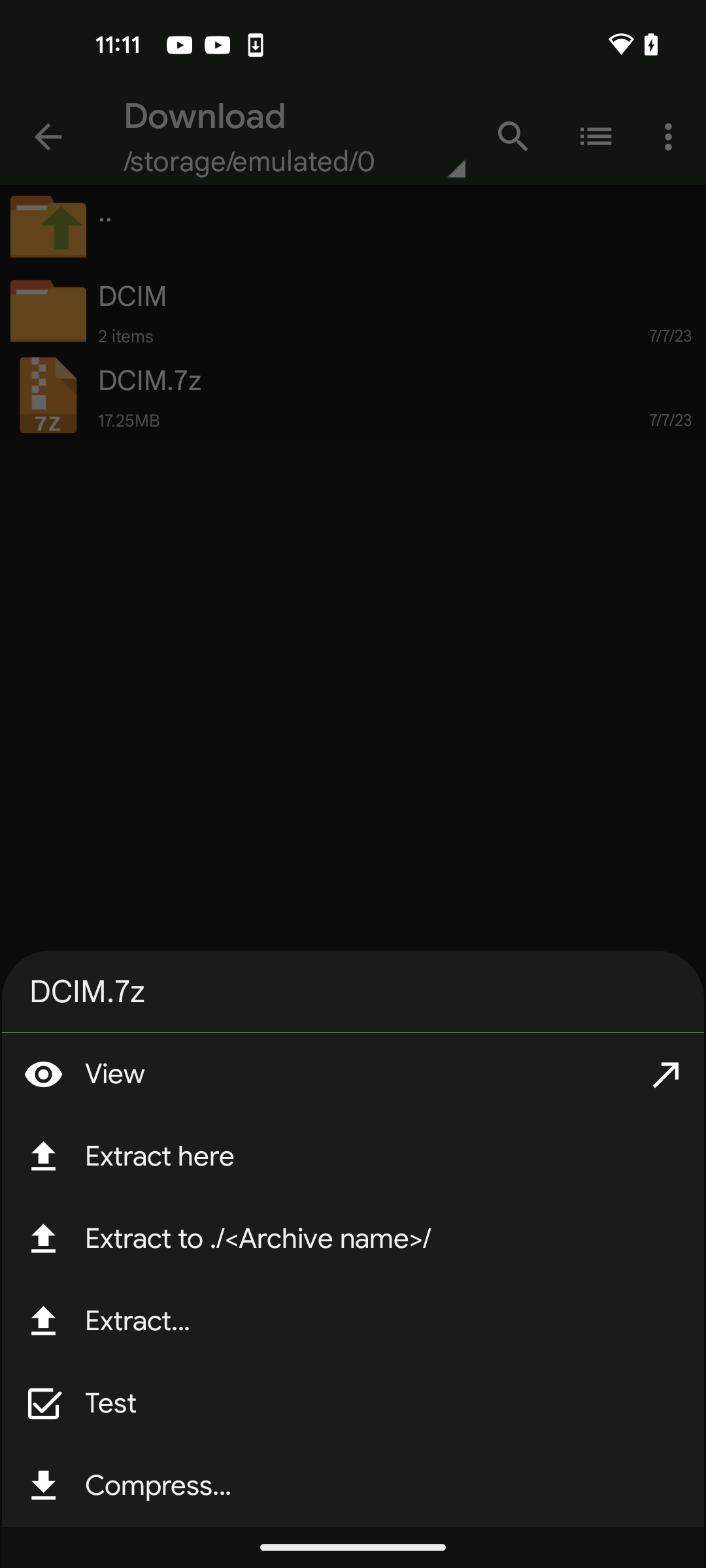}
\end{center}

\textbf{System Prompt:}
\begin{tcolorbox}[colback=white]
\ttfamily
You are a GUI agent. You are given a task and a screenshot of the screen. 
You need to perform a series of actions to complete the task. 
You need to choose actions from the following list:

action\_type: Answer, action\_target: None, value: Answer text, point\_2d: None  

\#\# Explanation: Return the final answer to the user's question

action\_type: Click, action\_target: Element description, value: None, point\_2d: [x, y]  

\#\# Explanation: Tap or click a specific UI element and provide its coordinates

action\_type: Write, action\_target: Element description or None, value: Text to enter, point\_2d: [x, y] or None  

\#\# Explanation: Enter text into a specific input field or at the current focus if coordinate is None

action\_type: LongPress, action\_target: Element description, value: None, point\_2d: [x, y]  

\#\# Explanation: Press and hold on a specific UI element (mobile only) and provide its coordinates

action\_type: Scroll, action\_target: None, value: "up" | "down" | "left" | "right", point\_2d: None  

\#\# Explanation: Scroll a view or container in the specified direction

action\_type: Swipe, action\_target: Optional position or None, value: "up" | "down" | "left" | "right", point\_2d: [x, y] or None  

\#\# Explanation: Perform a swipe gesture on the screen in the given direction (mobile)

action\_type: Wait, action\_target: None, value: Number of seconds, point\_2d: None  

\#\# Explanation: Pause execution to allow the UI to load or update

action\_type: NavigateHome, action\_target: None, value: None, point\_2d: None  

\#\# Explanation: Navigate to the device's home screen

action\_type: NavigateBack, action\_target: None, value: None, point\_2d: None  

\#\# Explanation: Press the system "Back" button

action\_type: OpenApp, action\_target: None, value: App name, point\_2d: None  

\#\# Explanation: Launch an app by its name (mobile only)

action\_type: Terminate, action\_target: None, value: End-task message, point\_2d: None  

\#\# Explanation: Signal the end of the current task with a final message
\end{tcolorbox}

\vspace{0.5em}
\textbf{Human Instruction and History:}
\begin{tcolorbox}[colback=white]
\ttfamily
<image>

Please generate the next move according to the UI screenshot, instruction and previous actions.

\textbf{Instruction:} In the ZArchive App, decompress the DCIM.7z zip file and save it to the Pocketbook folder.

\textbf{Interaction History:}

Step 1 Action: Open the ZArchiver app to begin the decompression process as required by the instruction.  

Step 2 Action: Tap on the 'DCIM.7z' archive file to open the extraction options menu.\\

The response should be structured in the following format. Make sure the output between <answer> and </answer> is a valid JSON object. Regarding the key "point\_2d", please provide the coordinates on the screen where the action is to be performed; if not applicable, use [-100, -100].

\begin{verbatim}
<think> Your step-by-step thought process here... </think>  
<answer>  
{  
  "action_description": "the description of the action to perform, 
                            summarized in one sentence",  
  "action_type": "the type of action to perform. Please follow the system prompt 
                    for available actions.",  
  "value": "the input text or direction ('up', 'down', 'left', 'right') 
                    for the 'scroll' action, if applicable; otherwise, use 'None'",  
  "point_2d": [x, y]  
}  
</answer>
\end{verbatim}
\end{tcolorbox}
\end{datanotebox}

The answer and the additional information used for reward computation are shown below.
The \texttt{gt\_point\_2d} coordinates are normalized to the range $[0,1]$, 
while \texttt{gt\_bbox} is further scaled by a factor of $1000$.

\begin{verbatim}
{
    'gt_action': 'Click',
    'gt_input_text': 'None',
    'gt_point_2d': [0.194, 0.843],
    'gt_target': "'Extract...' option in the extraction context menu for DCIM.7z", 
    'image_height': 2400, 
    'image_width': 1080, 
    'gt_bbox': [36, 824, 268, 857]
}
\end{verbatim}

\section{Long-Horizon Trajectory Case Studies}\label{ap:long_traj_example}
Figures~\ref{fig:trajectory_example_guipivot_androidworld} and~\ref{fig:trajectory_example_qwen2.5_7b_androidworld} compare GUI-Libra-7B with its base model, Qwen2.5-VL-7B-Instruct, on AndroidWorld Task 18 (\textsc{ExpenseDeleteMultiple}). The task requires deleting three specific expenses in Pro Expense: \emph{School Supplies}, \emph{Religious}, and \emph{Flight Tickets}. As shown in Figure~\ref{fig:trajectory_example_guipivot_androidworld}, GUI-Libra successfully completes this long-horizon task by alternating between iterative reasoning and grounded actions. In contrast, the base model requires more steps to discover how to delete a single item and then fails to remove the second one, highlighting its difficulty in sustaining multi-step progress and demonstrating the advantage of GUI-Libra on long-horizon decision making.

We also include a web navigation case study in Figure~\ref{fig:trajectory_example_guipivot_4b_webarenalite}, where GUI-Libra-4B successfully completes a long-horizon WebArena-Lite-v2 task, further illustrating strong generalization to multi-step web interactions.

\begin{figure*}[t]
    \vspace{-0.15in}
    \centering
    \includegraphics[width=0.86\linewidth,height=1.4\linewidth]{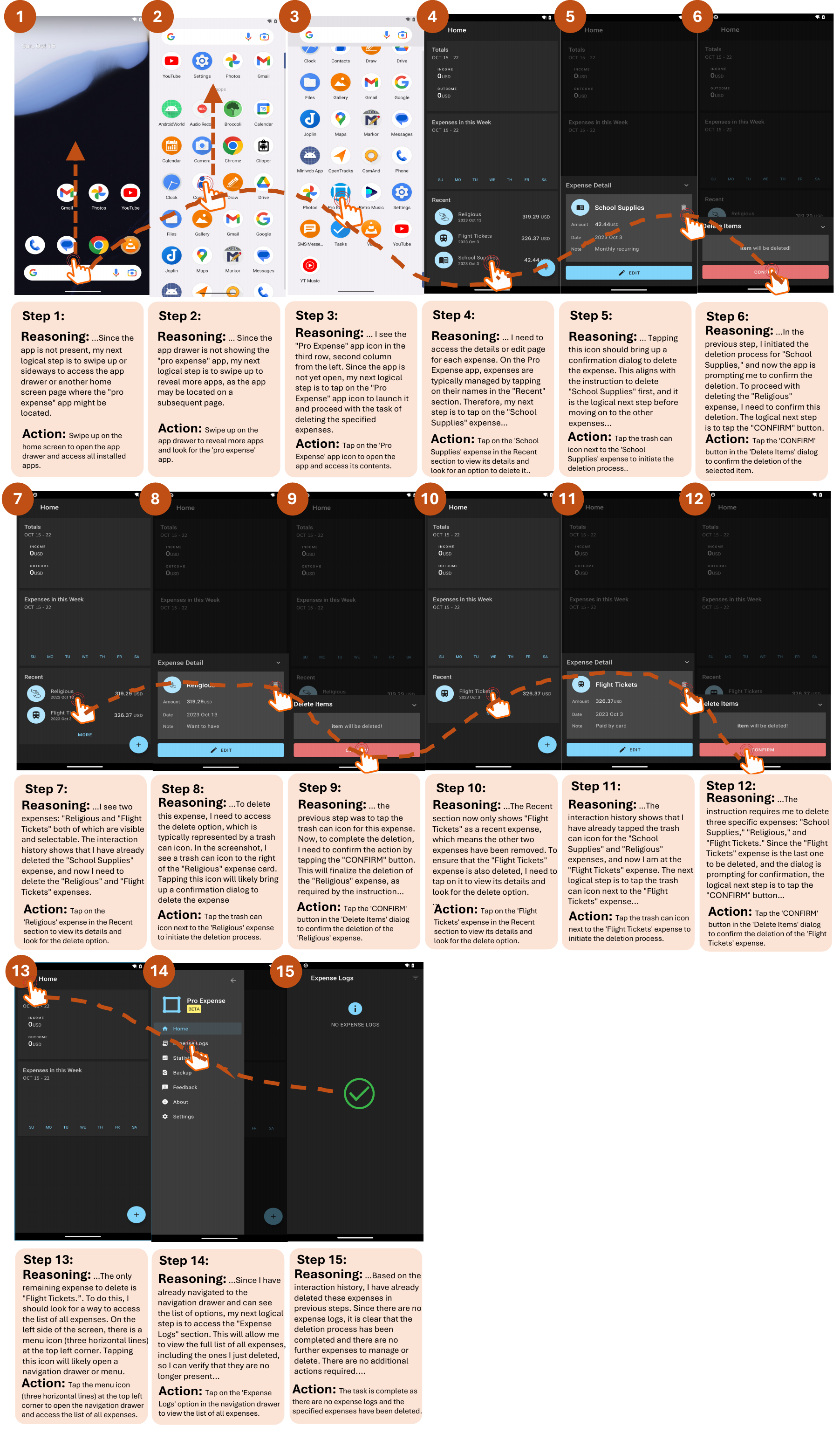}
    \vspace{-5pt}
    \caption{Trajectory Example of GUI-Libra-7B for Task 18 ExpenseDeleteMultiple: Delete the following expenses from pro expense: School Supplies, Religious, Flight Tickets.}
    \label{fig:trajectory_example_guipivot_androidworld}
\end{figure*}

\begin{figure*}[t]
    \vspace{-0.15in}
    \centering
    \includegraphics[width=0.97\linewidth,height=1.4\linewidth]{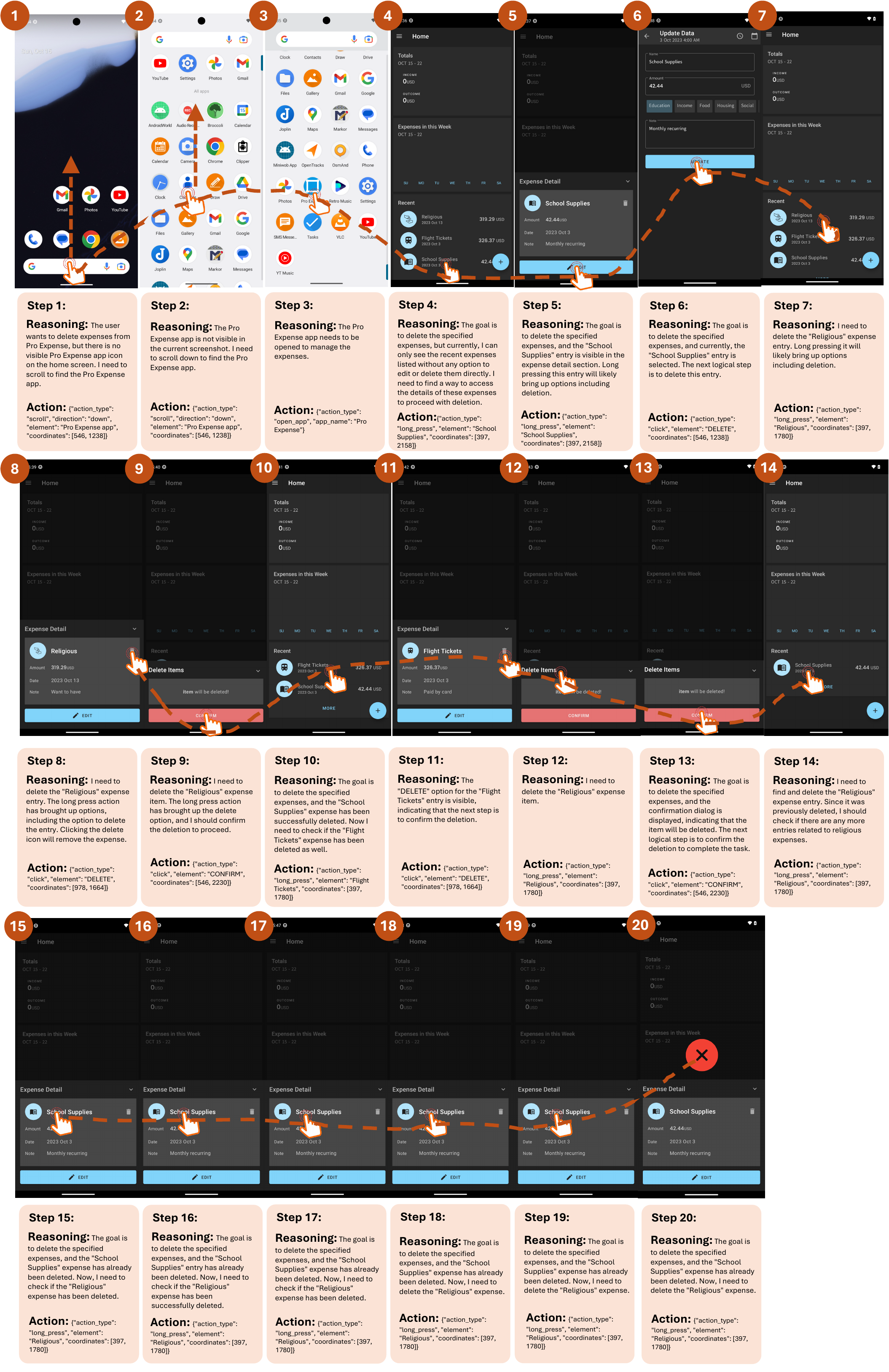}
     \vspace{-5pt}
    \caption{Trajectory Example of the base model Qwen2.5-VL-7B-Instruct for Task 18 ExpenseDeleteMultiple: Delete the following expenses from pro expense: School Supplies, Religious, Flight Tickets.}
    \label{fig:trajectory_example_qwen2.5_7b_androidworld}
\end{figure*}

\begin{figure*}[t]
    \vspace{-0.15in}
    \centering
    \includegraphics[width=1.0\linewidth]{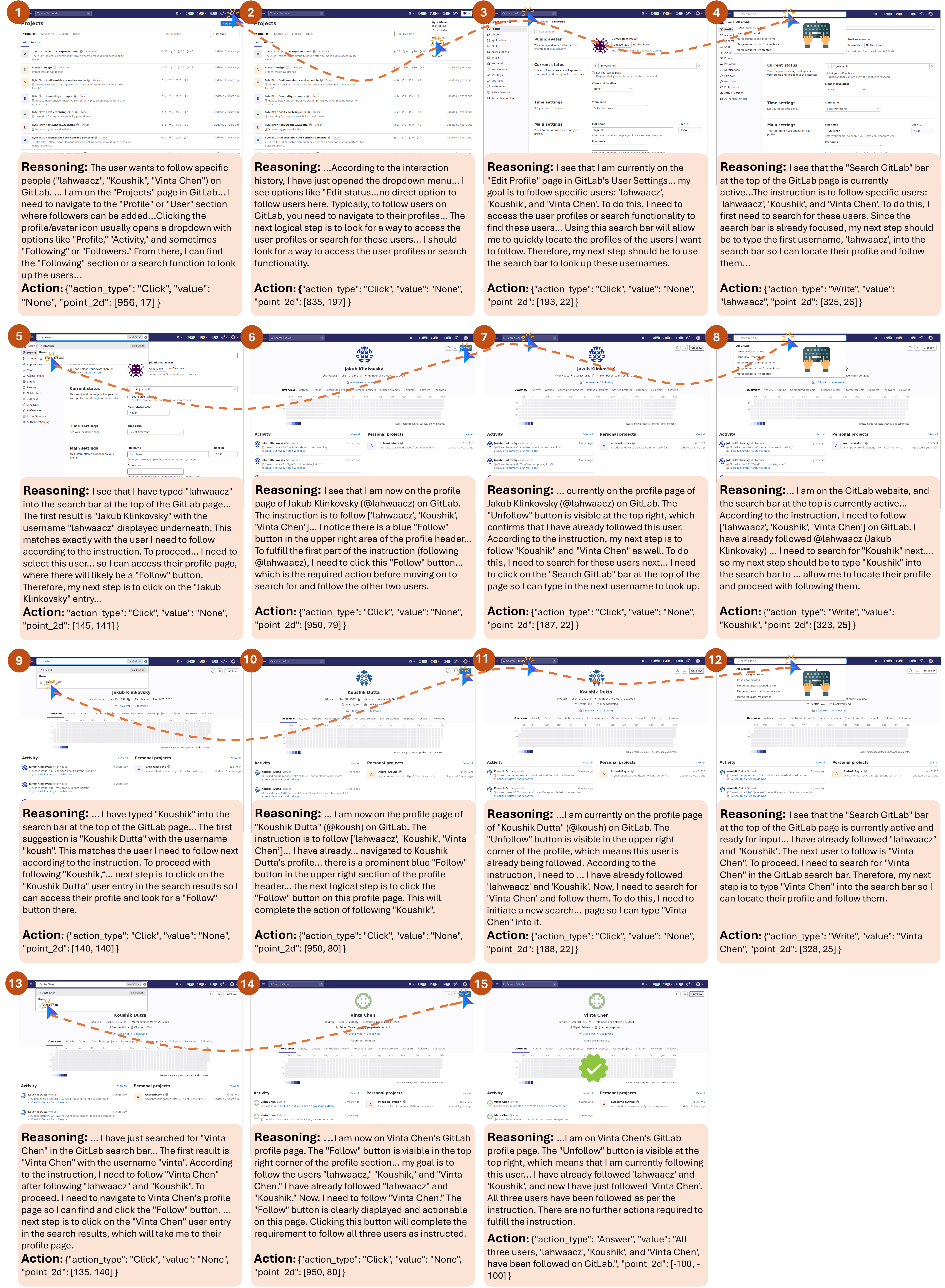}
    \caption{Trajectory Example of GUI-Libra-4B for WebArena-Lite-v2 Task: \textit{Follow ['lahwaacz', 'Koushik', 'Vinta Chen'] on Gitlab.}}
    \label{fig:trajectory_example_guipivot_4b_webarenalite}
\end{figure*}

\begin{figure*}[t]
    \vspace{-0.15in}
    \centering
    \includegraphics[width=1.0\linewidth]{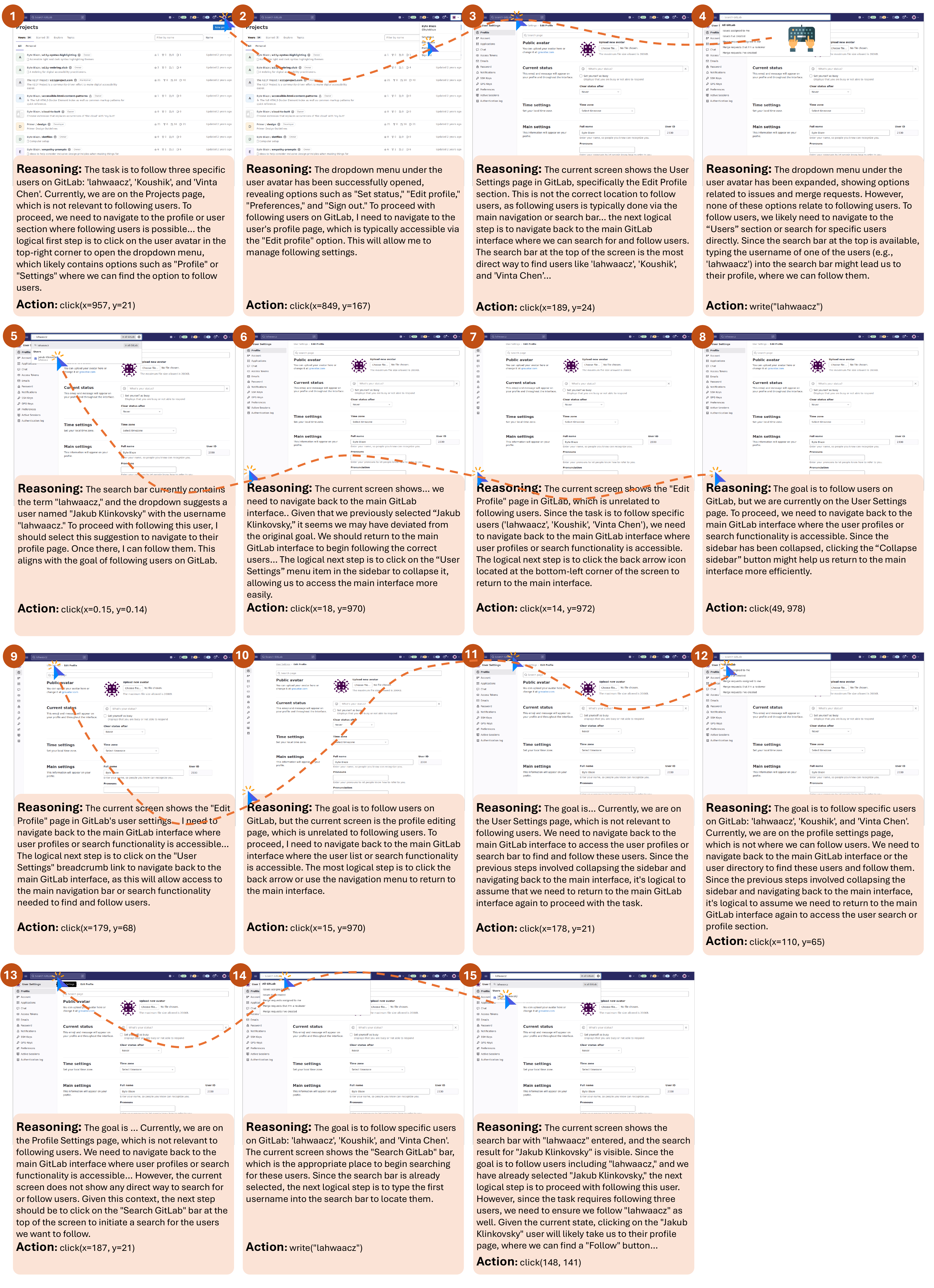}
    \caption{Trajectory Example of Qwen3-VL-4B-Instruct for WebArena-Lite-v2 Task: \textit{Follow ['lahwaacz', 'Koushik', 'Vinta Chen'] on Gitlab.}}
    \label{fig:trajectory_example_guipivot_4b_webarenalite}
\end{figure*}

%% file: sections/additional_metrics.tex
\begin{table}[]
\centering
\caption{Grounding Accuracy on AndroidControl-v2.}
\label{tab:AC_grounding_accuracy}
\resizebox{0.75\linewidth}{!}{
\begin{tabular}{lllll}
\toprule
 & \multicolumn{2}{c}{\textbf{High Level}} 
 & \multicolumn{2}{c}{\textbf{Low Level}} \\
\textbf{Model}
 & Pass@1 & Pass@4 
 & Pass@1 & Pass@4  \\

\midrule
\multicolumn{5}{c}{\textbf{Proprietary Models with SeeAct-V Framework}}\\
\midrule

GPT-4o + UGround-v1-7B  & 58.7 & 67.3
& 83.4 & 88.3 \\

GPT-4.1 + UGround-v1-7B  & 60.5 & 64.6
& 83.0 & 86.1 \\

GPT-5-mini + UGround-v1-7B  & 61.9 & 66.8
& 83.9 & 87.9 \\

GPT-5 + UGround-v1-7B  & \underline{74.4} & \textbf{82.1}
& \underline{93.7} & 94.6 \\

\midrule
\multicolumn{5}{c}{\textbf{Open-source Native Models}}\\
\midrule

 GUI-R1-3B & 53.4 & 66.4
& 81.2 & 87.0 \\

  GUI-R1-7B  & 58.3 & 72.2
& 87.9 & 87.9 \\

Aguvis-7B
 & 67.3 & 78.0 
 & 85.7 & 87.0  \\

 UI-TARS-1.5-7B
 & 57.9 & 71.3
 & 78.5 & 88.8 \\

GLM-4.1V-9B-Thinking
 & 44.0 & 55.6 
 & 68.2 & 75.8  \\

Qwen2.5-VL-32B
 & 49.0 & 66.6 
 & 78.4 & 85.4 \\

Qwen2.5-VL-72B
 & 56.5 & 72.9 
 & 82.9 & 90.2 \\

 Qwen3-VL-32B  & 70.9 & 79.4
& 91.9 & 93.3 \\

\midrule
Qwen2.5-VL-3B (Baseline)
 & 42.2 & 46.6 
 & 78.0 & 80.7\\

\rowcolor{beijingblue}
\textbf{GUI-Libra-3B (Ours)}
 & 68.2 & 77.1
  & 89.2 & 91.5 \\

Qwen2.5-VL-7B (Baseline)
 & 60.5 & 69.5 
 & 82.5 & 88.3  \\

 \rowcolor{beijingblue}
\textbf{GUI-Libra-7B (Ours)}
 & 71.8 & 78.0
  & 91.5 & 93.7 \\

 Qwen3-VL-4B (Baseline) & 65.9 & 76.7
 & 90.6 & 93.3 \\

\rowcolor{beijingblue}
\textbf{GUI-Libra-4B  (Ours)} 
 & 74.0 & \underline{81.6}
 & 91.0 & \underline{95.1} \\

 Qwen3-VL-8B (Baseline)   & 67.7 & 79.8
 & 93.3 & \underline{95.1} \\

\rowcolor{beijingblue}
\textbf{GUI-Libra-8B (Ours)} 
  & \textbf{76.2} & \underline{81.6}
 & \textbf{95.5} & \textbf{96.4} \\

\bottomrule
\end{tabular}}
\end{table}

\begin{table}[]
\centering
\caption{Type Accuracy on AndroidControl-v2.}
\label{tab:ac_type_accuracy}
\resizebox{0.75\linewidth}{!}{
\begin{tabular}{lllll}
\toprule
 & \multicolumn{2}{c}{\textbf{High Level}} 
 & \multicolumn{2}{c}{\textbf{Low Level}} \\
\textbf{Model}
 & Pass@1 & Pass@4 
 & Pass@1 & Pass@4  \\
\midrule
\multicolumn{5}{c}{\textbf{Proprietary Models with SeeAct-V Framework}}\\
\midrule

GPT-4o + UGround-v1-7B & 73.9 & 82.7
& 89.2 & 93.7 \\

GPT-4.1 + UGround-v1-7B & 72.9 & 78.9
& 89.5 & 93.0 \\

GPT-5-mini + UGround-v1-7B  & 70.9 & 77.4
& 89.7 & 93.5 \\

GPT-5 + UGround-v1-7B  & \textbf{74.1} & 78.9
& 88.4 & 91.5 \\

\midrule
\multicolumn{5}{c}{\textbf{Open-source Native Models}}\\
\midrule

GUI-R1-3B & 63.1 & 76.4
& 71.1 & 90.2 \\

GUI-R1-7B  & 61.8 & 76.1
& 76.1 & 89.2 \\

Aguvis-7B  & 58.0 & 59.8 
 & 60.6 & 62.1 \\

 UI-TARS-1.5-7B
 & 69.4 & \textbf{87.4} 
 & 70.9 & 91.0 \\

GLM-4.1V-9B-Thinking
 & 65.8 & 74.1 
 & 89.2 & 92.5 \\

Qwen2.5-VL-32B  & 67.3 & 78.1 
 & 85.7 & 89.7 \\

Qwen2.5-VL-72B  & 72.9 &\underline{86.9}
 & 89.5 & \textbf{96.2}\\

Qwen3-VL-32B  & \underline{73.6} & 81.2
& 90.2 & 92.2 \\

\midrule
Qwen2.5-VL-3B (Baseline)  & 53.0 & 77.4 
 & 80.2 & 90.0 \\

\rowcolor{beijingblue}
\textbf{GUI-Libra-3B  (Ours)}  
 & 70.4 & 78.1
 & 89.5 & 93.7 \\
 
Qwen2.5-VL-7B (Baseline) & 62.8 & 77.4 
 & 78.6 & 91.7  \\

\rowcolor{beijingblue}
\textbf{GUI-Libra-7B  (Ours)}
 & 72.9 & 78.6
 & 88.9 & 91.2 \\

Qwen3-VL-4B (Baseline)  & 63.6 & 78.4
 & 86.4 & 88.7 \\

\rowcolor{beijingblue}
\textbf{GUI-Libra-4B  (Ours)}
 & 72.6 & 76.1
 & \underline{92.5} & 95.0 \\

Qwen3-VL-8B (Baseline)   & 68.8 & 77.6
 & 82.4 & 86.9 \\

\rowcolor{beijingblue}
\textbf{GUI-Libra-8B  (Ours)} 
 & \textbf{74.1} & 78.9
 & \textbf{93.7} & \underline{95.5} \\
 
\bottomrule
\end{tabular}}
\end{table}

\begin{table}[t]
\centering
\caption{Grounding accuracy (\%) on Multimodal-Mind2Web-v2.}
\label{tab:mind2web_grounding}
\resizebox{\linewidth}{!}{
\begin{tabular}{lllllllll}
\toprule
& \multicolumn{2}{c}{Cross-Task} 
& \multicolumn{2}{c}{Cross-Website} 
& \multicolumn{2}{c}{Cross-Domain} 
& \multicolumn{2}{c}{Average} \\
\textbf{Model}
& Pass@1 & Pass@4  & Pass@1 & Pass@4   & Pass@1 & Pass@4   & Pass@1 & Pass@4 \\
\midrule
\multicolumn{9}{c}{\textbf{Proprietary Models with SeeAct-V Framework}}\\
\midrule

GPT-4o + UGround-v1-7B & 42.0 & 44.0 & 40.6 & 42.9 & 44.9 & 47.2 & 42.5 & 44.7 \\

GPT-4.1 + UGround-v1-7B   & 47.4 & 49.3
 & 43.5 & 45.5
 & 48.7 & 51.4
 & 46.5 & 48.7 \\

GPT-5-mini + UGround-v1-7B  & 52.5 & 53.1 & 50.7 & 51.6 & 53.6 & 54.3 & 52.3 & 53.0 \\

GPT-5 + UGround-v1-7B   & 55.3 & 56.6
 & 53.1 & 53.4
 & 55.8 & 57.3
 & 54.7 & 55.8 \\

\midrule
\multicolumn{9}{c}{\textbf{Open-source Native Models}}\\
\midrule

GUI-R1-3B 
& 29.1 & 44.6 & 27.2 & 45.5 & 29.1 & 44.9 & 28.5 & 45.0 \\

GUI-R1-7B  & 43.8 & 56.8 & 40.6 & 54.9 & 45.7 & 56.3 & 43.4 & 56.0 \\

Aguvis-7B
& 40.5 & 53.5
& 34.9 & 47.4
& 38.8 & 48.8
& 38.1 & 49.9 \\

UI-TARS-1.5-7B
& 46.2 & 59.0
& 42.2 & 57.6
& 46.4 & 60.2
& 45.0 & 58.9 \\

GLM-4.1V-9B-Thinking
& 36.8 & 42.8
& 34.2 & 39.7
& 36.9 & 44.1
& 36.0 & 42.2 \\

Qwen2.5-VL-32B & 53.0 & 61.8 & 50.9 & \underline{61.6} & 52.6 & 62.9 & 52.2 & 62.1 \\

Qwen2.5-VL-72B
& 55.7 & \textbf{65.0} & 53.7 & 61.5 & 56.3 & \textbf{64.4} & 55.2 & \textbf{63.6} \\

Qwen3-VL-32B 
& 55.9 & 62.4 & 52.7 & \textbf{61.7} & \underline{56.7} & \underline{63.9} & 55.1 & \underline{62.7} \\

 \midrule
Qwen2.5-VL-3B (Baseline)
& 37.9 & 40.1 & 35.6 & 38.6 & 38.8 & 43.2 & 37.4 & 40.6 \\

\rowcolor{beijingblue}
\textbf{GUI-Libra-3B (Ours)}
 & 50.1 & 57.8
 & 50.7 & 58.6
 & 51.6 & 58.5
 & 50.8 & 58.3 \\

Qwen2.5-VL-7B (Baseline)
& 44.1 & 54.0 & 45.5 & 54.2 & 46.2 & 57.4 & 45.3 & 55.2 \\

\rowcolor{beijingblue}
\textbf{GUI-Libra-7B (Ours)}
 & 53.5 & 59.8
 & 54.2 & 60.4
 & 53.6 & 60.7
 & 53.7 & 60.3 \\

Qwen3-VL-4B (Baseline) 
 & 49.9 & 59.6 & 47.9 & 58.2 & 49.9 & 57.5 & 49.2 & 58.4 \\

\rowcolor{beijingblue}
\textbf{GUI-Libra-4B  (Ours)}
  & \underline{57.2} & 61.7
 & \underline{56.2} & 61.4
 & \underline{56.7} & 62.1
 &\underline{56.7} & 61.7 \\

Qwen3-VL-8B (Baseline)
& 52.6 & 60.0 & 49.6 & 59.3 & 52.4 & 59.3 & 51.5 & 59.5 \\

\rowcolor{beijingblue}
\textbf{GUI-Libra-8B (Ours)}
 & \textbf{58.6} & \underline{62.6}
 & \textbf{56.5} & 61.2
 & \textbf{58.2} & 62.5
 & \textbf{57.8} & 62.1 \\

\bottomrule
\end{tabular}}
\end{table}

\begin{table}[t]
\centering
\caption{Operation F1 score on Multimodal-Mind2Web-v2.}
\label{tab:mind2web_operation_f1}
\resizebox{\linewidth}{!}{
\begin{tabular}{lllllllll}
\toprule
& \multicolumn{2}{c}{Cross-Task} 
& \multicolumn{2}{c}{Cross-Website} 
& \multicolumn{2}{c}{Cross-Domain} 
& \multicolumn{2}{c}{Average} \\
\textbf{Model}
& Best@1 & Best@4  & Best@1 & Best@4   & Best@1 & Best@4  & Best@1 & Best@4 \\
\midrule
\multicolumn{9}{c}{\textbf{Proprietary Models with SeeAct-V Framework}}\\
\midrule

GPT-4o + UGround-v1-7B  & 70.1 & 78.9 & 70.9 & 79.6 & 67.7 & 75.7 & 69.6 & 78.1 \\

GPT-4.1 + UGround-v1-7B   & 72.7 & 81.6
 & 71.1 & 81.0
 & 71.1 & 79.4
 & 71.6 & 80.7 \\

GPT-5-mini + UGround-v1-7B  & 80.7 & 88.8 & 77.3 & 86.8 & 77.9 & 85.3 & 78.6 & 87.0 \\

GPT-5 + UGround-v1-7B   & 79.7 & 89.2
 & 79.6 & 88.9
 & 79.1 & 88.2
 & 79.5 & 88.8 \\ 

\midrule
\multicolumn{9}{c}{\textbf{Open-source Native Models}}\\
\midrule

GUI-R1-3B  & 79.5 & 87.9 & 75.4 & 85.3 & 79.6 & 89.4 & 78.1 & 87.5 \\

GUI-R1-7B  & 78.9 & 89.6 & 75.2 & 88.5 & 79.8 & 90.7 & 78.0 & 89.6 \\

Aguvis-7B
& 84.4 & 91.1
& 80.7 & 87.9
& 83.3 & 91.5
& 82.8 & 90.2 \\

UI-TARS-1.5-7B
& 81.2 & 85.3
& 78.5 & 81.6
& 81.1 & 85.0
& 80.3 & 84.0 \\

GLM-4.1V-9B-Thinking
& 73.6 & 82.0
& 69.2 & 80.0
& 73.7 & 81.6
& 72.2 & 81.2 \\

Qwen2.5-VL-32B
& 84.2 & \underline{92.6} & 81.9 & \textbf{92.0} & 82.6 & \textbf{93.0} & 82.9 & \underline{92.5} \\

Qwen2.5-VL-72B
& \underline{85.2} & \textbf{93.5} & 83.3 & \underline{91.5} & 84.3 & \textbf{93.0} & \underline{84.3} & \textbf{92.7} \\

Qwen3-VL-32B 
& 83.9 & 91.6 & 81.3 & 89.9 & 83.7 & \underline{92.2} & 83.0 & 91.2 \\

 \midrule
Qwen2.5-VL-3B (Baseline)
& 59.3 & 86.5 & 51.6 & 83.0 & 64.7 & 88.3 & 58.5 & 85.9 \\

\rowcolor{beijingblue}
\textbf{GUI-Libra-3B (Ours)}
 & 84.6 & 88.0
 & 80.8 & 85.2
 & \textbf{86.4} & 89.2
 & 83.9 & 87.5 \\

Qwen2.5-VL-7B (Baseline)
 & 68.1 & 90.9 & 66.7 & 90.4 & 73.3 & 91.6 & 69.4 & 91.0 \\

\rowcolor{beijingblue}
\textbf{GUI-Libra-7B  (Ours)}
 & \underline{85.2} & 88.1
 & \textbf{84.3} & 88.2
 & \underline{85.9} & 89.4
 & \textbf{85.1} & 88.6 \\

Qwen3-VL-4B (Baseline) 
& 81.6 & 88.9 & 78.9 & 88.2 & 79.5 & 89.9 & 80.0 & 89.0 \\

\rowcolor{beijingblue}
\textbf{GUI-Libra-4B (Ours)} 
  & 85.0 & 88.9
 & \underline{83.6} & 88.2
 & 82.8 & 88.1
 & 83.8 & 88.4 \\

Qwen3-VL-8B (Baseline) 
 & 83.2 & 91.2 & 80.0 & 89.0 & 82.7 & 90.2 & 82.0 & 90.1 \\

\rowcolor{beijingblue}
\textbf{GUI-Libra-8B  (Ours)}
 & \textbf{85.5} & 88.1
 & 81.8 & 86.9
 & 84.4 & 87.0
 & 83.9 & 87.3 \\

\bottomrule
\end{tabular}}
\end{table}

%% file: sections/appendix-proof.tex
\section{Proofs for Theoretical Analysis}
\label{app:RLPR}

This section provides detailed proofs for the theoretical results in Sec.~\ref{sec:RLPR}.

\subsubsection{When does offline step-wise matching predict online success?}

\begin{theorem}{Offline-to-online bound under partial verifiability (Restatement of Theorem \ref{thm:off2on_main})}{off2on_app}
Assume Assumption~\ref{ass:failure} and that, for all $t\in[H]$, $d_{\pi,t}(s)>0$ implies $d_\mu(s)>0$ (i.e., $\mathrm{supp}(d_{\pi,t})\subseteq \mathrm{supp}(d_\mu)$).
This condition ensures the occupancy ratio $C(\pi)$ is well-defined.
Then the online success probability satisfies
\begin{equation}
J(\pi)\ \ge\
1\;-\;H\cdot C(\pi)\cdot\Big(1-M_{\mathrm{off}}(\pi)-\bar\eta_\pi\Big).
\label{eq:offline_online_bound_app}
\end{equation}
In particular, if $C(\pi)$ is uniformly bounded over a policy class and $\bar\eta_\pi$ is small \emph{or stable across policies},
then $M_{\mathrm{off}}(\pi)$ becomes predictive of $J(\pi)$ through the affine lower bound.
\end{theorem}

\begin{proof}
Let $E_t$ denote the event that the action at step $t$ is invalid:
\[
E_t \;\triangleq\; \{a_t\notin \mathcal{A}^*(s_t)\}.
\]
By Assumption~\ref{ass:failure}, failure implies that at least one invalid step occurs, i.e.,
\[
\{\text{failure}\}\subseteq \bigcup_{t=1}^H E_t.
\]
Therefore,
\[
1-J(\pi)
=
\mathbb{P}(\text{failure})
\le
\mathbb{P}\!\left(\bigcup_{t=1}^H E_t\right)
\le
\sum_{t=1}^H \mathbb{P}(E_t),
\]
where the last inequality is the union bound.

Fix any $t\in[H]$. Conditioning on $s_t$ gives
\[
\mathbb{P}(E_t)
=
\mathbb{E}_{s\sim d_{\pi,t}}
\Big[\mathbb{P}(a_t\notin \mathcal{A}^*(s)\mid s_t=s)\Big]
=
\mathbb{E}_{s\sim d_{\pi,t}}\Big[1-\pi(\mathcal{A}^*(s)\mid s)\Big].
\]
Define the nonnegative function $f(s)\triangleq 1-\pi(\mathcal{A}^*(s)\mid s)\ge 0$.
By the definition of $C(\pi)$ and the support condition $d_{\pi,t}\ll d_\mu$,
\[
\mathbb{E}_{s\sim d_{\pi,t}}[f(s)]
=
\mathbb{E}_{s\sim d_\mu}\!\left[\frac{d_{\pi,t}(s)}{d_\mu(s)}\,f(s)\right]
\le
C(\pi)\cdot \mathbb{E}_{s\sim d_\mu}[f(s)].
\]
Next, since the offline dataset verifies only the demonstrated action $\tilde a(s)\in \mathcal{A}^*(s)$,
we can decompose the true step-wise validity probability as
\[
\pi(\mathcal{A}^*(s)\mid s)
=
\pi(\tilde a(s)\mid s)\;+\;\pi\!\left(\mathcal{A}^*(s)\setminus\{\tilde a(s)\}\mid s\right)
=
\pi(\tilde a(s)\mid s)\;+\;\eta_\pi(s).
\]
Taking expectation over $s\sim d_\mu$ yields
\[
\mathbb{E}_{s\sim d_\mu}\big[\pi(\mathcal{A}^*(s)\mid s)\big]
=
M_{\mathrm{off}}(\pi)\;+\;\bar\eta_\pi,
\]
and hence
\[
\mathbb{E}_{s\sim d_\mu}[f(s)]
=
1-\mathbb{E}_{s\sim d_\mu}\big[\pi(\mathcal{A}^*(s)\mid s)\big]
=
1-M_{\mathrm{off}}(\pi)-\bar\eta_\pi.
\]
Combining the above bounds, for each $t$ we have
\[
\mathbb{P}(E_t)
\le
C(\pi)\cdot \big(1-M_{\mathrm{off}}(\pi)-\bar\eta_\pi\big).
\]
Summing over $t=1,\dots,H$ and using $1-J(\pi)\le \sum_{t=1}^H \mathbb{P}(E_t)$ gives
\[
1-J(\pi)
\le
H\cdot C(\pi)\cdot \big(1-M_{\mathrm{off}}(\pi)-\bar\eta_\pi\big),
\]
which is equivalent to Eq~\ref{eq:offline_online_bound_app}. This concludes the proof.
\end{proof}

\paragraph{Discussion.}
Theorem~\ref{thm:off2on_app} (Theorem~\ref{thm:off2on_main}) clarifies why an offline one-step matching score can be an unreliable proxy for the online success probability $J(\pi)$ in multi-step GUI navigation under partial verifiability.
Offline evaluation is computed on a fixed dataset distribution $d_\mu$ and credits only the single demonstrated action $\tilde a(s)$.
In contrast, online success depends on the policy-induced state distributions $\{d_{\pi,t}\}_{t=1}^H$ and on choosing \emph{any} valid action in $\mathcal{A}^*(s)$ over a long horizon.
The theorem makes this mismatch explicit through two quantities: the occupancy mismatch $C(\pi)$ and the unobserved off-demo validity mass $\bar\eta_\pi$.
As a result, offline-to-online predictability can break down in several ways.

\textbf{(i) Distribution shift and error accumulation (large $C(\pi)$).}
Even if $\pi$ matches the demonstrator well on states drawn from $d_\mu$, small errors can compound over time and shift the online trajectory distribution away from the offline support.
When $C(\pi)$ is large, $M_{\mathrm{off}}(\pi)$ provides limited information about the states that dominate online performance.
In this regime, $M_{\mathrm{off}}(\pi)$ may improve while $J(\pi)$ stagnates (or decreases) because failures are driven by states that are rarely or never seen in the offline data.

\textbf{(ii) Non-identifiability under partial verifiability (unstable $\bar\eta_\pi$).}
Offline matching measures only $\pi(\tilde a(s)\mid s)$, whereas the true one-step validity is
\[
\pi(\mathcal{A}^*(s)\mid s)=\pi(\tilde a(s)\mid s)+\eta_\pi(s),
\qquad
\eta_\pi(s)=\pi\!\left(\mathcal{A}^*(s)\setminus\{\tilde a(s)\}\mid s\right).
\]
Therefore, $M_{\mathrm{off}}(\pi)$ does not determine true step validity unless the off-demo validity mass $\eta_\pi(s)$ is negligible or approximately invariant across the policies being compared.
Importantly, a \emph{larger} $\bar\eta_\pi$ does \emph{not} imply better predictability: the issue is that $\eta_\pi$ is unobserved under offline verification and can vary substantially across policies, so changes in $M_{\mathrm{off}}(\pi)$ may reflect reallocation of probability mass rather than genuine improvements in correctness.

\begin{example}{Larger $\eta_\pi$ does not mean better predictability}{policy_comparison}
    Consider a single decision state $s$ (so $C(\pi)=1$) with three actions:
a demonstrated valid action $a^\star=\tilde a(s)$, an alternative valid action $a'$, and an invalid action $a_{\mathrm{bad}}$.
Thus $\mathcal{A}^*(s)=\{a^\star,a'\}$ and choosing any action outside $\mathcal{A}^*(s)$ causes failure.
Offline matching credits only $a^\star$, so
\[
M_{\mathrm{off}}(\pi)=\pi(a^\star\mid s),\qquad
\eta_\pi(s)=\pi(a'\mid s),\qquad
J(\pi)=\pi(a^\star\mid s)+\pi(a'\mid s)=1-\pi(a_{\mathrm{bad}}\mid s).
\]

Compare two policies:
\[
\pi_1(a^\star\mid s)=0.2,\ \pi_1(a'\mid s)=0.7,\ \pi_1(a_{\mathrm{bad}}\mid s)=0.1
\quad\Rightarrow\quad
M_{\mathrm{off}}(\pi_1)=0.2,\ \eta_{\pi_1}(s)=0.7,\ J(\pi_1)=0.9,
\]
\[
\pi_2(a^\star\mid s)=0.4,\ \pi_2(a'\mid s)=0.1,\ \pi_2(a_{\mathrm{bad}}\mid s)=0.5
\quad\Rightarrow\quad
M_{\mathrm{off}}(\pi_2)=0.4,\ \eta_{\pi_2}(s)=0.1,\ J(\pi_2)=0.5.
\]
Here the offline score increases ($0.2\to0.4$), while the off-demo validity mass changes dramatically ($0.7\to0.1$) and the true success probability drops sharply ($0.9\to0.5$).
This happens because $M_{\mathrm{off}}$ only tracks probability on the single demonstrated action $a^\star$; it cannot distinguish whether probability mass is reallocated from other valid alternatives $a'$ (uncredited) to invalid actions $a_{\mathrm{bad}}$.
\end{example}

\textbf{(iii) Offline overfitting can reduce online robustness.}
Example~\ref{exa:policy_comparison} also illustrates another failure mode: maximizing an offline demo-matching score can encourage demo-specific behavior. Because $M_{\mathrm{off}}(\pi)$ rewards only matching the single demonstrated action $\tilde a(s)$, a policy may increase $M_{\mathrm{off}}(\pi)$ by concentrating probability mass on $\tilde a(s)$ while reducing exploration of other valid alternatives.
This overfitting reduces behavioral diversity and weakens recovery strategies that are essential for interactive agents.
In long-horizon GUI navigation, the impact is amplified by error accumulation.
A small early mistake can move the agent to states that are poorly covered by the offline data, where the policy has not learned robust correction behaviors.
This increases state-distribution shift (larger $C(\pi)$) and can lower the overall success probability $J(\pi)$.
Consequently, it is possible for $M_{\mathrm{off}}(\pi)$ to improve while $J(\pi)$ decreases.

Taken together, these failure modes explain why offline one-step matching can be poorly predictive of online success in multi-step GUI navigation:
offline evaluation measures demo-matching under $d_\mu$, whereas online success depends on long-horizon validity under $\{d_{\pi,t}\}$ and is confounded by unobserved (and potentially unstable) off-demo validity mass.

\subsubsection{KL regularization improves predictability}\label{ap:KL_improves_predictability}
Many RLVR pipelines drop the KL regularization term for efficiency \citep{yu2025dapo,liu2025understanding,zhou2025guig1,yang2025gta1}.
In our setting, however, KL regularization is important because step-wise offline matching is only partially verifiable: matches to the demonstrated action provide reliable positive signal, while non-matches are ambiguous.
Below we connect KL regularization (to a reference policy) to the two quantities that govern offline-to-online predictability: the occupancy mismatch $C(\pi)$ and the off-demo validity mass $\bar\eta_\pi$.

\begin{lemma}{KL regularization controls distribution shift}{kl_occupancy}
Let $\pi_{\mathrm{ref}}$ be a reference policy and assume a per-state KL constraint
\[
\mathrm{KL}\!\left(\pi(\cdot\mid s)\,\|\,\pi_{\mathrm{ref}}(\cdot\mid s)\right)\le \varepsilon,
\qquad \forall s\in\mathcal{S}.
\]
Then Pinsker's inequality implies
\[
\mathrm{TV}\!\left(\pi(\cdot\mid s),\pi_{\mathrm{ref}}(\cdot\mid s)\right)\le \sqrt{\varepsilon/2},
\qquad \forall s\in\mathcal{S}.
\]
Consequently, the induced state visitation distributions satisfy
\[
\|d_{\pi,t}-d_{\pi_{\mathrm{ref}},t}\|_1 \le t\sqrt{2\varepsilon},
\qquad \forall t\in[H].
\]
Moreover, if $d_\mu(s)\ge \rho>0$ whenever $d_\mu(s)>0$ and $d_{\pi_{\mathrm{ref}},t}\ll d_\mu$, then
\[
\sup_{s:\,d_\mu(s)>0}\frac{d_{\pi,t}(s)}{d_\mu(s)}
\le
\sup_{s:\,d_\mu(s)>0}\frac{d_{\pi_{\mathrm{ref}},t}(s)}{d_\mu(s)}
+\frac{\|d_{\pi,t}-d_{\pi_{\mathrm{ref}},t}\|_1}{\rho}
\le
C(\pi_{\mathrm{ref}})+\frac{t\sqrt{2\varepsilon}}{\rho}.
\]
In particular, KL regularization bounds how much the occupancy mismatch $C(\pi)$ can grow relative to $\pi_{\mathrm{ref}}$.
\end{lemma}

\begin{proof}
\textbf{Step 1: From KL to per-state TV.}
By Pinsker's inequality, for every $s$,
\[
\mathrm{TV}\!\left(\pi(\cdot\mid s),\pi_{\mathrm{ref}}(\cdot\mid s)\right)
\le
\sqrt{\tfrac12\,\mathrm{KL}\!\left(\pi(\cdot\mid s)\,\|\,\pi_{\mathrm{ref}}(\cdot\mid s)\right)}
\le
\sqrt{\varepsilon/2}.
\]
Equivalently, for every measurable set $A\subseteq\mathcal{A}$,
\begin{equation}
\label{eq:tv_event_bound_app}
\big|\pi(A\mid s)-\pi_{\mathrm{ref}}(A\mid s)\big|
\le
\mathrm{TV}\!\left(\pi(\cdot\mid s),\pi_{\mathrm{ref}}(\cdot\mid s)\right)
\le
\sqrt{\varepsilon/2}.
\end{equation}

\textbf{Step 2: One-step propagation bound.}
Let $P_\pi$ denote the (time-homogeneous) state-transition operator induced by $\pi$:
\[
(P_\pi \nu)(s') \triangleq \sum_{s\in\mathcal{S}} \nu(s)\sum_{a\in\mathcal{A}} \pi(a\mid s)\,P(s'\mid s,a),
\]
and similarly define $P_{\pi_{\mathrm{ref}}}$.
For any distribution $\nu$ over states, consider $\|\nu P_\pi - \nu P_{\pi_{\mathrm{ref}}}\|_1$.
For each $s'$, define
\[
P_\pi(s'\mid s)\triangleq \sum_a \pi(a\mid s)P(s'\mid s,a),
\qquad
P_{\pi_{\mathrm{ref}}}(s'\mid s)\triangleq \sum_a \pi_{\mathrm{ref}}(a\mid s)P(s'\mid s,a).
\]
Then
\[
\|\nu P_\pi - \nu P_{\pi_{\mathrm{ref}}}\|_1
=
\left\|\sum_s \nu(s)\big(P_\pi(\cdot\mid s)-P_{\pi_{\mathrm{ref}}}(\cdot\mid s)\big)\right\|_1
\le
\sum_s \nu(s)\,\big\|P_\pi(\cdot\mid s)-P_{\pi_{\mathrm{ref}}}(\cdot\mid s)\big\|_1.
\]
Moreover, for each fixed $s$,
\[
\big\|P_\pi(\cdot\mid s)-P_{\pi_{\mathrm{ref}}}(\cdot\mid s)\big\|_1
=
\left\|\sum_a \big(\pi(a\mid s)-\pi_{\mathrm{ref}}(a\mid s)\big)P(\cdot\mid s,a)\right\|_1
\le
\sum_a \big|\pi(a\mid s)-\pi_{\mathrm{ref}}(a\mid s)\big|
=
\|\pi(\cdot\mid s)-\pi_{\mathrm{ref}}(\cdot\mid s)\|_1.
\]
Using $\|\cdot\|_1 = 2\,\mathrm{TV}(\cdot,\cdot)$ for distributions,
\[
\big\|P_\pi(\cdot\mid s)-P_{\pi_{\mathrm{ref}}}(\cdot\mid s)\big\|_1
\le
2\,\mathrm{TV}\!\left(\pi(\cdot\mid s),\pi_{\mathrm{ref}}(\cdot\mid s)\right)
\le
2\sqrt{\varepsilon/2}
=
\sqrt{2\varepsilon}.
\]
Therefore, for any $\nu$,
\begin{equation}
\label{eq:one_step_shift_app}
\|\nu P_\pi - \nu P_{\pi_{\mathrm{ref}}}\|_1 \le \sqrt{2\varepsilon}.
\end{equation}

\textbf{Step 3: Telescoping over $t$ steps.}
Let $d_{\pi,t}$ and $d_{\pi_{\mathrm{ref}},t}$ denote the state distributions at step $t$ under $\pi$ and $\pi_{\mathrm{ref}}$, respectively.
Using the recursion $d_{\pi,t+1}=d_{\pi,t}P_\pi$ and $d_{\pi_{\mathrm{ref}},t+1}=d_{\pi_{\mathrm{ref}},t}P_{\pi_{\mathrm{ref}}}$, we have
\[
\|d_{\pi,t+1}-d_{\pi_{\mathrm{ref}},t+1}\|_1
=
\|d_{\pi,t}P_\pi-d_{\pi_{\mathrm{ref}},t}P_{\pi_{\mathrm{ref}}}\|_1
\le
\underbrace{\|d_{\pi,t}P_\pi-d_{\pi,t}P_{\pi_{\mathrm{ref}}}\|_1}_{\le\sqrt{2\varepsilon}\ \text{by \eqref{eq:one_step_shift_app}}}
+
\underbrace{\|d_{\pi,t}P_{\pi_{\mathrm{ref}}}-d_{\pi_{\mathrm{ref}},t}P_{\pi_{\mathrm{ref}}}\|_1}_{\le \|d_{\pi,t}-d_{\pi_{\mathrm{ref}},t}\|_1}.
\]
Thus,
\[
\|d_{\pi,t+1}-d_{\pi_{\mathrm{ref}},t+1}\|_1
\le
\|d_{\pi,t}-d_{\pi_{\mathrm{ref}},t}\|_1 + \sqrt{2\varepsilon}.
\]
Iterating this inequality and using $\|d_{\pi,0}-d_{\pi_{\mathrm{ref}},0}\|_1=0$ (same initial distribution) yields
\[
\|d_{\pi,t}-d_{\pi_{\mathrm{ref}},t}\|_1 \le t\sqrt{2\varepsilon}.
\]

\textbf{Step 4: Bounding occupancy mismatch relative to $d_\mu$.}
Assume $d_{\pi_{\mathrm{ref}},t}\ll d_\mu$ and $d_\mu(s)\ge \rho>0$ whenever $d_\mu(s)>0$.
For any $s$ with $d_\mu(s)>0$,
\[
\frac{d_{\pi,t}(s)}{d_\mu(s)}
=
\frac{d_{\pi_{\mathrm{ref}},t}(s)}{d_\mu(s)}+\frac{d_{\pi,t}(s)-d_{\pi_{\mathrm{ref}},t}(s)}{d_\mu(s)}
\le
\frac{d_{\pi_{\mathrm{ref}},t}(s)}{d_\mu(s)}+\frac{|d_{\pi,t}(s)-d_{\pi_{\mathrm{ref}},t}(s)|}{\rho}.
\]
Taking supremum over $s$ with $d_\mu(s)>0$ and using
$\sup_s |x_s|\le \sum_s |x_s| = \|x\|_1$ gives
\[
\sup_{s:\,d_\mu(s)>0}\frac{d_{\pi,t}(s)}{d_\mu(s)}
\le
\sup_{s:\,d_\mu(s)>0}\frac{d_{\pi_{\mathrm{ref}},t}(s)}{d_\mu(s)}
+\frac{\|d_{\pi,t}-d_{\pi_{\mathrm{ref}},t}\|_1}{\rho}
\le
C(\pi_{\mathrm{ref}})+\frac{t\sqrt{2\varepsilon}}{\rho},
\]
which completes the proof.
\end{proof}

\begin{lemma}{KL limits off-demo validity mass under single-demo supervision}{kl_shrinks_ambiguity}
Fix a state $s$ and let $\tilde a=\tilde a(s)$ be the demonstrated action.
Assume $\mathrm{KL}(\pi(\cdot\mid s)\,\|\,\pi_{\mathrm{ref}}(\cdot\mid s))\le \varepsilon$ and that the reference policy is demo-concentrated:
$\pi_{\mathrm{ref}}(\tilde a\mid s)\ge 1-\delta(s)$.
Then
\[
1-\pi(\tilde a\mid s)\le \delta(s)+\sqrt{\varepsilon/2},
\qquad
\eta_\pi(s)=\pi\!\left(\mathcal{A}^*(s)\setminus\{\tilde a\}\mid s\right)\le 1-\pi(\tilde a\mid s),
\]
and therefore
\[
\eta_\pi(s)\le \delta(s)+\sqrt{\varepsilon/2},
\qquad
\bar\eta_\pi\le \bar\delta+\sqrt{\varepsilon/2},
\quad
\bar\delta\triangleq \mathbb{E}_{s\sim d_\mu}[\delta(s)].
\]
\end{lemma}

\begin{proof}
By Pinsker's inequality,
\[
\mathrm{TV}\!\left(\pi(\cdot\mid s),\pi_{\mathrm{ref}}(\cdot\mid s)\right)\le \sqrt{\varepsilon/2}.
\]
Let $E\triangleq\{a\neq \tilde a\}$. Using the standard event bound for total variation distance,
\[
\pi(E\mid s)\le \pi_{\mathrm{ref}}(E\mid s)
+\mathrm{TV}\!\left(\pi(\cdot\mid s),\pi_{\mathrm{ref}}(\cdot\mid s)\right)
\le \big(1-\pi_{\mathrm{ref}}(\tilde a\mid s)\big)+\sqrt{\varepsilon/2}
\le \delta(s)+\sqrt{\varepsilon/2}.
\]
Since $\pi(E\mid s)=1-\pi(\tilde a\mid s)$, we obtain
$1-\pi(\tilde a\mid s)\le \delta(s)+\sqrt{\varepsilon/2}$.

Next, because $\mathcal{A}^*(s)\setminus\{\tilde a\}\subseteq E$, we have
\[
\eta_\pi(s)=\pi\!\left(\mathcal{A}^*(s)\setminus\{\tilde a\}\mid s\right)\le \pi(E\mid s)
\le \delta(s)+\sqrt{\varepsilon/2}.
\]
Taking expectation over $s\sim d_\mu$ yields
\[
\bar\eta_\pi
=
\mathbb{E}_{s\sim d_\mu}[\eta_\pi(s)]
\le
\mathbb{E}_{s\sim d_\mu}[\delta(s)] + \sqrt{\varepsilon/2}
=
\bar\delta+\sqrt{\varepsilon/2},
\]
which completes the proof.
\end{proof}

\paragraph{Takeaway.}
The offline-to-online bound shows that predictability depends on (i) distribution shift along the policy's own trajectories (captured by $C(\pi)$) and (ii) the unobserved probability mass on valid-but-uncredited actions under single-demo verification (captured by $\bar\eta_\pi$).
Lemma~\ref{lemma:kl_occupancy} and Lemma~\ref{lemma:kl_shrinks_ambiguity} explain why KL regularization helps in this setting:
a KL trust region simultaneously limits state-distribution drift (controlling $C(\pi)$) and prevents the policy from moving too much probability mass away from the demonstrated action (controlling $\bar\eta_\pi$).
As a result, conservative KL-regularized optimization keeps training in a regime where improvements in the offline matching score $M_{\mathrm{off}}(\pi)$ are more likely to reflect genuine improvements in online success $J(\pi)$.